\documentclass[twoside]{article}

\usepackage{microtype,bbm}
\usepackage{graphicx}
\usepackage{subcaption}
\usepackage{makecell}

\newcommand{\nocontentsline}[3]{}
\let\origcontentsline\addcontentsline
\newcommand\stoptoc{\let\addcontentsline\nocontentsline}
\newcommand\resumetoc{\let\addcontentsline\origcontentsline}

\usepackage[table,xcdraw]{xcolor}
\definecolor{ScoreBest}{HTML}{1565C0}
\definecolor{ScoreSecond}{HTML}{C62828}
\definecolor{ScoreThird}{HTML}{2E7D32}
\newcommand{\bestscore}[1]{\textcolor{ScoreBest}{\mathbf{#1}}}
\newcommand{\secondscore}[1]{\textcolor{ScoreSecond}{\mathbf{#1}}}

\newcommand{\ranklegend}{\textcolor{ScoreBest}{Blue}/\textcolor{ScoreSecond}{red}: best/second-best unrounded means across all methods; ties share a colour.}
\usepackage{nicefrac}
\usepackage{mathtools}
\usepackage{booktabs} 
\usepackage{float}
\usepackage{longtable}
\usepackage{amsmath}
\usepackage[round]{natbib}

\usepackage{hyperref}

\usepackage[preprint]{aistats2027}
\hypersetup{
    hidelinks,
    pdftitle={The Credal Geometry of Conformal Prediction: Uncertainty Measures and their Limits},
    pdfauthor={Anonymous Authors},
    pdfsubject={AISTATS 2027 submission}
}

\usepackage{amsmath}
\usepackage{amssymb}
\usepackage{mathtools}
\usepackage{amsthm}

\def\RR{{\mathbb R}}    
\def\PP{{\mathbb P}}     

\def\11{{\mathbf 1}}    

   \def\cM{{\mathcal M}}    \def\cH{{\mathcal H}}           \def\cP{{\mathcal P}}           \def\cX{{\mathcal X}} \def\cY{{\mathcal Y}}

\usepackage[capitalize,noabbrev]{cleveref}

\usepackage{enumitem}
\theoremstyle{plain}
\usepackage{thmtools, thm-restate}

\makeatletter
\@ifundefined{newcounteralias}{}{%
  \renewcommand\thmt@autorefsetup{%
    \@xa\def\csname\thmt@envname autorefname\@xa\endcsname
      \@xa{\thmt@thmname}%
  }%
}
\makeatother

\usepackage{needspace}
\newtheorem{theorem}{Theorem}[section]
\newtheorem{proposition}[theorem]{Proposition}

\theoremstyle{definition}
\newtheorem{definition}[theorem]{Definition}
\newtheorem{assumption}[theorem]{Assumption}
\theoremstyle{remark}

\makeatletter
\addtotheorempreheadhook{%
  \def\theHequation{\theHthmt@dummyctr.\arabic{equation}}%
}
\makeatother

\usepackage{amsthm}
\usepackage{multicol}
\usepackage{caption}
\usepackage{minitoc}
\usepackage{subcaption}
\usepackage{graphicx}
\usepackage{url}            
\usepackage{booktabs}       
\usepackage{amsfonts}       
\usepackage{bm}

\newcommand{\circint}{\mathop{\mathpalette\docircint\relax}\!\int}
\newcommand{\docircint}[2]{%
    \ifx
        #1\displaystyle
        \displaycircint
    \else
        \normalcircint{#1}%
    \fi
}
\newcommand{\displaycircint}{\displaystyle\mathsf{c}\mkern-18mu}
\newcommand{\normalcircint}[1]{%
    \smallerc{#1}\ifx
                     #1\textstyle\mkern-9mu\else
                     \mkern-8.2mu
    \fi
}
\newcommand{\smallerc}[1]{%
    \vcenter{\hbox{$\ifx
                        #1\textstyle\scriptstyle\else
                        \scriptscriptstyle
    \fi\mathsf{c}$}}%
}

\usepackage[textsize=tiny]{todonotes}

\begin{document}

    \runningtitle{Quantifying Conformal Uncertainty through Credal Geometry}
    \runningauthor{Chau et al.}
    
    \twocolumn[
    \aistatstitle{Beyond Prediction-Set Size: \\ Quantifying Conformal Uncertainty through Credal Geometry}

        \aistatsauthor{Siu Lun Chau \And Soroush H. Zargarbashi \And
        Yusuf Sale \And Michele Caprio}

        \aistatsaddress{Nanyang Technological University \\ Singapore \And
        CISPA \\ Germany \And
        LMU Munich \\ Germany \And
        University of Warwick\\ UK}
    ]


\begin{abstract}
    Prediction-set size offers a coarse summary of conformal uncertainty. We develop a richer perspective through the geometry of credal sets: families of probability distributions compatible with conformal prediction (CP) regions. We establish this representation for split CP, extending a correspondence recently identified for full CP. Building on this representation, we derive analytical characterisations and exact algorithms for several widely used credal uncertainty measures, reducing potentially exponential computations to linear time after sorting the conformal p-values. We further unify several of these measures with classical conformal summaries by integrating set-size penalties based on Tsallis entropies across error levels. We also prove an impossibility theorem: for common split-conformal regression constructions, uncertainty summaries of the induced credal sets cannot distinguish between inputs. Classification experiments on selective prediction and distribution-shift detection show that these summaries often recover useful ranking information lost by prediction-set size at a fixed error level, enabling finer-grained uncertainty assessment. These results establish credal geometry as a principled framework for quantifying conformal uncertainty and understanding its limitations.
\end{abstract}

    \stoptoc
    \section{INTRODUCTION}

Conformal prediction~(CP) equips a fitted classifier with prediction sets that provide finite-sample coverage guarantees under exchangeability~\citep{conformal_tutorial,balasubramanian2014conformal}. For an input $x$, CP assigns a p-value $\pi_x(y)$ to each candidate label $y$. At a chosen error level $\alpha$, the labels satisfying $\pi_x(y)>\alpha$ form the prediction set $C_\alpha(x)$. The vector of p-values across all labels, which we call the \emph{conformal profile} or \emph{contour}, determines how this set changes as $\alpha$ varies. Although these p-values are not class probabilities, their complete profile contains more information than a prediction set at any single level.

This distinction matters when uncertainty guides decisions, such as which predictions to abstain from or which observations to flag as potentially shifted. A common conformal uncertainty score is the prediction-set size $|C_\alpha(x)|$ at a user-selected level. This score counts only the labels whose p-values exceed a single threshold and takes at most $K+1$ distinct values in a $K$-class problem. It therefore provides a coarse summary of the profile. Figure~\ref{fig: p_val_dist_example} illustrates how distinct profiles can produce identical prediction sets. This motivates our central question: \emph{Can the complete conformal profile provide finer and more useful uncertainty rankings?}

\begin{figure}[t]
    \centering
    \includegraphics[width=0.9\linewidth]{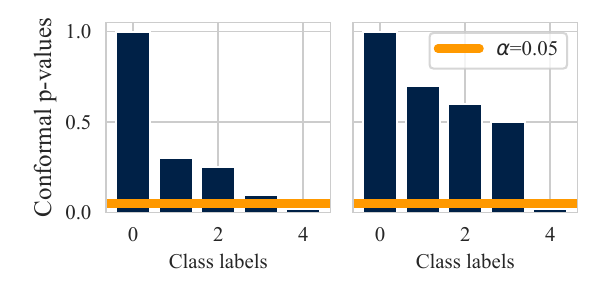}
    \vspace{-1em}
    \caption{Prediction-set size at $\alpha=0.05$ are identical, yet the left-hand input has smaller p-values for every competing label. This motivates finer-grained conformal uncertainty quantification.}
    \vspace{-1em}
    \label{fig: p_val_dist_example}
\end{figure}

Recent connections between CP and imprecise probability~(IP)~\citep{walley} offer some answers: a suitably normalised conformal profile induces a \emph{credal set}, or \emph{core}, of compatible probability distributions~\citep{cella,martinbordini2023}. The nested prediction sets obtained by varying the conformal threshold have also been linked to imprecise highest density regions~(IHDRs), which generalise classical highest density regions from a single probability distribution to a family of distributions~\citep{caprio2025conformal,caprio2025joyscategoricalconformalprediction}. Building on these connections, we develop a framework for uncertainty quantification~(UQ) based on a structural analysis of conformal credal sets: we connect classical conformal summaries to credal geometry, derive efficiently computable uncertainty measures, and characterise when these measures retain or lose the ability to distinguish test inputs.

Understanding what the quantified uncertainty represents is equally important. The uncertainty literature commonly distinguishes between uncertainty arising from randomness in the target (aleatoric uncertainty) and uncertainty arising from limited knowledge of the predictive mechanism (epistemic uncertainty)~\citep{hullermeier_aleatoric_2021}. We interpret our scores as summaries of \emph{total conformal uncertainty}: they reflect the combined effects of conditional label variability, model estimation, and finite calibration data, as captured by the chosen nonconformity score. Section~\ref{subsec: entropy_imprecision} examines this interpretation by analysing the core's finite-sample and limiting behaviour. Our analysis complements studies of aleatoric and epistemic uncertainty in CP~\citep{sale2025aleatoricconformal}, methods that incorporate model epistemic uncertainty into conformal scores~\citep{cruzcabezas2025epicscore}, and work on quantifying these two components for credal sets~\citep{gonzalezgarcia2026credal}. Separating these components is beyond the scope of this work.

\textbf{Contributions.}
Our contributions are fivefold.

\textbf{First}, we verify the credal representation and weak-IHDR correspondence for finite-label split CP, extending the correspondence established for full CP.

\textbf{Second}, we derive analytical characterisations and exact algorithms for established credal uncertainty measures in the case of CP, enabling linear-time evaluation after sorting the conformal contour and avoiding enumeration of exponentially many events.

\textbf{Third}, we unify several credal uncertainty measures and classical conformal summaries by integrating set-size penalties based on Tsallis entropies~\citep{tsallis1988generalization} along the IHDR path. 


\textbf{Fourth}, we prove an impossibility theorem for regression: under the conditions of Theorem~\ref{thm: regression_rank_core}, the conformal rank core is identical across inputs, so no fixed uncertainty functional of it can distinguish them, even when prediction regions vary in response space.

\textbf{Fifth}, classification experiments on selective prediction and distribution-shift detection show that credal summaries can recover useful ranking information lost by fixed-level prediction-set cardinality. Their relative performance against probability-vector summaries depends on the task and nonconformity score.

    \section{PRELIMINARIES}
\label{sec: prelim}

\textbf{Notation.}
Let $\cX$ be the input space, $\cY=\{y^{(k)}\}_{k=1}^K$ a finite label space, and $Z=(X,Y)$. We write $\cP(\cY)$ for the space of probability measures on $(\cY,2^{\cY})$. We start with conformal classification and explain extensions to regression in Section~\ref{subsec: regression_main}.

\subsection{Split conformal prediction}
\label{subsec: cp_intro}

We begin by stating the sampling requirement used throughout the paper. Let $\mathcal G_{\mathrm{tr}}$ be the $\sigma$-field generated by the training data and all randomness used to fit a classifier $\hat f:\cX\to\Delta_{K-1}$ and choose a nonconformity rule.

\begin{assumption}[Split-CP sampling]
\label{ass: split_exchangeability}
Conditional on $\mathcal G_{\mathrm{tr}}$, the calibration examples $Z_1,\ldots,Z_{n_{\mathrm{cal}}}$ and the future example $Z_{n_{\mathrm{cal}}+1}$ are exchangeable, and the same fixed measurable score $s=s_{\mathcal G_{\mathrm{tr}}}:\cX\times\cY\to\RR$ is applied to each example. In particular, neither calibration responses nor the future example are used to fit $\hat f$ or choose $s$.
If the score uses auxiliary randomisation, let $W_1,\ldots,W_{n_{\mathrm{cal}}+1}$ be i.i.d. random elements, conditionally independent of the data given $\mathcal G_{\mathrm{tr}}$; the augmented examples $(Z_i,W_i)$ are scored by the same rule.
\end{assumption}

Independent and identically distributed observations, with a training split selected independently of the data, satisfy Assumption~\ref{ass: split_exchangeability}. 
Write $s_i=s(X_i,Y_i)$ for the calibration scores. The non-randomised conformal p-value and prediction region are
\begin{equation}
\label{eq: split_pvalue}
\begin{aligned}
 \pi_x(y)&:=\frac{1+\sum_{i=1}^{n_{\mathrm{cal}}}\mathbf 1\{s_i\geq s(x,y)\}}{n_{\mathrm{cal}}+1},\\
 C_\alpha(x)&:=\{y\in\cY:\pi_x(y)>\alpha\}.
\end{aligned}
\end{equation}
Ties can make this rank construction conservative. Randomised rank tie-breaking can give exact calibration and should be distinguished from randomisation internal to a score such as Adaptive Prediction Sets~(APS)~\citep{romano2020classification}. Under Assumption~\ref{ass: split_exchangeability}, exchangeability of the $n_{\mathrm{cal}}+1$ scores gives, almost surely in $\mathcal G_{\mathrm{tr}}$,
$\PP\{Y_{n_{\mathrm{cal}}+1}\in C_\alpha(X_{n_{\mathrm{cal}}+1})\mid\mathcal G_{\mathrm{tr}}\}\geq1-\alpha$ for every $\alpha\in(0,1)$. Thus coverage is conditional on the training fit but marginal over calibration and test examples; it is not generally conditional on $X=x$.
The map $\pi_x:\cY\to[0,1]$ is known as the conformal transducer~\citep{vovk2018conformal}.

\textbf{Classical conformal scores and efficiency criteria.}
Conformal efficiency criteria are traditionally used to summarise and compare conformal predictors~\citep{vovk2017efficiency}. Two classical criteria are obtained by averaging \emph{conformal unconfidence} and \emph{conformal fuzziness} over test inputs. For an individual test input $x$, unconfidence is defined for $K\geq 2$ as the second-largest p-value, equivalently one minus conformal confidence~\citep{saunders1999transduction}, while fuzziness is defined as $F(x):=\sum_{y\in\cY}\pi_x(y)-\max_{y\in\cY}\pi_x(y)$. Their averages assess predictor efficiency without requiring a fixed significance level, with lower values preferred. In Section~\ref{subsec: MMI-CP}, we recover these individual scores through credal geometry and connect them to established uncertainty measures in imprecise probability. These connections give the scores a principled interpretation as summaries of conformal uncertainty, complementing their traditional role as ingredients of averaged efficiency criteria.



\subsection{Credal sets and uncertainty measures}
\label{subsec: credal_measures}

Imprecise probability~(IP) represents uncertainty through non-additive set functions and sets of probability measures~\citep{walley,augustin_introduction_2014}. We work with nonempty, closed, convex \emph{credal sets} $\cM\subseteq\cP(\cY)$. Their upper and lower probabilities are $\overline\PP(A):=\sup_{P\in\cM}P(A)$ and $\underline\PP(A):=\inf_{P\in\cM}P(A)$, respectively. Event-wise upper bounds also specify a set of compatible distributions, known as the \emph{core}.

\begin{definition}[Capacity, upper probability, and core]
\label{def: capacities}
A normalised monotone set function $\overline\PP:2^{\cY}\to[0,1]$ is a \emph{capacity}. It is an \emph{upper probability} when it is the upper envelope of a nonempty family of probability measures. Its conjugate (lower probability) is $\underline\PP(A):=1-\overline\PP(A^c)$, and its core is
$\cM(\overline\PP):=\{P\in\cP(\cY):P(A)\leq\overline\PP(A)\ \text{for all }A\subseteq\cY\}$.
\end{definition}

In possibility theory~\citep{dubois1985theorie}, the maxitive upper probability considered in Proposition~\ref{prop: consonance upper prob} is also known as a plausibility measure. A credal representation does not induce a unique scalar measure of uncertainty. Instead, different summaries capture different aspects of the uncertainty encoded by the credal set. We consider three such measures: nonspecificity, variation in entropy, and upper--lower expectation gaps. For a broader comparison of set-based uncertainty measures with distribution-based uncertainty summaries across a range of uncertainty-sensitive tasks, see \citet{wang2026set}.

\begin{definition}[Generalised Hartley]
\label{def: gh}
Let $m$ be the M\"obius inverse of the lower probability $\underline\PP$, $m(A):=\sum_{B\subseteq A}(-1)^{|A|-|B|}\underline\PP(B)$. For a core $\cM$ with M\"obius coefficients $m(A)$, Generalised Hartley~(GH) nonspecificity is
\begin{equation}
\label{eq: gh_general}
 \operatorname{GH}(\cM):=\sum_{\emptyset\neq A\subseteq\cY}m(A)\log_2|A|.
\end{equation}
\end{definition}
GH also applies to lower probabilities that are not belief functions, whose M\"obius coefficients can be negative~\citep{abellan2005additivity}. 


\begin{definition}[Entropy imprecision]
\label{def: entropy_imprecision}
For a credal set $\cM$, let $H(P):=-\sum_{y\in\cY}P(y)\log_2P(y)$, with $0\log_2 0:=0$. Its entropy imprecision is
\begin{equation}
\label{eq: entropy_imprecision_general}
 I_H(\cM):=\max_{P\in\cM}H(P)-\min_{P\in\cM}H(P).
\end{equation}
\end{definition}
This measures the range of Shannon entropy across compatible distributions~\citep{abellan,hullermeier2022credal}. 


\begin{definition}[Maximum mean imprecision]
\label{def: mmi}
For an upper probability $\overline\PP$ and a nonempty, uniformly bounded class $\cH$ of functions $h:\cY\to\RR$, Maximum Mean Imprecision~(MMI) is
\begin{equation}
\label{eq: mmi_general}
 \operatorname{MMI}_{\cH}(\overline\PP):=\sup_{h\in\cH}\left\{\circint h\,d\overline\PP-\circint h\,d\underline\PP\right\},
\end{equation}
where $\circint$ denotes the Choquet integral~\citep{choquet1953}, a generalisation of Lebesgue integral to non-additive measures. 
\end{definition}
MMI compares upper and lower Choquet expectations over the selected test functions~\citep{chau2025integral} and has recently been adapted to quantify verbalised imprecision of large language models~\citep{yang2026verbalizing}. In classification, the indicator class $\cH_{\mathrm{TV}}:=\{\mathbf1_A:A\subseteq\cY\}$ gives the largest event-wise probability gap, $\max_{A\subseteq\cY}\{\overline\PP(A)-\underline\PP(A)\}$.

These measures summarise different aspects of the credal representation. Their computation can be expensive, with naïve evaluation scaling exponentially with the number of outcomes. Section~\ref{sec: method} exploits the nested geometry of conformal cores to compute these summaries in linear time after sorting and to clarify their interpretation and limitations. Appendix~\ref{appendix: ipml} provides further background on imprecise probability and its use in machine learning.

    \section{FROM CONFORMAL CREDAL SETS TO UQ SCORES}
\label{sec: method}


\subsection{Split CP credal representation}
\label{subsec: splitCPIHDR}

Following \citet{cella} and \citet{caprio2025conformal}, our goal is to construct a \emph{credal representation} of CP from the complete conformal contour $\pi_x$: a set of probability distributions compatible with plausibility assessments derived from the conformal p-values. For the finite label space $\cY=\{1,\ldots,K\}$, we assign each nonempty event $A\subseteq\cY$ the plausibility of its most plausible label. Every compatible probability distribution $P$ must therefore satisfy $P(A)\leq\max_{y\in A}\pi_x(y)$. If the raw maximum $\rho_x:=\max_y\pi_x(y)$ is below one, no probability distribution can satisfy these bounds, since the constraint on the full label space would require $1=P(\cY)\leq\rho_x<1$. Normalising the contour so that its maximum equals one resolves this incompatibility and ensures that the resulting credal set is nonempty.

To obtain a nonempty credal representation, we follow \citet[Section~3.2]{martinbordini2023} and set the contour value at one maximising label to one. We select this label, denoted $y_x^\star$, using a fixed measurable rule to break ties, and define
\begin{equation}
\label{eq: consonance}
 \widetilde\pi_x(y_x^\star)=1,\qquad
 \widetilde\pi_x(y)=\pi_x(y)\quad(y\neq y_x^\star).
\end{equation}
Only the selected coordinate is modified; the contour values need not sum to one. For every $t\in[0,1)$, this adjustment leaves nonempty cuts unchanged and replaces an empty cut with $\{y_x^\star\}$. The cuts remain nested, as required for a consonant plausibility measure, and can only expand, so marginal coverage is preserved. The adjustment also removes the warnings conveyed by empty prediction sets. When the raw maximiser is unique, it retains all competing-label values but discards the raw maximum $\rho_x$, which can be reported separately as conformal credibility. Henceforth, $\pi_x$ denotes the normalised contour, and $C_t(x):=\{y\in\cY:\pi_x(y)>t\}$ its strict cut at level $t$.
\begin{restatable}{proposition}{propone}
\label{prop: consonance upper prob}
Define $\overline\PP_x(\emptyset):=0$ and $\overline\PP_x(A):=\max_{y\in A}\pi_x(y)$ for $A\neq\emptyset$. Then $\overline\PP_x$ is the upper envelope of its nonempty, convex, compact core $\cM_x:=\cM(\overline\PP_x)$.
\end{restatable}
This standard possibility-core construction~\citep{friedman1995plausibility} inherits event-wise conformal validity~\citep{cella}: the event that some $A$ contains the future label $Y$ but has $\overline\PP_X(A)\leq t$ is exactly $\{\pi_X(Y)\leq t\}$. Appendix~\ref{app: consonance_details} gives further preservation and validity arguments.


\subsection{Split conformal regions are IHDRs}
\label{subsec: ihdr_correspondence}

To connect the credal representation to CP, we recover the normalised conformal regions from the core's lower probability $\underline\PP_x$. Recall a set $A$ is a weak $(1-t)$-IHDR if $\underline\PP_x(A)\geq1-t$ and no proper subset of $A$ satisfies the same bound~\citep{coolen}. Allowing the lower probability to exceed $1-t$ accommodates the discrete grid of split-conformal p-values.

\begin{restatable}[Split-conformal regions as IHDRs]{proposition}{proptwo}
\label{prop: IHDR_equals_CPR}
For every $t\in[0,1]$, $C_t(x)$ is the unique weak $(1-t)$-IHDR of $\cM_x$, and
\begin{equation}
\label{eq: ihdr_intersection}
 C_t(x)=\bigcap\{A\subseteq\cY:\underline\PP_x(A)\geq1-t\}.
\end{equation}
\end{restatable}

This verifies that the IHDR correspondence established for full CP by \citet{caprio2025conformal} also holds for normalised finite-label split CP. Conversely, the core can be recovered from the probability bounds along the full path of conformal regions:
{\small
\begin{align}
\label{eq: core_from_ihdr_path}
 \cM_x=\bigl\{P\in\cP(\cY):
 P(C_t(x))\geq1-t,\ \forall t\in[0,1]\bigr\}.
\end{align}
}

Thus, the normalised contour, its nested conformal regions, and its credal core are equivalent representations (Appendix~\ref{sec: proofs}). This correspondence allows us to interpret normalised split CP as implicitly constructing a credal set and then extracting its prediction regions as IHDRs. We next study how to quantify the uncertainty encoded by this credal set and analyse the properties of the resulting summaries.


\subsection{Efficient uncertainty summaries for conformal cores}
\label{subsec: MMI-CP}

Although the core is defined by constraints on all $2^K$ events, its nested structure from Equation~\ref{eq: core_from_ihdr_path} reduces them to a chain and permits $\mathcal O(K)$ evaluation of the uncertainty measures after sorting the conformal contour. Choose a tie-consistent ordering with $y_{\sigma(1)}=y_x^\star$ and write $\alpha_j:=\pi_x(y_{\sigma(j)})$, so $1=\alpha_1\geq\cdots\geq\alpha_K$, with $\alpha_{K+1}:=0$. Set $F_j:=\{y_{\sigma(1)},\ldots,y_{\sigma(j)}\}$ and $m_j:=\alpha_j-\alpha_{j+1}$; the following representation underpins the subsequent derivations.
\begin{restatable}[Nested core representation]{proposition}{propcorechar}
\label{prop: core_characterisation}
Writing $p_j:=P(\{y_{\sigma(j)}\})$, a distribution $P$ belongs to $\cM_x$ if and only if $ \sum_{i=1}^{j}p_i\geq1-\alpha_{j+1},\qquad j=1,\ldots,K-1.$
Equivalently, $P=\sum_{j=1}^Km_jv^{(j)}$ for probability measures $v^{(j)}$ supported on $F_j$.
\end{restatable}

\textbf{TV-Maximum mean imprecision.}
To quantify disagreement about event probabilities among distributions in the core, we specialise Definition~\ref{def: mmi} to $\overline\PP_x$ and the class of event indicators $\cH_{\mathrm{TV}}$. The resulting TV-MMI equals the core's diameter under the total variation distance $d_{\mathrm{TV}}(P,Q):=\max_{A\subseteq\cY}|P(A)-Q(A)|$.

\begin{restatable}[TV diameter and conformal unconfidence]{proposition}{propthree}
\label{prop: mmi_tv}
For $K\geq2$ and $\cH_{\mathrm{TV}}=\{\mathbf1_A:A\subseteq\cY\}$,
\begin{equation}
\label{eq: tv_mmi}
\begin{aligned}
 \operatorname{MMI}_{\cH_{\mathrm{TV}}}(\overline\PP_x)
 &=\max_A\{\overline\PP_x(A)-\underline\PP_x(A)\}=\alpha_2.
\end{aligned}
\end{equation}
Moreover, $\alpha_2=\max_{P,Q\in\cM_x}d_{\mathrm{TV}}(P,Q)$.
\end{restatable}

Thus, \emph{conformal unconfidence}---the second-largest conformal p-value (Section~\ref{subsec: cp_intro})---has a direct credal interpretation: it equals both the core's TV diameter and the largest gap between its upper and lower probabilities over all events.

\textbf{Generalised Hartley nonspecificity.}
\label{subsec: ihdr_path}
Using the core characterisation from Proposition~\ref{prop: core_characterisation}, the M\"obius masses of $\underline\PP_x$ are $m(F_j)=m_j$ and zero off the nested chain. Substitution into Definition~\ref{def: gh} gives the following expression.
\begin{restatable}[Generalised Hartley]{proposition}{propghclosed}
\label{prop: gh_closed_form}
The GH nonspecificity of the conformal core is $\operatorname{GH}(\cM_x)=\sum_{j=1}^{K}(\alpha_j-\alpha_{j+1})\log_2j,$ which costs $\mathcal O(K)$ after sorting.
 
\end{restatable}

\textbf{Unifying summaries through the IHDR path.}
The IHDR path reveals a common structure behind GH and TV-MMI: both integrate penalties on prediction-region size across thresholds. In the GH expression, each weight $(\alpha_j-\alpha_{j+1})$ is the length of the threshold interval on which $C_t(x)$ contains exactly $j$ labels. GH therefore averages logarithmic region size over all thresholds. Similarly, TV-MMI equals $\alpha_2$, the length of the threshold interval on which the region contains more than one label. Its corresponding penalty is thus an indicator of multiple labels.

Integrating quantities along nested region paths is a common approach in statistics. For ordinary HDRs, examples include mass--volume analysis and the scoring of density level sets~\citep{clemencon2018massvolume,meng2025levelsets}. Building on this statistical practice and the exact IHDR representations of GH and TV-MMI, we replace their specific penalties with a general size-penalty function $g$, yielding \emph{integrated IHDR summaries}. This framework recovers existing uncertainty measures and provides a systematic way to construct additional summaries, with the choice of $g$ determining how strongly larger regions are penalised.

\begin{restatable}[Integrated IHDR size penalties]{proposition}{proppathfunctional}
\label{prop: ihdr_path_functionals}
For the regions $C_t$ in Proposition~\ref{prop: IHDR_equals_CPR}, let $g:\{0,\ldots,K\}\to\RR$ satisfy $g(0)=0$. Then
{\small
\begin{equation}
\label{eq: ihdr_path_functional}
 U_g(\cM_x):=\int_0^1g(|C_t(x)|)\,dt=\sum_{j=1}^K(\alpha_j-\alpha_{j+1})g(j).
\end{equation}
}
\end{restatable}
For nondecreasing penalties with $g(1)=0$, summation by parts gives $U_g=\sum_{j=2}^K\alpha_j\{g(j)-g(j-1)\}$. Thus the increment $g(j)-g(j-1)$ is precisely the weight on the $j$th-ranked candidate. Equal increments give every competing contour value the same weight; diminishing increments place relatively more weight on the leading labels.

\textbf{The Tsallis IHDR family.}
To obtain penalties that depend only on region size, we evaluate Tsallis entropy~\citep{tsallis1988generalization} at a reference distribution assigning equal probability $1/j$ to each of $j$ labels. For $q\geq0$, this gives
\begin{equation}
\label{eq: tsallis_size_penalty}
 g_q(j):=\begin{cases}
 \dfrac{j^{1-q}-1}{1-q},&q\neq1,\\[3pt]
 \ln j,&q=1,
 \end{cases}\qquad j\geq1.
\end{equation}
We set $g_q(0):=0$ by convention.


\begin{restatable}[Tsallis IHDR summaries]{corollary}{cortsallispath}
\label{cor: tsallis_path}
For $K\geq2$ and $q\geq0$, let $U_q:=U_{g_q}(\cM_x)/g_q(K)$. Then
\begin{equation}
\label{eq: tsallis_rank_weights}
 U_q=\sum_{j=2}^K\alpha_j w_q(j),\quad
 w_q(j)=\frac{\int_{j-1}^{j}u^{-q}\,du}{g_q(K)}.
\end{equation}
With fuzziness $F=\sum_{j=2}^K\alpha_j$ from Section~\ref{subsec: cp_intro} and $U_\infty:=\lim_{q\to\infty}U_q$,
\begin{equation}
\label{eq: tsallis_special_cases}
 U_0=\frac{F}{K-1},\quad U_1=\frac{\operatorname{GH}(\cM_x)}{\log_2K},\quad U_\infty=\alpha_2.
\end{equation}
Each $U_q$ lies in $[0,1]$, is monotone under inclusion of consonant cores on a fixed label space, and costs $\mathcal O(K)$ after sorting.
\end{restatable}
The proof is in Appendix~\ref{app:tsallis-proof}. The rank weights $w_q$ explain the role of $q$: $q=0$ weights all competing labels equally, $q=1$ discounts lower-ranked candidates, and increasing $q$ concentrates the relative weight further on the strongest competitors. In the limit, only the leading competitor contributes, recovering conformal unconfidence, or TV-MMI. Since $g_q(0)=g_q(1)=0$, replacing empty cuts by singletons leaves these integrated summaries unchanged.

The choice $q=2$ gives the \emph{saturating summary} $U_2$, with bounded penalty $g_2(j)=1-1/j$ and rapidly diminishing weights on lower-ranked competing labels. Appendix~\ref{app:saturation} derives its closed form and connection to reciprocal-cardinality specificity.

\textbf{Entropy imprecision.}
Although Entropy imprecision, an established credal uncertainty measure, is not an integrated IHDR summary, it benefits from the same nested-core representation, which enables exact computation in linear time after sorting. Since $\delta_{y_x^\star}\in\cM_x$, the minimum entropy in Definition~\ref{def: entropy_imprecision} is zero. Equation~\eqref{eq: core_from_ihdr_path} gives
\begin{equation}
\label{eq: entropy_ihdr_program}
\begin{aligned}
 I_H(\cM_x) =\max_{P\in\cP(\cY)}H(P)
\end{aligned}
\end{equation}
such that $P(C_t(x))\geq1-t,\quad\forall t\in[0,1].$
For normalised possibility contours, the following least-concave-majorant characterisation coincides to the upper-hull construction of the recent \citet[Section~4.2]{vu2026upper}. We provide a direct proof adapted to our nested-core representation for completeness.

\begin{restatable}[Entropy imprecision]{proposition}{propentropyclosed}
\label{prop: entropy_imprecision_closed_form}
Let $B_0:=0$, $B_j:=1-\alpha_{j+1}$, and let $G$ be the least concave majorant of $(j,B_j)_{j=0}^K$, i.e., the smallest concave function on $[0,K]$ with $G(j)\geq B_j$ for all $j$. With $q_j^\star:=G(j)-G(j-1)$, we have $I_H(\cM_x)=-\sum_{j=1}^Kq_j^\star\log_2q_j^\star.$
 
\end{restatable}
The vector $q^\star$ is obtained by applying the pool-adjacent-violators algorithm~(PAVA) to $(\alpha_j-\alpha_{j+1})_{j=1}^K$, pooling until block averages are nonincreasing. Thus $I_H$ costs $\mathcal O(K)$ after sorting.

\subsection{Regression: a common conformal core}
\label{subsec: regression_main}

The IHDR and core-path identities extend to measurable outcome spaces for measurable contours attaining one, with suprema replacing maxima (Appendix~\ref{app:measurable-ihdr}). In regression, mapping responses to conformal p-values discards the locations and widths of prediction regions. We establish an impossibility result for uncertainty summaries of this projected rank core. Let $\cY=\RR$ with Borel events and $\overline\PP_x(A)=\sup_{y\in A}\pi_x(y)$, with $\sup\emptyset=0$. Write $\cM_x$ for its core and $T_\#P$ for the law of $T(Y)$ under $Y\sim P$, with $T_\#\cM:=\{T_\#P:P\in\cM\}$. Absolute-residual scores already give $\pi_x(\hat f(x))=1$, requiring no stretch. More generally, the following range condition yields a common core on conformal ranks across inputs $x$.
\begin{restatable}[Common regression core on the rank grid]{theorem}{thmregressioncore}
\label{thm: regression_rank_core}
Let $S_{(1)}<\cdots<S_{(n)}$ be distinct calibration scores, with $n\geq1$, and $s(x,\cdot)$ be measurable. Suppose its range contains $[S_{(1)},\infty)$ for every $x$. Define $L_n=\{j/(n+1):j=1,\ldots,n+1\}$ and let $\mathcal Q_n$ be the core of $\overline R_n(B)=\max_{\ell\in B}\ell$ on $L_n$, with $\max\emptyset=0$. For the ordinary split-conformal contour, every input $x$ satisfies
\begin{equation}
\label{eq: regression_projected_core}
\begin{aligned}
 (\pi_x)_\#\cM_x&=\mathcal Q_n,\\
 \cM_x&=\{P\in\cP(\RR):(\pi_x)_\#P\in\mathcal Q_n\}.
\end{aligned}
\end{equation}
\end{restatable}
\begin{restatable}[Impossibility for summaries of the rank core]{corollary}{corregressionscores}
\label{cor: regression_constant_scores}
Under Theorem~\ref{thm: regression_rank_core}, let $K:=n+1$ and $D_t:=\{\ell\in L_n:\ell>t\}$. For any penalty $g:\{0,\ldots,K\}\to\RR$ with $g(0)=0$,
\begin{equation}
\label{eq: regression_path_functional}
 U_g(\mathcal Q_n)=\int_0^1g(|D_t|)\,dt=\frac1K\sum_{j=1}^K g(j).
\end{equation}
Thus no fixed uncertainty functional of the rank core can distinguish inputs. This applies to all $U_q$, including $U_\infty$, and entropy imprecision, for which $I_H(\mathcal Q_n)=\log_2K$.
\end{restatable}
These identities use counting measure on the rank grid; integrated region-volume summaries in response space need not be constant. The theorem covers absolute residuals, positive weighted residuals, and conformalised quantile regression whenever its range condition holds. 

The theorem formalises the invariance of the rank core and specifies the conditions under which it holds. For absolute-residual scores, this is consistent with the fact that finite prediction intervals have the same width across inputs at any fixed level. More generally, residual-score cores are affine images of a common core (Proposition~\ref{prop: regression_affine_core}). Their geometry in response space can therefore vary with $x$, as can MMI computed using functions sensitive to response location or scale. Distinguishing inputs through these credal representations requires retaining response-space geometry or using a construction whose rank core varies across inputs. Appendix~\ref{appendix: regression} provides proofs, qualifications, and implications for posterior-predictive scoring.

\subsection{Interpreting the core's uncertainty}
\label{subsec: uncertainty_interpretation}
\label{subsec: entropy_imprecision}

In finite-label classification, the identity between TV-MMI and the core's TV diameter (Proposition~\ref{prop: mmi_tv}) lets us examine whether conformal ambiguity disappears as data accumulate.
\begin{restatable}[Core collapse and finite-sample imprecision]{corollary}{corcorecollapse}
\label{cor: core_collapse}
For $K\geq2$, the normalised core $\cM_x$ is a singleton if and only if $\alpha_2=0$, in which case $\cM_x=\{\delta_{y_{\sigma(1)}}\}$. Moreover, for the ordinary split p-values in Equation~\eqref{eq: split_pvalue}, after max-to-one normalisation, $\operatorname{diam}_{\mathrm{TV}}(\cM_x)=\alpha_2\geq1/(n_{\mathrm{cal}}+1)>0$. Thus the core cannot collapse at any finite calibration-sample size, even with an oracle fitted predictor.
\end{restatable}
As $n_{\mathrm{cal}}\to\infty$, the finite-sample floor vanishes, but the limiting TV-MMI is positive whenever the limiting contour has a positive second-largest value. The core can therefore remain non-singleton even with an oracle predictor and population calibration. Our controlled classification study illustrates this: TV-MMI approaches a positive oracle/population reference with non-degenerate conditional labels and is zero in the deterministic-label population control (Appendix~\ref{sec:asymptotic_ambiguity_ablation}). This motivates interpreting the scores as \emph{total conformal uncertainty}: they reflect conditional label variability alongside model estimation and finite calibration, as represented through the chosen nonconformity score. They do not separate aleatoric and epistemic components.


    \section{RELATED WORK}
\label{sec: related work}

\textbf{CP and imprecise probability.}
Recent work has established several connections between CP and IP. A normalised full-conformal contour can induce a consonant possibility measure~\citep{cella}. For finite-label inductive CP, \citet{martinbordini2023} interpreted the raw p-value vector as a generally subnormalised possibility distribution. \citet{caprio2025conformal} established the CP--IHDR correspondence in a full-conformal formulation; \citet{caprio2025joyscategoricalconformalprediction} used a categorical argument to motivate numerical uncertainty quantification through the induced credal set.


\textbf{Conformal outputs as uncertainty rankings.}
CP is usually evaluated through coverage and prediction-set size at a selected level~\citep{conformal_tutorial,angelopoulosuncertainty2021}. \citet{karimi2023quantifying} propose uncertainty bounds based on set size normalised by the number of classes, with corrections for finite calibration samples and separate treatment of empty sets. These summaries retain the ties induced by cardinality at a fixed level. Classical confidence summaries~\citep{saunders1999transduction} and efficiency criteria~\citep{vovk2017efficiency} provide threshold-free alternatives. We examine their connection to credal uncertainty measures and their usefulness for downstream decisions.

\textbf{Credal uncertainty quantification.}
Our analysis draws on GH nonspecificity and entropy imprecision~\citep{hartley1928transmission,abellan2005additivity,abellan_measures_2006,sale2023volume}, MMI~\citep{chau2025integral}, and earlier studies of event-wise possibility--necessity gaps~\citep{shattuck2016fishburn}. We specialise these general IP constructions to conformal cores in Section~\ref{sec: method}.

    \section{WHAT DOES CONFORMAL UNCERTAINTY REVEAL?}
\label{sec: experiments}

We assess how credal summaries quantify conformal uncertainty beyond the size of a single prediction set. Comparisons with summaries of the original probability vector then test their relative practicality for selective prediction and distribution-shift detection.


\begin{table*}[t]
\centering
\caption{LAC: mean $\pm$ standard error over ten paired splits. Higher is better; set sizes use normalised contours. ImageNetV2 uses ResNetv2-50; C10-C averages all 75 corruption--severity conditions. \ranklegend}
\label{tab:main}
\begin{subtable}[t]{.49\textwidth}
\centering
\caption{Selective prediction: AUARC (\%).}
\label{tab: ARC}
\resizebox{\linewidth}{!}{
\begin{tabular}{@{}lcccc@{}}
\toprule
Summary & \makecell{CIFAR-\\100} & \makecell{Caltech-\\101} & \makecell{Tiny\\ImageNet} & ImageNetV2 \\
\midrule
\multicolumn{5}{@{}l}{\textit{Credal-based measures}} \\
$U_0$ ($F$) & $86.42_{\pm 0.27}$ & $97.56_{\pm 0.11}$ & $84.11_{\pm 0.48}$ & $82.36_{\pm 0.46}$ \\
$U_1$ (GH) & $\secondscore{87.31}_{\pm 0.29}$ & $\bestscore{97.82}_{\pm 0.11}$ & $\secondscore{85.14}_{\pm 0.46}$ & $\bestscore{88.07}_{\pm 0.32}$ \\
$U_2$ (sat.) & $87.13_{\pm 0.31}$ & $\secondscore{97.76}_{\pm 0.13}$ & $84.78_{\pm 0.45}$ & $87.56_{\pm 0.34}$ \\
$U_\infty$ (TV-MMI) & $86.29_{\pm 0.34}$ & $97.69_{\pm 0.14}$ & $83.46_{\pm 0.49}$ & $86.17_{\pm 0.39}$ \\
$I_{\mathrm H}$ & $87.26_{\pm 0.29}$ & $97.72_{\pm 0.13}$ & $85.06_{\pm 0.46}$ & $\secondscore{88.02}_{\pm 0.33}$ \\
\midrule
\multicolumn{5}{@{}l}{\textit{Probability-vector-based measures}} \\
$H(q)$ & $\bestscore{87.31}_{\pm 0.28}$ & $95.10_{\pm 0.17}$ & $85.05_{\pm 0.46}$ & $85.14_{\pm 0.41}$ \\
$1-q_{(1)}$ & $87.28_{\pm 0.31}$ & $96.73_{\pm 0.14}$ & $\bestscore{85.30}_{\pm 0.44}$ & $86.88_{\pm 0.38}$ \\
$q_{(2)}$ & $86.29_{\pm 0.34}$ & $97.68_{\pm 0.14}$ & $83.46_{\pm 0.49}$ & $86.17_{\pm 0.39}$ \\
\midrule
\multicolumn{5}{@{}l}{\textit{Set-size-based measures (normalised)}} \\
$|C_{0.01}|$ & $84.46_{\pm 0.27}$ & $95.96_{\pm 0.17}$ & $83.35_{\pm 0.49}$ & $82.86_{\pm 0.46}$ \\
$|C_{0.05}|$ & $86.06_{\pm 0.26}$ & $94.57_{\pm 0.36}$ & $84.07_{\pm 0.50}$ & $86.45_{\pm 0.29}$ \\
$|C_{0.10}|$ & $86.08_{\pm 0.30}$ & $91.38_{\pm 0.41}$ & $84.29_{\pm 0.51}$ & $87.06_{\pm 0.30}$ \\
\bottomrule
\end{tabular}}
\end{subtable}\hfill
\begin{subtable}[t]{.49\textwidth}
\centering
\caption{Distribution-shift detection: AUROC (\%).}
\label{tab:ood_detection}
\resizebox{\linewidth}{!}{
\begin{tabular}{@{}lcccc@{}}
\toprule
Ranking & C100 & C10.1 & Fashion & C10-C \\
\midrule
\multicolumn{5}{@{}l}{\textit{Credal-based measures}}\\
$U_0$ ($F$) & $\secondscore{86.08}_{\pm 0.08}$ & $59.84_{\pm 0.17}$ & $\bestscore{96.71}_{\pm 0.03}$ & $68.76_{\pm 0.08}$ \\
$U_1$ (GH) & $\bestscore{86.24}_{\pm 0.08}$ & $60.08_{\pm 0.17}$ & $\secondscore{95.72}_{\pm 0.03}$ & $\bestscore{68.93}_{\pm 0.08}$ \\
$U_2$ (sat.) & $86.01_{\pm 0.08}$ & $\secondscore{60.15}_{\pm 0.17}$ & $94.59_{\pm 0.04}$ & $\secondscore{68.87}_{\pm 0.09}$ \\
$U_\infty$ (TV-MMI) & $85.35_{\pm 0.09}$ & $60.09_{\pm 0.17}$ & $92.81_{\pm 0.06}$ & $68.58_{\pm 0.09}$ \\
$I_{\mathrm H}$ & $85.94_{\pm 0.09}$ & $\bestscore{60.17}_{\pm 0.17}$ & $94.24_{\pm 0.05}$ & $68.85_{\pm 0.09}$ \\
\midrule
\multicolumn{5}{@{}l}{\textit{Probability-vector-based measures}}\\
$H(q)$ & $85.74_{\pm 0.08}$ & $59.90_{\pm 0.18}$ & $94.83_{\pm 0.04}$ & $68.63_{\pm 0.09}$ \\
$1-q_{(1)}$ & $85.51_{\pm 0.09}$ & $59.89_{\pm 0.18}$ & $94.08_{\pm 0.05}$ & $68.55_{\pm 0.09}$ \\
$q_{(2)}$ & $85.36_{\pm 0.09}$ & $60.13_{\pm 0.17}$ & $92.81_{\pm 0.06}$ & $68.60_{\pm 0.09}$ \\
\midrule
\multicolumn{5}{@{}l}{\textit{Set-size-based measures (normalised)}}\\
$|C_{0.01}|$ & $82.80_{\pm 0.16}$ & $57.36_{\pm 0.14}$ & $95.43_{\pm 0.11}$ & $65.31_{\pm 0.10}$ \\
$|C_{0.05}|$ & $53.87_{\pm 0.66}$ & $50.54_{\pm 0.12}$ & $55.05_{\pm 1.05}$ & $51.64_{\pm 0.27}$ \\
$|C_{0.10}|$ & $50.00_{\pm 0.00}$ & $50.00_{\pm 0.00}$ & $50.00_{\pm 0.00}$ & $50.00_{\pm 0.00}$ \\
\bottomrule
\end{tabular}}
\end{subtable}
\end{table*}

\textbf{Comparisons and protocol.}
We compare the Tsallis IHDR summaries $U_0$ (fuzziness), $U_1$ (GH), $U_2$ (saturation), and $U_\infty$ (TV-MMI), together with entropy imprecision $I_H$, against $|C_\alpha|$ at $\alpha\in\{.01,.05,.10\}$. The family uses the common normalisation of Corollary~\ref{cor: tsallis_path}. These positive rescalings leave all reported ranking and correlation metrics unchanged. For the original probabilities $q_y(x):=\hat f_y(x)$, the predictive baselines are entropy $-\sum_yq_y\log_2q_y$, maximum-probability uncertainty $1-q_{(1)}$~\citep{hendrycks2017baseline}, and second probability $q_{(2)}$, where $q_{(1)}\geq q_{(2)}$ are the two largest entries. Larger values always indicate greater uncertainty. We use normalised cuts throughout the main tables; Appendices~\ref{app:selective-details} and~\ref{app:ood-details} also report raw cardinalities to distinguish profile information from the removal of empty sets.

The main construction is LAC~\citep{sadinle2019least}, $s(x,y)=1-q_y(x)$, chosen for its transparency: calibration applies a common monotone transform to each class probability before one leading contour coordinate is stretched to one. In particular, $U_\infty$ (TV-MMI) is a discretised monotone transform of $q_{(2)}$, so differences in its ranking arise only from calibration ties. $U_0$, $U_1$, $U_2$, and $I_H$ summarise the remaining profile differently. Appendices~\ref{app:score-structure} and~\ref{app:score-sensitivity} give these identities and comparisons with APS, RAPS, and SAPS. All comparisons share predictions and splits over ten seeds. Standard errors describe resampling variation conditional on a fixed fitted model, not variation across independently trained networks.

The experimental evaluations were run on an Ubuntu 24.04.3 LTS workstation with a 24-core Intel Core Ultra 9 285K CPU, 192\,GB RAM, and an NVIDIA GeForce RTX 5090 GPU with 32\,GB memory.

\textbf{Selective prediction.}
We evaluate CIFAR-100, Caltech-101, Tiny ImageNet, and ImageNetV2~\citep{KrizhevskyHinton2009TinyImages,FeiFei2004Caltech101,le2015tinyimagenet,recht2019imagenet}. Each split uses $5{,}000$ calibration examples ($2{,}677$ for Caltech-101) and $1{,}000$ disjoint evaluation examples. ImageNetV2 calibration and evaluation both use its target-domain pool. We use a rejection setup similar to \citet{garciagalindo2024reject}, rejecting inputs in decreasing uncertainty and evaluating accuracy on the retained predictions. We integrate retained accuracy over $0$--$90\%$ rejection, reporting normalised AUARC. Model details appear in Appendix~\ref{app:selective-models}.



In Table~\ref{tab: ARC}, both $U_1$ and $U_2$ exceed every tested fixed-level cardinality on all four datasets. Their paired gains over the strongest mean cardinality range from $0.85$ to $1.86$ percentage points for $U_1$ and from $0.50$ to $1.80$ for $U_2$. Relative to the strongest of these three probability baselines in each condition, $U_1$ gains $0.135\pm0.031$ points on Caltech-101 and $1.196\pm0.071$ on ImageNetV2, while essentially matching entropy on CIFAR-100. 




\textbf{What the choice of penalty retains.}
The Tsallis IHDR family makes the weighting behind these comparisons explicit. 
Under LAC, $U_2$ exceeds all three probability-vector baselines on Caltech-101 and ImageNetV2--ResNetv2-50, with paired AUARC gains over the strongest reported probability baselines of $0.077\pm0.017$ and $0.684\pm0.101$ points. $U_1$ retains an advantage over $U_2$ on the five conditions with at least 100 labels. Under APS in our ablation studies, $U_2$ instead improves on $U_1$ in all seven condition means. This reversal shows that the useful balance between leading competitors and less plausible candidates depends on the nonconformity score. Table~\ref{tab:tsallis-overview} in Appendix~\ref{app:tsallis-experiments} summarises the LAC--APS comparison; Tables~\ref{tab:tsallis-lac-paired} and~\ref{tab:tsallis-aps-paired} give paired differences, and Table~\ref{tab:tsallis-tail} reports the tail-truncation diagnostic.

\textbf{Distribution-shift detection.}
A fixed CIFAR-10 ResNet-56 supplies probabilities for clean examples and five shifted datasets, as well as all 75 CIFAR-10-C conditions~\citep{recht2018cifar10,hendrycks}. Each split calibrates on $5{,}000$ clean examples and evaluates balanced clean/shifted detection data. No shifted example is used as a calibration observation. Table~\ref{tab:ood_detection} shows four benchmarks; Appendix~\ref{app:ood-details} includes Tiny ImageNet and SVHN and details the protocol, including dependence between clean calibration images and their corrupted views.
Table~\ref{tab:ood_detection} again shows a cost to thresholding: $U_1$ exceeds every tested cardinality in each reported aggregate. $U_1$ has the highest mean on CIFAR-100 and CIFAR-10-C; $I_H$ leads narrowly on CIFAR-10.1, while $U_0$ leads on Fashion-MNIST. $U_2$ improves on $U_\infty$ across all four reported aggregates and exceeds all three probability-vector baselines on CIFAR-100 and CIFAR-10-C. $U_1$, $U_2$, and $I_H$ each exceed $1-q_{(1)}$ on all four reported aggregates. 



\textbf{Sensitivity and scope.}
Appendix~\ref{app:selective-details} repeats selective prediction and error/loss association for all five credal summaries with APS, RAPS, and SAPS, alongside LAC. OOD detection in Appendix~\ref{app:ood-score-ablation} includes saturation under LAC and APS; the four original summaries are evaluated under all four scores for OOD detection and active learning. Under the tested APS, RAPS, and SAPS settings, $1-q_{(1)}$ exceeds all five credal summaries in mean selective AUARC and error AUROC on all seven conditions. OOD comparisons differ: APS and RAPS retain credal summaries exceeding this baseline on four of six aggregates, while SAPS does not.




\textbf{Take-home message.}
Our aim is to use credal geometry to systematically characterise conformal uncertainty and extract information beyond set cardinality at a single error level. Our results demonstrate the value of this approach: in the main LAC experiments, summaries of the full conformal profile yield substantially better uncertainty rankings than fixed-level cardinality. Its relative performance against summaries of the original probability vector depends on the task, fitted model, and non-conformity score. These findings support a more informative assessment of conformal uncertainty achieveable through studying the underlying credal geometry; they do not establish it as the best representation of uncertainty available from the data, nor is that an objective of this work.

\textbf{Additional empirical results.}
The appendix provides complete experimental protocols and selective-prediction results across seven dataset/model conditions and four nonconformity scores, together with error-detection and loss-association analyses (Appendices~\ref{appendix: exp}--\ref{app:selective-details}). It also reports the full OOD benchmarks, all 75 corruption--severity conditions, and an analysis of calibration-size sensitivity (Appendix~\ref{app:ood-details}). Paired comparisons of the Tsallis penalties and a tail-truncation diagnostic examine how rank weighting and lower-ranked labels affect uncertainty rankings (Appendix~\ref{app:tsallis-experiments}). Finally, controlled experiments investigate whether conformal imprecision persists as training and calibration samples grow, while exploratory active-learning experiments assess the utility of these summaries for label acquisition (Appendices~\ref{sec:asymptotic_ambiguity_ablation}--\ref{app:active-learning}).
    \section{SUMMARY AND DISCUSSION}
\label{sec: discussion}

Conformal prediction provides set-valued predictions with coverage guarantees. These guarantees bound the probability of excluding the true outcome; quantifying the predictive uncertainty represented by the sets is a separate question. Set size offers one possible summary, but captures only a snapshot of the conformal profile at a chosen error level. Our credal-geometric analysis addresses this quantification problem, providing simple algorithms that extract information from the full profile and enable a finer assessment of conformal uncertainty. Our regression results also reveal a limitation: under the stated assumptions, summaries of the projected rank core cannot distinguish inputs. This motivates investigating summaries that retain relevant response geometry or conformal constructions whose rank cores vary across inputs. Such extensions should make explicit which changes in response location or scale matter for the intended decision, and establish conditions under which the resulting scores are finite and interpretable.

    \newpage

\subsection*{AI use statement}

We used generative AI tools to assist with checking mathematical arguments, suggesting proof strategies, and refining the formulation and presentation of selected mathematical claims. We also used these tools to provide feedback on experimental methodology, and extend experimental code. Additionally, we used generative AI tools to identify and examine relevant literature, revise portions of the manuscript, improve notation and readability, and prepare tables and figures.

We have not used generative AI tools for propose or refining hypotheses, design experiments, clean or reformat datasets and results interpretation. The rest of the required disclosure tasks are not applicable to this work.

We have reviewed all AI-assisted work. Checking the proofs and code were verified and tested. We take responsibility for the final content of this work, including text, claims, or artifacts produced with the aid of generative AI.

    \resumetoc

    \bibliography{references_aistats}
    \bibliographystyle{apalike}

    \section*{Checklist}

The checklist follows the references. For each question, choose your answer from the three possible options: Yes, No, Not Applicable.  You are encouraged to include a justification to your answer, either by referencing the appropriate section of your paper or providing a brief inline description (1-2 sentences).
Please do not modify the questions.  Note that the Checklist section does not count towards the page limit. Not including the checklist in the first submission won't result in desk rejection, although in such case we will ask you to upload it during the author response period and include it in camera ready (if accepted).

\begin{enumerate}

  \item For all models and algorithms presented, check if you include:
  \begin{enumerate}
    \item A clear description of the mathematical setting, assumptions, algorithm, and/or model. [Yes]
    \item An analysis of the properties and complexity (time, space, sample size) of any algorithm. [Yes]
    \item (Optional) Anonymized source code, with specification of all dependencies, including external libraries. [Yes]
  \end{enumerate}

  \item For any theoretical claim, check if you include:
  \begin{enumerate}
    \item Statements of the full set of assumptions of all theoretical results. [Yes]
    \item Complete proofs of all theoretical results. [Yes]
    \item Clear explanations of any assumptions. [Yes]
  \end{enumerate}

  \item For all figures and tables that present empirical results, check if you include:
  \begin{enumerate}
    \item The code, data, and instructions needed to reproduce the main experimental results (either in the supplemental material or as a URL). [Yes]
    \item All the training details (e.g., data splits, hyperparameters, how they were chosen). [Yes]
    \item A clear definition of the specific measure or statistics and error bars (e.g., with respect to the random seed after running experiments multiple times). [Yes]
    \item A description of the computing infrastructure used. (e.g., type of GPUs, internal cluster, or cloud provider). [Yes]
  \end{enumerate}

  \item If you are using existing assets (e.g., code, data, models) or curating/releasing new assets, check if you include:
  \begin{enumerate}
    \item Citations of the creator If your work uses existing assets. [Yes]
    \item The license information of the assets, if applicable. [Not Applicable]
    \item New assets either in the supplemental material or as a URL, if applicable. [Not Applicable]
    \item Information about consent from data providers/curators. [Yes]
    \item Discussion of sensible content if applicable, e.g., personally identifiable information or offensive content. [Not Applicable]
  \end{enumerate}

  \item If you used crowdsourcing or conducted research with human subjects, check if you include:
  \begin{enumerate}
    \item The full text of instructions given to participants and screenshots. [Not Applicable]
    \item Descriptions of potential participant risks, with links to Institutional Review Board (IRB) approvals if applicable. [Not Applicable]
    \item The estimated hourly wage paid to participants and the total amount spent on participant compensation. [Not Applicable]
  \end{enumerate}

\end{enumerate}

    \clearpage
    \appendix
    \thispagestyle{empty}
    \onecolumn

    \tableofcontents
    \clearpage
    \section{BACKGROUND ON IMPRECISE PROBABILISTIC MACHINE LEARNING}
\label{appendix: ipml}
\label{appendix: concepts}

\paragraph{Why allow imprecision?}
A precise predictive distribution specifies probabilities for outcomes, but it does not by itself describe how strongly the available evidence supports that probability assessment. For example, knowing that a coin is fair supports the single model $\operatorname{Bernoulli}(1/2)$. If no restriction on its bias is justified, one may instead retain the set $\{\operatorname{Bernoulli}(p):p\in[0,1]\}$. The first representation describes randomness under a specified model; the second also leaves the probability assessment unresolved. Assigning a predictive probability of $1/2$ in both situations would conceal this distinction. Imprecise probability~(IP) formalises such distinctions through sets or bounds on probabilities, allowing partial information to be expressed without selecting a single distribution~\citep{walley,augustin_introduction_2014}.

\paragraph{Representations of imprecise probability.}
Credal sets collect probability distributions compatible with specified information or assessments~\citep{levi2,walley}. Their lower and upper probabilities give the infimum and supremum of the probabilities assigned to each event across the set. In the coin example, the probability of heads is fixed at $1/2$ for the known fair coin but ranges over $[0,1]$ when its bias is unrestricted. Lower and upper expectations similarly describe the range of expected values supported by the set~\citep{decooman}, allowing decisions to account for more than one probability assessment. Other IP representations include belief functions, constructed from masses assigned to subsets of outcomes~\citep{shafer1976mathematical}, and possibility measures, under which an event is as plausible as its most plausible outcome~\citep{dubois1985theorie}. These event-wise summaries are examples of capacities: monotone set functions that need not be additive~\citep{choquet1953}. The different representations impose different structural assumptions.

\paragraph{Imprecise probabilistic machine learning.}
Imprecise probabilistic machine learning~(IPML) develops learning, inference, and decision-making with these richer uncertainty representations~\citep{chau2025integral}. A learning procedure may retain a family of predictions, or bounds on predictions, and assess how decisions change across the admissible alternatives. This perspective has supported inference from incomplete data~\citep{zaffalon2002exact} or overdetermine data~\citep{wang2026epistemic}, credal deep ensembles~\citep{wang2024credal},  and domain generalisation through imprecise learning~\citep{singh2024domain}. Related developments address tests between sets of distributions~\citep{chau2024credal} and the elicitation of imprecise forecasts~\citep{singhtruthful}. These examples illustrate two complementary tasks: constructing an informative uncertainty representation and extracting quantities useful for evaluation or action. A scalar uncertainty score summarises selected aspects of the representation, and its usefulness depends on the downstream task.

\paragraph{The connection to conformal uncertainty.}
CP connects to IP through its plausibility assessments and the compatible probability distributions they admit. Recent work identifies conformal prediction regions as IHDRs of the induced credal core~\citep{caprio2025conformal}. We verify this correspondence for normalised split-conformal classification and use the core's structure to derive tractable uncertainty summaries. This gives contour scores event-wise and geometric interpretations and unifies several summaries through the Tsallis IHDR family. Here the core represents compatibility with conformal bounds; it does not by itself identify the true conditional distribution or isolate epistemic uncertainty. We interpret the classification summaries as total conformal uncertainty, reflecting label variability, model estimation, and finite calibration through the chosen score (Section~\ref{subsec: uncertainty_interpretation}).


    \section{PROOFS AND DERIVATIONS}
\label{sec: proofs}

We restate each result before its proof, retaining its original number. Except for the measurable-outcome extension below, the label space is finite throughout this section; the contour is always normalised. For a fixed input, we use the ordered values and nested sets
\[
 \alpha_j=\pi_x(y_{\sigma(j)}),\qquad
 1=\alpha_1\geq\cdots\geq\alpha_K\geq\alpha_{K+1}=0,
\]
\[
 F_j=\{y_{\sigma(1)},\ldots,y_{\sigma(j)}\},\qquad
 m_j=\alpha_j-\alpha_{j+1},\qquad j=1,\ldots,K.
\]
We write $\cM_x$ for the probability core of the induced upper probability and $\underline\PP_x(A)=1-\overline\PP_x(A^c)$ for its conjugate lower probability.

\Needspace{8\baselineskip}
\subsection{Nonempty cores and upper envelopes}
\propone*

\begin{proof}
\emph{Nonemptiness.}
Choose a label $y_x^\star$ with $\pi_x(y_x^\star)=1$. The point mass at this label is dominated by the upper probability. Indeed, an event containing the label has upper probability one, while an event not containing it receives zero mass. Thus $\delta_{y_x^\star}\in\cM_x$.

\emph{The upper-envelope identity.}
By the definition of the core, every event satisfies $\sup_{P\in\cM_x}P(A)\leq\overline\PP_x(A)$. Equality is immediate for the empty event. If the event contains the chosen maximiser, equality is attained by its point mass.

Now consider a nonempty event that does not contain the maximiser. For each label $y$ in this event, define
\[
 P_y=\pi_x(y)\delta_y+\bigl(1-\pi_x(y)\bigr)\delta_{y_x^\star}.
\]
This probability measure belongs to the core. To see this, consider any event $B$. If it contains the maximiser, its upper probability is one and the constraint is automatic. Otherwise,
\[
 P_y(B)=
 \begin{cases}
  \pi_x(y),&y\in B,\\
  0,&y\notin B,
 \end{cases}
 \qquad\text{and hence}\qquad
 P_y(B)\leq\overline\PP_x(B).
\]
Because the original event excludes the maximiser, we obtain
\begin{align*}
 \sup_{P\in\cM_x}P(A)
 &\geq\max_{y\in A}P_y(A)\\
 &=\max_{y\in A}\pi_x(y)\\
 &=\overline\PP_x(A).
\end{align*}
Together with the opposite inequality, this proves the envelope identity.

\emph{Convexity and compactness.}
Every defining constraint is linear in the probability vector. Their intersection with the probability simplex is therefore convex and closed. Since the finite probability simplex is compact, the core is compact as well.
\end{proof}

\Needspace{8\baselineskip}
\subsection{IHDR correspondence and measurable outcome spaces}
\label{app:measurable-ihdr}
\proptwo*

\begin{proof}
\emph{Feasibility of the strict cut.}
The complement of the conformal cut contains precisely the labels whose contour values do not exceed the threshold. By conjugacy,
\begin{align*}
 \underline\PP_x(C_t(x))
 &=1-\overline\PP_x(C_t(x)^c)\\
 &=1-\max_{y:\pi_x(y)\leq t}\pi_x(y)\\
 &\geq1-t.
\end{align*}
Here and below, the maximum over an empty set is zero.

\emph{Minimality and uniqueness.}
Let another event satisfy the same lower-probability requirement, $\underline\PP_x(A)\geq1-t$. Conjugacy gives $\max_{y\in A^c}\pi_x(y)=\overline\PP_x(A^c)\leq t$. Thus no label with contour value strictly greater than the threshold can lie outside the event, so $C_t(x)\subseteq A$. The conformal cut is therefore feasible and contained in every feasible event. It is the unique inclusion-minimal feasible event, and intersecting all feasible events gives the same cut. This proves both claims, including the endpoint thresholds.

\emph{Recovery of the core (Equation~\eqref{eq: core_from_ihdr_path}).}
If $P$ belongs to the core, the IHDR identity just proved gives
\[
 P(C_t(x))\geq\underline\PP_x(C_t(x))\geq1-t,
 \qquad t\in[0,1].
\]
Conversely, suppose a probability measure satisfies all these path constraints. For any nonempty event, choose $t=\overline\PP_x(A)$. No label in that event has contour value strictly greater than this threshold, so $A\cap C_t(x)=\emptyset$. The assumed path constraint implies
\[
 P(A)\leq1-P(C_t(x))\leq t=\overline\PP_x(A).
\]
The empty-event constraint is automatic, so the measure belongs to the core.

\end{proof}

\paragraph{Extension to measurable outcome spaces.}
Let $(\cY,\mathcal A)$ be a measurable space and $\pi_x:\cY\to[0,1]$ a measurable contour with $\pi_x(y_x^\star)=1$. Define $\overline\PP_x(A)=\sup_{y\in A}\pi_x(y)$ for $A\in\mathcal A$, with $\sup\emptyset=0$, and let $\cM_x$ be its core of countably additive probability measures. The point mass $\delta_{y_x^\star}$ belongs to this core. For each $y\in\cY$, the two-point measure $P_y=\pi_x(y)\delta_y+\bigl(1-\pi_x(y)\bigr)\delta_{y_x^\star}$ is also dominated by $\overline\PP_x$, by the same event-wise argument used above. If $y_x^\star\notin A$, taking the supremum over $y\in A$ shows that $\sup_{P\in\cM_x}P(A)=\overline\PP_x(A)$; if $y_x^\star\in A$, equality is attained by $\delta_{y_x^\star}$. Thus the upper-envelope identity and its conjugate lower-envelope identity remain valid. In particular, for every measurable $A$ and $t\in[0,1]$,
\[
 \underline\PP_x(A)\geq1-t
 \quad\Longleftrightarrow\quad
 \sup_{y\in A^c}\pi_x(y)\leq t
 \quad\Longleftrightarrow\quad
 C_t(x)\subseteq A.
\]
The strict cut is measurable, so it is again the unique inclusion-minimal feasible event. The preceding proof of Equation~\eqref{eq: core_from_ihdr_path}, using $t=\sup_{y\in A}\pi_x(y)$, also applies unchanged. These conclusions require neither a finite outcome space nor compactness of the core.

For a fixed reference measure $\mu$, the map $t\mapsto\mu(C_t(x))$ is nonincreasing and hence measurable, with values in $[0,\infty]$. A Borel penalty $g$ therefore gives a measurable integrand, provided it is defined at every attained volume, including $\infty$ when necessary. Define the path summary
\[
 U_{g,\mu}(\cM_x):=\int_0^1g\bigl(\mu(C_t(x))\bigr)\,dt,
\]
which is finite whenever $\int_0^1|g(\mu(C_t(x)))|\,dt<\infty$. Counting measure gives cardinality and Lebesgue measure gives region volume. Volume summaries depend on the reference measure; their interpretation must be specified separately from the finite-label GH and uniform-distribution Tsallis identities derived below.

\Needspace{8\baselineskip}
\subsection{Nested core representation}
\propcorechar*

\begin{proof}
\emph{Necessity of the nested constraints.}
Let $P$ belong to the core and write $p_i=P(\{y_{\sigma(i)}\})$. Applying domination to the complement of each nested set gives
\[
 1-\sum_{i=1}^jp_i
 =P(F_j^c)
 \leq\overline\PP_x(F_j^c)
 =\alpha_{j+1},\qquad j=1,\ldots,K-1.
\]
Rearranging yields the claimed chain inequalities.

\emph{Sufficiency.}
Suppose that the chain inequalities hold. For a nonempty event, let $r=\min\{i:y_{\sigma(i)}\in A\}$. If $r>1$, the event excludes the first $r-1$ labels. Consequently,
\[
 P(A)
 \leq1-\sum_{i=1}^{r-1}p_i
 \leq\alpha_r
 =\overline\PP_x(A).
\]
If $r=1$, the upper probability is one, so domination follows from the fact that $P$ is a probability measure. The empty-event constraint is automatic. Thus the chain inequalities imply every constraint defining the core.

\emph{Construction of the mixture.}
The contour increments satisfy $\sum_{i=1}^jm_i=1-\alpha_{j+1}$ and $\sum_{i=1}^Km_i=1$. Take $U\sim\operatorname{Unif}(0,1)$ and use it to couple the two cumulative distributions:
\[
 I=\min\left\{i:\sum_{\ell=1}^ip_\ell\geq U\right\},\qquad
 J=\min\left\{j:\sum_{\ell=1}^jm_\ell\geq U\right\}.
\]
The chain inequalities say that the cumulative sums of the probability vector are at least those of the mass vector, so $I\leq J$ almost surely. For each positive mass, let $v^{(j)}=\mathcal L\bigl(y_{\sigma(I)}\mid J=j\bigr)$ be the conditional law of the selected label given the corresponding value of $J$. The coupling ensures that this law is supported on $F_j$. When the mass is zero, choose any probability measure supported on that set. The law of total probability now gives $P=\sum_{j=1}^Km_jv^{(j)}$.

Conversely, suppose such a mixture is given. For the index $r$ defined above, every focal set with index smaller than $r$ is disjoint from the event. Hence
\begin{align*}
 P(A)
 &=\sum_{j=r}^Km_jv^{(j)}(A)\\
 &\leq\sum_{j=r}^Km_j\\
 &=\alpha_r
 =\overline\PP_x(A).
\end{align*}
Thus every mixture of the stated form belongs to the core.
\end{proof}

\Needspace{8\baselineskip}
\subsection{TV-MMI and conformal unconfidence}
\propthree*

\begin{proof}
\emph{From event-wise imprecision to the TV diameter.}
Proposition~\ref{prop: consonance upper prob} identifies the upper probability with the upper envelope of the compact core. Conjugacy gives the corresponding lower envelope. Consequently,
\begin{align*}
 \max_{A\subseteq\cY}\bigl\{\overline\PP_x(A)-\underline\PP_x(A)\bigr\}
 &=\max_{A\subseteq\cY}\max_{P,Q\in\cM_x}\{P(A)-Q(A)\}\\
 &=\max_{P,Q\in\cM_x}\max_{A\subseteq\cY}\{P(A)-Q(A)\}\\
 &=\max_{P,Q\in\cM_x}d_{\mathrm{TV}}(P,Q).
\end{align*}
In the last step, taking complements shows that maximising the signed difference over events is equivalent to maximising its absolute value. For event indicators, the upper and lower Choquet integrals are the upper and lower event probabilities, so the first expression is also TV-MMI.

\emph{Evaluation of the event-wise maximum.}
Exactly one of an event and its complement contains the selected top-ranked label. That set has upper probability one. The other has upper probability at most the second-largest contour value. Therefore, for every event,
\begin{align*}
 \overline\PP_x(A)-\underline\PP_x(A)
 &=\overline\PP_x(A)+\overline\PP_x(A^c)-1\\
 &\leq\alpha_2.
\end{align*}
To attain this bound, take $A^\star=\{y_{\sigma(1)}\}$. Then $\overline\PP_x(A^\star)=1$ and $\overline\PP_x((A^\star)^c)=\alpha_2$, so the resulting gap is exactly the claimed value. The argument also covers tied maxima: in that case the second-largest value is one.
\end{proof}

\Needspace{8\baselineskip}
\subsection{Generalised Hartley nonspecificity}
\propghclosed*

\begin{proof}
Consider the belief function $\operatorname{Bel}_x(A)=\sum_{j:F_j\subseteq A}m_j$ defined by the nested focal sets and their contour increments. Its conjugate plausibility function is $\operatorname{Pl}_x(A)=\sum_{j:F_j\cap A\neq\emptyset}m_j$. For a nonempty event, let $r=\min\{i:y_{\sigma(i)}\in A\}$ be its first included label in the contour ordering. A focal set intersects this event exactly when its index is at least $r$. Thus
\begin{align*}
 \operatorname{Pl}_x(A)
 &=\sum_{j=r}^Km_j\\
 &=\sum_{j=r}^K(\alpha_j-\alpha_{j+1})\\
 &=\alpha_r
 =\max_{y\in A}\pi_x(y)
 =\overline\PP_x(A).
\end{align*}
Conjugacy identifies the belief function with the lower probability. Its only potentially nonzero M\"obius masses are therefore the masses on the nested focal sets; a tied contour value simply gives a zero mass.

Substituting these masses into the GH definition gives
\begin{align*}
 \operatorname{GH}(\cM_x)
 &=\sum_{j=1}^Km_j\log_2|F_j|\\
 &=\sum_{j=1}^K(\alpha_j-\alpha_{j+1})\log_2j.
\end{align*}
Once the contour is sorted, the sum requires one pass through its entries.
\end{proof}

\Needspace{8\baselineskip}
\subsection{Integrated IHDR size penalties}
\proppathfunctional*

\begin{proof}
\emph{Integration over the nested path.}
Apart from the finitely many contour levels, $C_t(x)=F_j$ whenever $\alpha_{j+1}<t<\alpha_j$. Tied values give intervals of zero length and therefore do not affect the integral. It follows that
\begin{align*}
 U_g(\cM_x)
 &=\int_0^1g(|C_t(x)|)\,dt\\
 &=\sum_{j=1}^K(\alpha_j-\alpha_{j+1})g(j).
\end{align*}
Expanding the differences and shifting the index yields
\begin{align*}
 \sum_{j=1}^K(\alpha_j-\alpha_{j+1})g(j)
 &=\alpha_1g(1)+\sum_{j=2}^K\alpha_j\{g(j)-g(j-1)\}
   -\alpha_{K+1}g(K)\\
 &=g(1)+\sum_{j=2}^K\alpha_j\{g(j)-g(j-1)\}.
\end{align*}
This proves the two equivalent expressions for the path functional.

\emph{Monotonicity under inclusion.}
Suppose that two consonant cores on the same label space satisfy $\cM_x\subseteq\cM'_x$. Their singleton upper envelopes are correspondingly ordered:
\[
 \pi_x(y)
 =\sup_{P\in\cM_x}P(\{y\})
 \leq\sup_{P\in\cM'_x}P(\{y\})
 =\pi'_x(y).
\]
Hence $C_t(x)\subseteq C'_t(x)$ for every $t\in[0,1]$. For nondecreasing $g$, apply the function to their cardinalities and integrate to obtain
\[
 U_g(\cM_x)
 =\int_0^1g(|C_t(x)|)\,dt
 \leq\int_0^1g(|C'_t(x)|)\,dt
 =U_g(\cM'_x).
\]
\end{proof}

\Needspace{8\baselineskip}
\subsection{Tsallis rank weights and the saturating summary}
\label{app:tsallis-proof}
\cortsallispath*

\begin{proof}
For $j\geq1$, the Tsallis entropy of the uniform distribution on $j$ labels is
\[
 H_q(\operatorname{Unif}(F_j))
 =\frac{1-\sum_{i=1}^j(1/j)^q}{q-1}
 =\frac{j^{1-q}-1}{1-q}=g_q(j),\qquad q\neq1.
\]
The continuous extension at $q=1$ is $\ln j$. For every $q\geq0$, $g_q(1)=0$ and its continuous extension to $u\geq1$ has derivative $u^{-q}$. Thus $g_q(j)-g_q(j-1)=\int_{j-1}^j u^{-q}\,du>0$ for $j\geq2$, and $g_q(K)=\int_1^K u^{-q}\,du>0$. Applying the summation-by-parts identity from the proof of Proposition~\ref{prop: ihdr_path_functionals} and dividing by $g_q(K)$ gives Equation~\eqref{eq: tsallis_rank_weights}. The weights are positive and sum to one. The same proof gives monotonicity under core inclusion because $g_q/g_q(K)$ is nondecreasing on $\{0,\ldots,K\}$. Evaluation requires one weighted sum after sorting.

For $q=0$, $g_0(j)=j-1$, so $U_0=F/(K-1)$. For $q=1$, $U_{g_1}=(\ln2)\operatorname{GH}(\cM_x)$ and $g_1(K)=\ln K$, giving $U_1=\operatorname{GH}(\cM_x)/\log_2K$. These are the first two identities in Equation~\eqref{eq: tsallis_special_cases}. For $q>1$,
\[
 \frac{g_q(j)}{g_q(K)}=\frac{1-j^{1-q}}{1-K^{1-q}}
 \longrightarrow\mathbf1\{j\geq2\},\qquad j=1,\ldots,K.
\]
The finite mass representation therefore gives
\[
 U_q=\sum_{j=1}^K m_j\frac{g_q(j)}{g_q(K)}
 \longrightarrow\sum_{j=2}^K m_j=\alpha_2=:U_\infty.
\]
Proposition~\ref{prop: mmi_tv} identifies $U_\infty$ with TV-MMI. The normalisation is necessary: the unnormalised integral $U_{g_q}$ tends to zero as $q\to\infty$.

To formalise the concentration on leading competitors, write
\[
 w_q(j)=\frac{\int_{j-1}^j u^{-q}\,du}{\int_1^K u^{-q}\,du}.
\]
For $q'>q$, the ratio of the corresponding densities on $[1,K]$ is proportional to $u^{-(q'-q)}$, a decreasing function of $u$. Increasing $q$ therefore shifts weight towards smaller ranks: $\sum_{j=2}^r w_q(j)=g_q(r)/g_q(K)$ is nondecreasing in $q$ for each $r=2,\ldots,K$.

\end{proof}

\paragraph{Saturation and reciprocal-cardinality specificity.}
\phantomsection
\label{app:saturation}
At $q=2$, the bounded penalty is $g_2(j)=1-1/j$ and the incremental weights are $w_2(j)=1/\{(1-1/K)j(j-1)\}$. Substituting these weights into Equation~\eqref{eq: tsallis_rank_weights} gives the \emph{saturating summary}
\begin{equation}
\label{eq: saturation_score}
 U_2(\cM_x)=\frac{1}{1-1/K}\sum_{j=2}^K\frac{\alpha_j}{j(j-1)}.
\end{equation}
Every competing contour value contributes, with rapidly diminishing weight on lower-ranked labels. The classical reciprocal-cardinality specificity functional~\citep{yager1983specificity,bodjanova2007granulation} is
\[
 \operatorname{Sp}(\pi_x)=\int_0^1|C_t(x)|^{-1}\,dt=\sum_{j=1}^K\frac{m_j}{j}.
\]
Since $\sum_jm_j=1$, the equivalent mass representation gives
\[
 U_2=\frac{1-\sum_jm_j/j}{1-1/K}
 =\frac{1-\operatorname{Sp}(\pi_x)}{1-1/K}.
\]
The specificity integral is understood up to the measure-zero endpoint $t=1$: the strict cuts are nonempty for $t<1$, and their finitely many boundary levels do not affect the integral. This identity brings classical specificity into conformal UQ through the same IHDR representation. All finite-$q$ penalties vanish at both zero and one, so replacing empty cuts by singletons leaves these path integrals unchanged.

\Needspace{8\baselineskip}
\subsection{Entropy imprecision}
\propentropyclosed*

\begin{proof}
\emph{Minimum entropy.}
The normalised core contains a point mass, whose Shannon entropy is zero. Since Shannon entropy is nonnegative on a finite probability simplex, $\min_{P\in\cM_x}H(P)=H(\delta_{y_{\sigma(1)}})=0$. Thus entropy imprecision equals the maximum entropy over the core.

\emph{Construction of a feasible maximiser.}
Apply PAVA to the contour increments $m_j=\alpha_j-\alpha_{j+1}$, $j=1,\ldots,K$, pooling adjacent blocks whenever the left block has smaller average than the right block. The resulting vector consists of the successive slopes $q_j^\star=G(j)-G(j-1)$ of the least concave majorant. Pooling preserves the total mass and produces nonnegative, nonincreasing slopes. Hence
\[
 q_1^\star\geq\cdots\geq q_K^\star\geq0,
 \qquad \sum_{j=1}^Kq_j^\star=1.
\]
Because the majorant lies above the original cumulative-mass points,
\[
 \sum_{i=1}^jq_i^\star
 =G(j)
 \geq B_j
 =1-\alpha_{j+1},\qquad j=1,\ldots,K-1.
\]
Proposition~\ref{prop: core_characterisation} therefore places the distribution with these ordered-label masses in the core.

\emph{Entropy optimality.}
Take any probability measure in the core and let $p^\downarrow$ be the decreasing rearrangement of its ordered-label masses. The sum of its largest $j$ entries is at least the sum of its first $j$ entries in the original contour ordering. Consequently,
\[
 \sum_{i=1}^jp_i^\downarrow
 \geq\sum_{i=1}^jp_i
 \geq B_j,\qquad j=1,\ldots,K.
\]
The cumulative-sum curve of the rearranged vector is concave and lies above every point defining $G$. By minimality of the least concave majorant,
\[
 \sum_{i=1}^jp_i^\downarrow
 \geq G(j)
 =\sum_{i=1}^jq_i^\star,\qquad j=1,\ldots,K,
\]
with equality of the total sums. This is precisely the statement that the rearranged vector majorises the proposed maximiser. Shannon entropy is invariant under permutation and Schur-concave, so $H(P)=H(p^\downarrow)\leq H(q^\star)$. The feasible distribution constructed above therefore attains maximum entropy. Strict concavity of Shannon entropy on the convex core makes the maximiser unique. Combining this with the zero minimum gives the stated entropy-imprecision formula.

\emph{Computational cost.}
PAVA can maintain adjacent blocks so that each is created once and removed at most once by merging. Its cost is linear in the number of labels after the contour has been sorted.
\end{proof}

\Needspace{8\baselineskip}
\subsection{Core collapse and finite-sample imprecision}
\corcorecollapse*

\begin{proof}
Proposition~\ref{prop: mmi_tv} gives $\operatorname{diam}_{\mathrm{TV}}(\cM_x)=\alpha_2$. Since TV is a metric and the core is nonempty, zero diameter is equivalent to the core containing exactly one probability measure. Moreover, the normalised core always contains the point mass at the selected maximiser. Hence
\[
 \alpha_2=0
 \quad\Longleftrightarrow\quad
 \cM_x=\{\delta_{y_{\sigma(1)}}\}.
\]

For the ordinary split-conformal construction, the plus-one correction ensures that every raw p-value satisfies $\pi_x^{\mathrm{raw}}(y)\geq1/(n_{\mathrm{cal}}+1)$. The max-to-one stretch does not decrease any coordinate. With at least two labels, the second-largest normalised value therefore satisfies the same bound:
\[
 \operatorname{diam}_{\mathrm{TV}}(\cM_x)
 =\alpha_2
 \geq\frac{1}{n_{\mathrm{cal}}+1}>0.
\]
This conclusion holds for every fitted predictor, including an oracle predictor.
\end{proof}

\subsection{Max-to-one normalisation and validity}
\label{app: consonance_details}

The max-to-one normalisation is established prior work; we record its exact effect under our strict-cut convention.

\paragraph{Effect of the max-to-one stretch.}
Let the stretched contour be defined by Equation~\eqref{eq: consonance}, and write $\widetilde C_\alpha(x)=\{y:\widetilde\pi_x(y)>\alpha\}$. For thresholds below one, the stretched cuts satisfy
\[
 \widetilde C_\alpha(x)=
 \begin{cases}
  C_\alpha(x),&C_\alpha(x)\neq\emptyset,\\
  \{y_{\sigma(1)}\},&C_\alpha(x)=\emptyset,
 \end{cases}
 \qquad \alpha\in[0,1).
\]
Under Assumption~\ref{ass: split_exchangeability}, their training-conditional coverage is
\[
 \PP\!\left\{Y_{n_{\mathrm{cal}}+1}\in
 \widetilde C_\alpha(X_{n_{\mathrm{cal}}+1})
 \mid\mathcal G_{\mathrm{tr}}\right\}\geq1-\alpha
 \qquad\text{almost surely}.
\]
\begin{proof}
Only the selected maximiser changes. If the original cut is nonempty, it already contains that label. Otherwise, the stretch adds exactly that label, since the threshold is below one. This proves the displayed description of the cuts.

Pointwise domination of the original contour also gives
\[
 \widetilde\pi_x(y)\geq\pi_x(y)
 \quad\Longrightarrow\quad
 C_\alpha(x)\subseteq\widetilde C_\alpha(x).
\]
Taking conditional probabilities of the corresponding coverage events and using the original split-conformal guarantee from Section~\ref{subsec: cp_intro} proves the coverage claim.
\end{proof}

\paragraph{Effect on the scores.}
Write the raw ordered values and the raw cut-path integral as
\[
 \beta_1=\rho_x\geq\beta_2\geq\cdots\geq\beta_K,
 \qquad U_g^{\mathrm{raw}}(x)=\int_0^1g(|C_t(x)|)\,dt,
 \qquad g(0)=0.
\]
The stretch changes only the part of the path where the original cut is empty. On that interval, the stretched cut is a singleton. Therefore,
\begin{align*}
 U_g(\cM_x)-U_g^{\mathrm{raw}}(x)
 &=\int_{\rho_x}^1\{g(1)-g(0)\}\,dt\\
 &=(1-\rho_x)g(1).
\end{align*}
The TV-MMI/unconfidence, fuzziness, and GH choices all satisfy $g(1)=0$. Their normalised formulae consequently agree with their raw cut-profile formulae. This is an equality of numerical summaries; there is no raw credal core when the raw maximum is below one.

The nested constraints also show that entropy imprecision depends only on the retained coordinates after normalisation, $\alpha_2=\beta_2,\ldots,\alpha_K=\beta_K$. It likewise has no pre-stretch credal value when the raw core is empty.

Thus, with a unique maximiser, contours agreeing below their maximum induce the same normalised core summaries even if their raw credibility differs. Under a tied maximum, its numerical height remains among the retained coordinates, while the chosen anchor affects the labelled core only up to relabelling.

\paragraph{Event-wise form of conformal p-value validity~\citep{cella}.}
Let a normalised conformal contour induce a consonant plausibility measure and satisfy
\[
 \PP\!\left\{\pi_{X_{n_{\mathrm{cal}}+1}}(Y_{n_{\mathrm{cal}}+1})\leq\alpha
 \mid\mathcal G_{\mathrm{tr}}\right\}\leq\alpha
 \qquad\text{almost surely}.
\]
Then it also satisfies the event-wise bound
\[
 \PP\!\left\{\exists A\subseteq\cY:
 Y_{n_{\mathrm{cal}}+1}\in A,\quad
 \overline\PP_{X_{n_{\mathrm{cal}}+1}}(A)\leq\alpha
 \mid\mathcal G_{\mathrm{tr}}\right\}\leq\alpha
 \qquad\text{almost surely}.
\]
\begin{proof}
For a realised input and label, suppose there exists an event containing the label whose upper probability is at most the threshold. Then
\[
 \pi_x(y)\leq\max_{z\in A}\pi_x(z)
 =\overline\PP_x(A)\leq\alpha.
\]
Conversely, if the contour value at the realised label is at most the threshold, its singleton $A=\{y\}$ is such an event, since $\overline\PP_x(A)=\pi_x(y)\leq\alpha$. The event in the claimed bound is therefore identical to the event in the assumed p-value-validity statement. Their conditional probabilities are equal, which proves the result.
\end{proof}

    \Needspace{10\baselineskip}
\section{REGRESSION: RANK-CORE INVARIANCE AND RESPONSE-SPACE GEOMETRY}
\label{appendix: regression}

We restate and prove Theorem~\ref{thm: regression_rank_core} and Corollary~\ref{cor: regression_constant_scores}, retaining their original numbers. We then collect the score-range conditions for residual, quantile, and posterior-predictive scores before distinguishing the common rank-grid core from the core on the original response space.

\Needspace{10\baselineskip}
\subsection{Common core on the rank grid}
\label{app: regression_rank_core}
\thmregressioncore*

We work on the real line with its Borel $\sigma$-field and countably additive probability measures. For a measurable score, the ordinary split-conformal contour is
\begin{align}
\label{eq: regression_transducer}
 \pi_x(y)=\frac{1+\sum_{i=1}^n\mathbf1\{S_i\geq s(x,y)\}}{n+1}.
\end{align}
The corresponding response-space core is
\[
 \cM_x=\left\{P\in\cP(\RR):
 P(A)\leq\overline\PP_x(A)
 \text{ for every }A\in\mathcal B(\RR)\right\},
\]
where the upper probability is the supremum of the contour over the event, with value zero for the empty event.

\begin{proof}
\emph{Every conformal rank is attained.}
Fix a test input. By assumption, the score range contains $[S_{(1)},\infty)$. At the smallest calibration score, all calibration scores enter the numerator of Equation~\eqref{eq: regression_transducer}. As the candidate score crosses each distinct calibration threshold, the number of included scores decreases by one. Above the largest threshold, none remain. The attained contour values are therefore exactly $1,n/(n+1),\ldots,1/(n+1)$. Thus the contour maps the response space onto the full rank grid.

\emph{Domination passes to the rank grid.}
Let a probability measure belong to the response-space core. For any subset of the finite grid, its inverse image is Borel because the contour is measurable. Hence
\begin{align*}
 (\pi_x)_\#P(B)
 &=P(\pi_x^{-1}(B))\\
 &\leq\sup_{y\in\pi_x^{-1}(B)}\pi_x(y)\\
 &=\max_{\ell\in B}\ell
 =\overline R_n(B).
\end{align*}
The last equality uses surjectivity; for the empty event, both sides are zero by convention. Thus $P\in\cM_x$ implies $(\pi_x)_\#P\in\mathcal Q_n$.

\emph{Domination on the rank grid is sufficient.}
Conversely, suppose the pushforward belongs to the grid core. For any Borel event $A$, its image is a subset of a finite set and is therefore measurable. Moreover, $A\subseteq\pi_x^{-1}(\pi_x(A))$, so
\begin{align*}
 P(A)
 &\leq(\pi_x)_\#P(\pi_x(A))\\
 &\leq\max_{\ell\in\pi_x(A)}\ell\\
 &=\overline\PP_x(A).
\end{align*}
Again, the empty event is covered by the zero-maximum convention. This proves the characterisation
\[
 \cM_x=\{P\in\cP(\RR):(\pi_x)_\#P\in\mathcal Q_n\}.
\]

\emph{Every distribution in the grid core has a lift.}
Take any distribution in the grid core and write its masses as $q_j=Q(\{j/(n+1)\})$, $j=1,\ldots,n+1$. Surjectivity allows us to choose a representative response $y_j$ with $\pi_x(y_j)=j/(n+1)$ for each grid point. For this fixed input, define the finite-support probability measure $P=\sum_{j=1}^{n+1}q_j\delta_{y_j}$. Its pushforward is precisely $Q$. The characterisation just proved places this lift in the response-space core. We have therefore shown both inclusions in $(\pi_x)_\#\cM_x=\mathcal Q_n$.
\end{proof}

For the uncertainty calculations below, it is useful to express the grid core through cumulative masses in increasing rank order:
\begin{align}
\label{eq: regression_rank_constraints}
 \mathcal Q_n=\left\{q\in\Delta_n:
 \sum_{j=1}^kq_j\leq\frac{k}{n+1},\quad k=1,\ldots,n\right\}.
\end{align}
Necessity follows by applying domination to each initial segment $B_k=\{1/(n+1),\ldots,k/(n+1)\}$, whose upper probability is $\overline R_n(B_k)=k/(n+1)$. Conversely, every nonempty grid event is contained in the initial segment ending at its largest element. Domination of that initial segment therefore implies domination of the event. The full-grid constraint is already enforced by the probability simplex.

\Needspace{10\baselineskip}
\subsection{Constant uncertainty summaries on the rank grid}
\corregressionscores*

\begin{proof}
\emph{Constancy of fixed core functionals.}
Theorem~\ref{thm: regression_rank_core} identifies every projected core with the same set, $(\pi_x)_\#\cM_x=\mathcal Q_n$. Applying any fixed functional therefore gives a value independent of the input. We now evaluate the integrated summaries and entropy imprecision.

\emph{Integrated IHDR summaries.}
Write $K=n+1$ and $D_t=\{\ell\in L_n:\ell>t\}$. For each $j=1,\ldots,K$, the region $D_t$ has cardinality $j$ when $(K-j)/K\leq t<(K-j+1)/K$. Every cardinality therefore occupies a threshold interval of length $1/K$. For any penalty $g$ with $g(0)=0$, integration gives
\[
 U_g(\mathcal Q_n)
 =\int_0^1g(|D_t|)\,dt
 =\frac1K\sum_{j=1}^Kg(j).
\]
Equivalently, the decreasing grid-contour values are $\alpha_j=(K+1-j)/K$, with $\alpha_{K+1}=0$, so every nested focal mass in Proposition~\ref{prop: ihdr_path_functionals} equals $1/K$.

\emph{The Tsallis IHDR family.}
Dividing by $g_q(K)$ yields
\begin{equation}
\label{eq: regression_constant_scores}
 U_q(\mathcal Q_n)=\frac{\sum_{j=1}^Kg_q(j)}{K g_q(K)},
 \qquad 0\leq q<\infty.
\end{equation}
Writing $h_K:=\sum_{j=1}^Kj^{-1}$, substitution of the linear, logarithmic, and reciprocal-cardinality penalties gives
\begin{align*}
 U_0(\mathcal Q_n)&=\frac12,\\
 U_1(\mathcal Q_n)&=\frac{\ln(K!)}{K\ln K},\\
 U_2(\mathcal Q_n)&=\frac{K-h_K}{K-1}.
\end{align*}
By Corollary~\ref{cor: tsallis_path}, the limiting family member is the second-largest grid-contour value:
\[
 U_\infty(\mathcal Q_n)=\frac{K-1}{K}
 =\operatorname{MMI}_{\cH_{\mathrm{TV}}}(\overline R_n).
\]
Thus the general integral representation covers normalised fuzziness, GH, saturation, and TV-MMI simultaneously.

\emph{Entropy imprecision.}
The uniform distribution on the grid has masses $q_j^{\mathrm{unif}}=1/(n+1)$, $j=1,\ldots,n+1$. It satisfies every cumulative constraint in Equation~\eqref{eq: regression_rank_constraints} with equality. It therefore belongs to the core and attains the largest entropy possible on a space with $n+1$ points, $\max_{Q\in\mathcal Q_n}H(Q)=\log_2(n+1)$. The point mass at grid level one also belongs to the core and attains zero entropy. Consequently, $I_H(\mathcal Q_n)=\log_2(n+1)-0=\log_2(n+1)$.

\end{proof}

\subsection{Scope of the score-range assumption}

\paragraph{Residual and quantile scores.}
The absolute-residual and positive weighted-residual scores are $s(x,y)=|y-\hat f(x)|$ and $s(x,y)=|y-\hat f(x)|/w(x)$ with $w(x)>0$, respectively. Both have range equal to the nonnegative real line, so they satisfy the score-range condition for their nonnegative calibration scores.

For conformalised quantile regression, the score is $s(x,y)=\max\{\hat f_\ell(x)-y,\,y-\hat f_u(x)\}$. Its range starts at the minimum $m_x=(\hat f_\ell(x)-\hat f_u(x))/2$ of this two-sided piecewise-linear function. The theorem therefore requires $m_x\leq S_{(1)}$, which need not hold at every test input. Weighted residuals and quantile regression can still produce input-dependent prediction-interval widths through the scale or fitted quantile functions. The common rank-grid core summarises the attainable conformal ranks; it does not retain this response-space geometry.

\paragraph{Posterior-predictive scores.}

Posterior-predictive scoring leads to two distinct constructions. In the conformal Bayes method of \citet{fong2021conformal}, the posterior is conditioned on the candidate-augmented dataset and the conformity scores take the form
\begin{align*}
 \sigma_i^{(y)}
 =p\!\left(Y_i\mid X_i,\mathcal D\cup\{(x,y)\}\right),
 \qquad i=1,\ldots,n+1.
\end{align*}
This is \emph{full} CP: changing the candidate response changes both the posterior and every score. Theorem~\ref{thm: regression_rank_core}, however, relies on fixed calibration scores. It therefore does not cover conformal Bayes. Membership of each rank in the usual conformal grid alone also cannot establish that every grid value is attained. Establishing constancy of its rank-core summaries requires a separate analysis of this candidate-dependent rank process.

A frozen posterior-predictive score used in \emph{split} CP is different. Define $c_x(y)=\widehat p(y\mid x,\mathcal D_{\mathrm{tr}})$ and $C_i=c_{X_i}(Y_i)$. Use this density as a conformity score, or equivalently use the nonconformity score $s(x,y)=-\log c_x(y)$. Recent split Bayesian conformal constructions use precisely this fixed-score form~\citep{wu_bayesian_cp}. Write the peak density as $c_x^\star=\sup_yc_x(y)$. If the density attains this peak and its range contains all positive levels up to the peak, the condition
\begin{align}
\label{eq: bayes_density_constant_condition}
 c_x^\star\geq\max_{1\leq i\leq n}C_i
\end{align}
ensures that the candidate conformity score crosses every calibration threshold. In nonconformity terms, the minimum candidate score is no greater than the smallest calibration score, $\inf_y s(x,y)\leq S_{(1)}$. Together with the score range extending arbitrarily high, this gives the range condition in Theorem~\ref{thm: regression_rank_core}. The rank-grid uncertainty scores are then the constants in Corollary~\ref{cor: regression_constant_scores}.

The peak condition is sufficient under the stated range assumption, but it is not automatic. For a Gaussian posterior predictive,
\[
 Y\mid x,\mathcal D_{\mathrm{tr}}\sim\mathcal N(\mu_x,\sigma_x^2),
 \qquad c_x^\star=\frac{1}{\sqrt{2\pi}\,\sigma_x}.
\]
With constant posterior-predictive scale, every calibration density is bounded by the same peak, and the constant-profile problem occurs. If the scale varies with the input, a diffuse test predictive can have a lower peak than some calibration densities. The candidate score then fails to cross every calibration threshold, so exact constancy is no longer guaranteed.

Multimodality can substantially change the geometry of the prediction region. By itself, however, it does not remove the constant-profile phenomenon when the density still traverses every calibration level. The conclusion therefore depends on how posterior-predictive scoring is incorporated: the published Fong--Holmes construction falls outside the split-conformal theorem, whereas a split, frozen posterior-predictive score can exhibit the same constant rank-grid core under the displayed sufficient condition.

\subsection{Response-space cores and limits of invariance}
\label{app: regression_response_core}

Equality of the projected cores does not imply equality of the response-space cores: the sets of responses sharing each p-value can change with the input. For the usual residual scores, this change has an exact description.

\Needspace{8\baselineskip}
\begin{proposition}[Affine transport of residual cores]
\label{prop: regression_affine_core}
Suppose the score is $s(x,y)=|y-\hat f(x)|/w(x)$ with $w(x)>0$, including the unweighted case. Define the reference contour and affine map by
\[
 \pi_0(u)=\frac{1+\sum_{i=1}^n\mathbf1\{S_i\geq|u|\}}{n+1},
 \qquad T_x(u)=\hat f(x)+w(x)u,
\]
and let $\cM_0$ be the reference contour's core on the real line. Then
\begin{align}
\label{eq: regression_affine_core}
 \pi_x(T_x(u))=\pi_0(u),\qquad \cM_x=(T_x)_\#\cM_0.
\end{align}
Any core functional invariant under these affine transformations therefore has the same value at every input, whenever it is well defined.
\end{proposition}

\begin{proof}
Substitution into the residual score gives
\[
 s(x,T_x(u))
 =\frac{|\hat f(x)+w(x)u-\hat f(x)|}{w(x)}
 =|u|.
\]
The contour identity follows immediately from the shared calibration scores.

The positive scale makes the affine map a Borel bijection with a Borel inverse. Therefore, for every Borel event,
\begin{align*}
 \overline\PP_x(A)
 &=\sup_{y\in A}\pi_x(y)\\
 &=\sup_{u\in T_x^{-1}(A)}\pi_0(u)\\
 &=\overline\PP_0(T_x^{-1}(A)).
\end{align*}
If a probability measure belongs to the reference core, its pushforward consequently satisfies
\[
 (T_x)_\#P(A)
 =P(T_x^{-1}(A))
 \leq\overline\PP_0(T_x^{-1}(A))
 =\overline\PP_x(A).
\]
This proves one inclusion of the cores. Applying the same argument to the inverse affine map proves the reverse inclusion. The statement about invariant functionals then follows by applying such a functional to the core identity.
\end{proof}

\paragraph{Integrated summaries can retain response-space geometry.}
Let $\lambda$ be Lebesgue measure and $C_t^0:=\{u:\pi_0(u)>t\}$. The affine contour identity gives $C_t(x)=T_x(C_t^0)$ and, whenever the volumes are finite, $\lambda(C_t(x))=w(x)\lambda(C_t^0)$. Thus a finite integrated summary $U_{g,\lambda}(\cM_x)$ from Appendix~\ref{app:measurable-ihdr} can vary with $w(x)$, even though the corresponding rank-grid summary is constant. Finiteness requires care: ordinary split p-values satisfy $\pi_x(y)\geq1/(n+1)$ for every response, so $C_t(x)=\RR$ for $0\leq t<1/(n+1)$. For the volume penalty $g(v)=v$, the full threshold integral is therefore infinite. A finite volume-based summary requires an appropriate domain restriction or a penalty with a finite value at infinite volume.

\paragraph{The cores need not be equal, and arbitrary MMI need not be constant.}
Consider one calibration score, $n=1$ and $S_1=1$, with $w(x)=1$ and two fitted centres, $\hat f(x_0)=0$ and $\hat f(x_1)=3$. The first contour equals one on the interval centred at zero; the second equals one on the interval centred at three. Both equal one half outside their respective intervals:
\[
 \pi_{x_0}(y)=
 \begin{cases}1,&y\in[-1,1],\\ \tfrac12,&y\notin[-1,1],\end{cases}
 \qquad
 \pi_{x_1}(y)=
 \begin{cases}1,&y\in[2,4],\\ \tfrac12,&y\notin[2,4].\end{cases}
\]
The point mass at zero therefore belongs to the first core but not the second: $\delta_0\in\cM_{x_0}$ and $\delta_0\notin\cM_{x_1}$.

Now use the same response event $A=[0,\infty)$ at both inputs. At the first input, both the event and its complement meet the level-one set. At the second, only the event meets it. The upper--lower gap is
\[
 \overline\PP_x(A)-\underline\PP_x(A)
 =\overline\PP_x(A)+\overline\PP_x(A^c)-1.
\]
Consequently,
\begin{align*}
 \operatorname{MMI}_{\{\mathbf1_A\}}(\overline\PP_{x_0})
 &=1+1-1=1,\\
 \operatorname{MMI}_{\{\mathbf1_A\}}(\overline\PP_{x_1})
 &=1+\frac12-1=\frac12.
\end{align*}
An arbitrary fixed class of functions of the response therefore need not yield constant MMI: it can retain geometry discarded by the rank projection.

This dependence also occurs for the bounded-Lipschitz class $\mathcal H_{\mathrm{BL}}=\{h:\RR\to[0,1]:\operatorname{Lip}(h)\leq1\}$. Again take one calibration score equal to one, but now fix the fitted centre at zero and allow a positive scale $w$. The contour has two nested focal sets, $[-w,w]$ and $\RR$, each with mass one half. Define the oscillation of a function on a set by $\operatorname{osc}_A(h)=\sup_Ah-\inf_Ah$. Its upper--lower Choquet gap is then
\[
 \int h\,\mathrm d\overline\PP_x-\int h\,\mathrm d\underline\PP_x
 =\frac12\operatorname{osc}_{[-w,w]}(h)
  +\frac12\operatorname{osc}_{\RR}(h).
\]
The range restriction and Lipschitz bound imply
\[
 \operatorname{osc}_{[-w,w]}(h)\leq\min\{1,2w\},
 \qquad \operatorname{osc}_{\RR}(h)\leq1.
\]
Both bounds are attained simultaneously by $h(y)=\min\{1,\max\{0,y+w\}\}$. It follows that
\[
 \operatorname{MMI}_{\mathcal H_{\mathrm{BL}}}(\overline\PP_x)
 =\frac12+\frac12\min\{1,2w\}.
\]
This value varies with the fitted scale even though the rank-grid core remains unchanged.

\paragraph{TV-MMI on the full response space.}
Suppose the level-one set contains two distinct responses. Both associated point masses belong to the response-space core, and their TV distance is one. An event containing one of these responses but not the other has upper probability one and lower probability zero. Since no probability gap can exceed one, the class of all Borel event indicators satisfies
\[
 \operatorname{MMI}_{\cH_{\mathrm{TV}}}(\overline\PP_x)
 =\sup_{P,Q\in\cM_x}d_{\mathrm{TV}}(P,Q)=1.
\]
For absolute and positive weighted-residual scores, this applies whenever $S_{(1)}>0$, because the level-one set is an interval of positive length. Thus full response-space TV-MMI is also constant in this case, but its value is one rather than the rank-grid value $n/(n+1)$. The rank projection removes ambiguity among responses sharing the same p-value.

\paragraph{GH and entropy on a continuous response space.}
The finite-space GH and Shannon-entropy range used in this paper apply directly to the rank-grid core. They require additional definitions to apply to the response-space core on the real line.

Suppose the level-one set contains an interval of positive length. A uniform distribution on any finite collection of points in that interval belongs to the core. With $m$ such points, its Shannon entropy is $H(P_m)=\log_2m$. Allowing arbitrarily many points therefore makes the entropy unbounded even when attention is restricted to discrete distributions.

Replacing Shannon entropy by differential entropy does not produce a finite score for the unrestricted residual core either. To see the lower divergence, choose uniform densities on intervals of length $\varepsilon$ inside the level-one set. They belong to the core and satisfy
\[
 h(P_\varepsilon)=\log_2\varepsilon\longrightarrow-\infty
 \qquad\text{as }\varepsilon\downarrow0.
\]

For the upper divergence, set $a=1/(n+1)$ and fix a uniform distribution $P_0$ on an interval inside the level-one set. For any probability measure $Q$, the mixture $P_Q=(1-a)P_0+aQ$ belongs to the core. Indeed, an event meeting the level-one set has upper probability one. Any other nonempty event receives mass at most $a$, while every nonempty event has upper probability at least $a$.

Take $Q_L$ uniform on an interval of length $L$ disjoint from the support of $P_0$. For these disjoint supports, the differential entropy of the mixture is
\[
 h(P_{Q_L})
 =h_2(a)+(1-a)h(P_0)+a\log_2L,
\]
where $h_2(a)=-a\log_2a-(1-a)\log_2(1-a)$ is the binary mixing entropy. The first two terms are fixed, and the last tends to positive infinity with the interval length. Thus the differential-entropy range is unbounded in both directions.

A finite response-space entropy score would therefore require additional restrictions. Equality of the projected cores alone cannot justify constancy of uncertainty measures that retain response-space geometry.

    \clearpage
\section{EXPERIMENTAL SETUP AND NONCONFORMITY SCORES}
\label{appendix: exp}

The experiments compare probability-vector summaries, fixed-level conformal set sizes, and summaries of the complete normalised conformal core. Selective prediction and error/loss association evaluate all five core summaries under LAC, APS, RAPS, and SAPS on seven dataset/model conditions (Appendix~\ref{app:selective-details}). OOD detection evaluates all five under LAC and APS, and the four original summaries under RAPS and SAPS (Appendix~\ref{app:ood-details}). Appendix~\ref{app:tsallis-experiments} compares the four fixed Tsallis penalties through paired differences and a tail diagnostic. The population-limit study and exploratory active-learning experiment appear in Appendices~\ref{sec:asymptotic_ambiguity_ablation} and~\ref{app:active-learning}, respectively; active learning evaluates the four original summaries.

The fixed-predictor comparisons share cached predictions and pair calibration/evaluation samples and random seeds across summaries. Active learning pairs initial partitions and model seeds, with later predictions changing as acquisition paths diverge. Each task retains its own sampling and standard-error definition below. Colour ranks compare unrounded means within a dataset or condition and do not denote statistical significance; tied means share a rank. \ranklegend

\subsection{Uncertainty summaries and normalisation}
\label{app:experimental-summaries}

Write $q_y=\hat f_y(x)$, $q_{(1)}\geq\cdots\geq q_{(K)}$, and $\hat y=\arg\max_y q_y$. The three predictive uncertainty baselines are Shannon entropy $H(q)=-\sum_y q_y\log_2 q_y$, maximum-probability uncertainty $1-q_{(1)}$~\citep{hendrycks2017baseline}, and the second-largest probability $q_{(2)}$. Larger values always indicate greater uncertainty. The five core summaries are $U_\infty$ (TV-MMI), $U_1$ (GH), $U_2$ (saturation), $I_{\mathrm H}$, and $U_0$ (fuzziness), computed after the one-coordinate normalisation in Equation~\eqref{eq: consonance}. The family normalisation in Corollary~\ref{cor: tsallis_path} rescales each summary by a positive constant and leaves all reported ranking and correlation metrics unchanged; $I_{\mathrm H}$ is evaluated by the pool-adjacent-violators construction. For each $\alpha\in\{.01,.05,.10\}$, we include both the raw set size $|C_\alpha^{\rm raw}|$ and the normalised size $|C_\alpha|$. The latter equals $\max\{|C_\alpha^{\rm raw}|,1\}$ and is the size control in the main tables.

Predictive baselines are computed from the original probability vector without temperature fitting and do not change with the conformal score. Normalised cardinalities are the main-text controls; the appendix also retains raw cardinalities. The active-learning experiment uses raw cardinalities for acquisition, as specified in Appendix~\ref{app:active-learning}.

\subsection{Nonconformity scores and implementation}
\label{app:nonconformity-scores}

Let $V_r\sim\operatorname{Unif}(0,1)$ be independent draws in sorted-rank coordinates. The nonconformity scores for the label at rank $r$ are
\begin{align*}
 s_r^{\rm LAC}&=1-q_{(r)},\\
 s_r^{\rm APS}&=\sum_{j=1}^r q_{(j)}-V_rq_{(r)},\\
 s_r^{\rm RAPS}&=s_r^{\rm APS}+.01(r-5)_+,\\
 s_r^{\rm SAPS}&=\begin{cases}(1-V_1)q_{(1)},&r=1,\\
 q_{(1)}+.2(r-1-V_r),&r\geq2.
 \end{cases}
\end{align*}
LAC, APS, RAPS, and SAPS follow their respective constructions~\citep{sadinle2019least,romano2020classification,angelopoulosuncertainty2021,huang2024conformal}. RAPS uses penalty $.01$ and rank cutoff $5$; SAPS uses weight $.2$. These fixed settings were chosen before inspecting the added runs. The comparison therefore evaluates these configurations, rather than a separately tuned reproduction of each paper's prediction-set-efficiency study.

For each score, the true-label entry of the calibration score matrix is retained. The raw contour for a candidate score $s$ is $\pi^{\rm raw}(s)=(1+\sum_{i=1}^n\mathbf1\{s_i^{\rm cal}\geq s\})/(n+1)$.
Cuts use the strict inequality $\pi^{\rm raw}>\alpha$, and normalisation raises a single maximal coordinate to one. We use the actual all-label RAPS/SAPS implementations in TorchCP 1.2.1~\citep{huang2024torchcp}, with \texttt{randomized=True} and the identity transform on already computed probabilities. Calibration extracts its true-label scores from the same all-label implementation: this avoids a different coupling of the top-rank uniform in TorchCP's single-label SAPS call.

A scoped adapter supplies common NumPy float64 uniform matrices to TorchCP. It preserves TorchCP's probability-tie ordering and arithmetic, including its float32 RAPS rank-penalty tensor; probabilities and supplied uniforms are float64. Our APS implementation uses stable NumPy tie ordering, whereas TorchCP uses its own sorting convention. Shared rankwise random draws therefore need not be assigned to the same class within a probability tie.

\subsection{Information retained by the score construction}
\label{app:score-structure}

\paragraph{What LAC TV-MMI can distinguish.}
Let $z_i=q_{Y_i}(X_i)$ on calibration observations and define $T_n(t)=(1+\sum_{i=1}^n\mathbf1\{z_i\leq t\})/(n+1)$.
For a fixed calibration split, the normalised sorted contour has $\alpha_1=1$ and $\alpha_j=T_n(q_{(j)})$ for $j\geq2$. Consequently,
\begin{align*}
 U_\infty&=T_n(q_{(2)}),&
 U_0&=\frac{1}{K-1}\sum_{j=2}^K T_n(q_{(j)}),\\
 U_1&=\frac{1}{\log_2K}\sum_{j=2}^K T_n(q_{(j)})\log_2\frac{j}{j-1},&
 U_2&=\frac{1}{1-1/K}\sum_{j=2}^K\frac{T_n(q_{(j)})}{j(j-1)}.
\end{align*}
Thus LAC TV-MMI is a nondecreasing empirical transform of $q_{(2)}$: it can merge ranks into ties but cannot reverse their strict order. Its different Pearson coefficient is not evidence of additional ranking information. GH, $U_2$, $I_{\mathrm H}$, and $F$ depend on the remaining contour profile, with saturation concentrating relatively more weight on the leading competing labels than GH. The identity $1-q_{(1)}=\sum_{j=2}^K q_{(j)}$ supplies a further control: it sums competing probabilities uniformly, whereas the integrated LAC summaries apply calibration and the chosen rank weights.

\paragraph{What SAPS retains for symmetric summaries.}
Conditional on its calibration survival function, class count, and rankwise uniform draws, every non-top SAPS contour coordinate is a function of $q_{(1)}$ and the rank, without the magnitudes of the remaining class probabilities. Stretching the top coordinate to one therefore makes the five permutation-invariant core summaries randomised functions of $q_{(1)}$ under these conditions. With uniforms coupled across observations, a larger $q_{(1)}$ decreases all remaining contour coordinates. Actual independent draws can add ranking inversions. This directly motivates comparison with $1-q_{(1)}$. The argument concerns these symmetric summaries; prediction sets still depend on which labels occupy the ranks.

\subsection{Verification and reproduction}
\label{app:reproducibility}

\paragraph{Score implementations.}
Conformance checks cover $K=2,10,26,200,1000$, ties, zero probabilities, dense and sparse vectors, and the zero-penalty RAPS--APS identity within TorchCP. Independent recomputation also verifies all six reported metrics for the 1,680 retained original RAPS/SAPS method--seed rows; the saturation extension is checked against these archived metrics and split manifests below.

\paragraph{Selective-prediction tables.}
The supplied \path{scripts/generate_selective_tables.py} reads the archived per-seed correctness and NLL comparisons, the three additional model evaluations, the RAPS/SAPS evaluation, and the matched saturation extension. It reconstructs the original APS/LAC normalised-size controls from saved raw size vectors via $\max\{|C_\alpha^{\rm raw}|,1\}$, using the verified calibration/evaluation rows and common tie keys. All reconstructed raw metrics agree with their source values to within $2.3\times10^{-16}$, and the LAC normalisation controls agree with their independently saved audit. The complete table input comprises 3,920 method--seed rows: seven conditions, four conformal scores, 14 summaries, and ten seeds. The generator checks all six metrics for the 3,500 shared rows against the saturation experiment archive, then adds its 280 saturation rows and the 140 maximum-probability rows omitted from the original RAPS/SAPS tables. It also checks that predictive baselines agree across score constructions and records the source-file hashes in \path{tables/generated_selective_sources.tex}. Appendix~\ref{app:selective-details} includes all six metrics, with no filtering of conditions according to the outcome.

\paragraph{OOD tables.}
Baseline AUROCs and split indices were independently verified against the shared probability caches. Normalised APS cardinalities were reconstructed with the same random-number sequence and checked against the saved raw cardinalities. Before incorporating saturation, we checked all 20,800 shared condition--seed--method AUROCs against the original LAC and APS results and matched the calibration, ID, and OOD split hashes for all 800 condition--seed pairs. Reusing the maximum-probability baseline across scores additionally checks 6,400 shared archived predictive-baseline values, 1,600 archived baseline values, and all calibration, ID, and OOD hashes for 3,200 score--condition--seed rows.

\paragraph{Tsallis comparisons and saturation checks.}
The saturating summary's rank formula is checked against its mass formula on every evaluated contour, with independent strict-cut integration checks on small examples containing ties, zeros, and singleton/vacuous contours. Selective split and tie-key hashes agree with the existing manifests. Recomputed benchmark metrics and the recorded formula comparisons pass 35,260 checks, with maximum discrepancy $5.6\times10^{-16}$. Existing entropy-imprecision results are reused only in the archived comparisons. The numerical experiment is saved in \path{additional_baselines/saturating_penalty/}; \path{scripts/generate_tsallis_tables.py} generates the Tsallis comparison and error-AP tables from its per-seed outputs and records input hashes in \path{tables/generated_tsallis_metadata.json}. The archive also contains both raw and normalised fixed-size controls, per-condition OOD results, paired comparisons, and per-example selective scores. Appendices~\ref{app:selective-details} and~\ref{app:ood-details} include saturation alongside the detailed fixed-size benchmarks for each evaluated score; Appendix~\ref{app:tsallis-experiments} focuses on the penalty comparison, paired differences, and contribution diagnostics.

\section{SELECTIVE PREDICTION AND ERROR/LOSS ASSOCIATION}
\label{app:selective-details}

This appendix specifies the selective-prediction protocol and reports every dataset/model condition for all four conformal scores. The main comparison uses LAC and four natural-image benchmarks. CIFAR-10, Fashion-MNIST, and an additional ImageNetV2 model are retained here as controls; their results also show where gains are small or absent. All comparisons use one fixed classifier per condition and cached predictions, so changing the uncertainty summary does not change the classifier's accuracy.

\subsection{Datasets, fitted models, and cached predictions}
\label{app:selective-models}

\paragraph{CIFAR-10 and CIFAR-100.}
We use the 10,000 official test images of each dataset~\citep{KrizhevskyHinton2009TinyImages}, with the published \texttt{cifar10\_resnet56} and \texttt{cifar100\_resnet20} checkpoints, respectively.\footnote{\url{https://github.com/chenyaofo/pytorch-cifar-models}.} Images retain their native $32\times32$ resolution. RGB channel means/standard deviations are $(.4914,.4822,.4465)/(.2023,.1994,.2010)$ for CIFAR-10 and $(.5071,.4867,.4408)/(.2675,.2565,.2761)$ for CIFAR-100. The cached classifiers have full-pool accuracies of $94.37\%$ and $68.83\%$.

\paragraph{Caltech-101.}
We use 101 object classes, excluding the background class~\citep{FeiFei2004Caltech101}. A fixed seed of 2026 selects 30 training images per class (3,030 images); the other 5,647 images form the held-out pool. A ResNet-34 initialised with torchvision's \texttt{IMAGENET1K\_V1} weights is trained for five epochs using SGD, batch size 32, learning rate $10^{-3}$, momentum $.9$, and zero weight decay. Training uses a random $224\times224$ crop (area range $.8$--$1$, aspect ratio $3/4$--$4/3$) and horizontal flips with probability $.5$. Evaluation resizes the shorter side to 256 and takes a central $224\times224$ crop. Both use RGB ImageNet normalisation, with mean $(.485,.456,.406)$ and standard deviation $(.229,.224,.225)$. The final checkpoint is retained without held-out model selection. Saved image-index and content manifests verify that training and held-out pools are disjoint. Held-out accuracy is $88.40\%$.

\paragraph{Tiny ImageNet.}
We use the 10,000 official validation images (200 classes, 50 images per class)~\citep{le2015tinyimagenet} and a published TorchUncertainty ResNet-18 with a CIFAR-style stem and no convolutional bias. The checkpoint is \texttt{version\_438.safetensors}, the first model member in the publisher's archive order; this selection rule was fixed before evaluating predictions or uncertainty scores.\footnote{\url{https://huggingface.co/datasets/torch-uncertainty/Checkpoints}.\newline Revision \texttt{77e3e0adb21be723be095ad7494dd84e22c1eb96}, archive \texttt{tiny-imagenet-resnet18-0-1023.tgz}. The accompanying training description, \emph{A Symmetry-Aware Exploration of Bayesian Neural Network Posteriors}, Appendix B.1, specifies retaining the final training step to avoid validation-based checkpoint selection.} Evaluation uses the original $64\times64$ RGB images and the ImageNet channel normalisation above, without resizing or augmentation. Labels follow the official validation annotations in the checkpoint's original class order. Every cached image is uniquely matched to a JPEG in the official archive, verifying all 10,000 image identities and the 50-per-class counts. Full-pool accuracy is $63.24\%$.

\paragraph{ImageNetV2.}
We use the 10,000 images of the matched-frequency variant, with 1,000 classes and ten images per class~\citep{recht2019imagenet}. The main model is the ImageNet-1k pretrained \texttt{timm/resnetv2\_50.a1h\_in1k} checkpoint.\footnote{\url{https://huggingface.co/timm/resnetv2_50.a1h_in1k}, revision \texttt{a9035ec8691ab8c6eb6035aa3930c37583bcee5b}.} Its prescribed transform resizes the shorter side to 235 with bicubic interpolation, takes a central $224\times224$ crop, and uses channel mean and standard deviation $(.5,.5,.5)$. An additional torchvision ResNet-18 with \texttt{IMAGENET1K\_V1} weights uses a bilinear resize to 256, central $224\times224$ crop, and the standard ImageNet normalisation above. Both models use the same image identities and numeric ImageNet class indices. Their full-pool accuracies are $68.85\%$ and $57.28\%$, respectively. Neither checkpoint is trained or selected on ImageNetV2. Calibration and evaluation are disjoint subsets of the same ImageNetV2 pool; this experiment does not evaluate coverage when calibrating on a different source distribution.

\paragraph{Fashion-MNIST control.}
We retain the original 10,000-image Fashion-MNIST test pool~\citep{Xiao2017FashionMNIST} and the archived ResNet-18 checkpoint \texttt{FMNIST\_resnet.tar} with a ten-class head. Images retain their $28\times28$ resolution, are normalised with mean $.1307$ and standard deviation $.3081$, and are repeated across three channels. Full-pool accuracy is $93.17\%$. This is an archived checkpoint control; its metadata does not provide the detailed training manifest available for the newly fitted Caltech-101 model.

\paragraph{Prediction caches.}
All methods within a condition read the same probability matrix. The added Tiny ImageNet and ImageNetV2 caches use float32 model inference, without mixed precision or TF32, and float64 log-softmax to retain stable true-label log probabilities. Their environment uses PyTorch 2.9.1, torchvision 0.24.1, and, for ResNetv2-50, timm 1.0.24. Original caches are reused without temperature fitting. Checkpoint hashes, transforms, image identities, and probability-cache hashes are retained with the experiment artifacts.

\subsection{Splits, uncertainty directions, and evaluation}
\label{app:selective-protocol}

For seeds $0,\ldots,9$, a NumPy random permutation of each fixed prediction pool assigns 5,000 observations to calibration and the next 1,000 to evaluation; Caltech-101 uses 2,677 calibration observations and 1,000 evaluation observations. The remaining observations are unused in that split. Every method and conformal score uses the same calibration/evaluation identities for a given condition and seed. The ImageNetV2 models also share these identities. The saved split hashes are checked when extending the comparisons.

The uncertainty summaries, score implementations, and fixed parameters are defined in Appendix~\ref{appendix: exp}. All five core summaries, three probability baselines, and six raw/normalised cardinalities are evaluated for each of the four conformal scores.

\paragraph{Selective prediction.}
We reject observations in decreasing uncertainty order and measure accuracy on those retained. A common independent uniform key per observation resolves uncertainty ties, so methods are paired even when their scores are discrete. For rejection fraction $r$, let $A(r)$ denote retained accuracy. We report $\operatorname{AUARC}=(100/.9)\int_0^{.9}A(r)\,dr$, using trapezoidal integration on the empirical grid $r=0,1/1000,\ldots,.9$. Higher values are better; all methods have the same accuracy at zero rejection. Shared keys matter for set sizes and empirical conformal ranks, which may induce large tie groups.

\paragraph{Error and loss association.}
The binary error target is $E=\mathbf1\{\hat y\neq Y\}$ and the loss is $L=-\log q_Y$, in nats. Error AUROC treats errors as positives and gives tied uncertainty values average ranks. For $U_2$, we also report error average precision (AP), with errors as positives, as a percentage. Pearson and Spearman correlations target $E$ or $L$; correlation with correctness has the opposite sign. Spearman uses average tied ranks. Loss uses cached stable log probabilities when available and the float64 logarithm of the original true-class probabilities otherwise, without probability clipping; all evaluated losses are finite. A constant uncertainty score has AUROC $.5$ but undefined correlation, shown as a dash. All loss correlations use the full evaluation set, without conditioning on whether a prediction is correct.

\paragraph{Table conventions.}
ImageNetV2 denotes ResNetv2-50 unless labelled ResNet-18. The symbols $C_\alpha$ and $C_\alpha^{\rm raw}$ denote normalised and raw cuts. A dash marks a correlation undefined in at least one split. Colour ranks compare unrounded means within each dataset column, following Appendix~\ref{appendix: exp}.

\paragraph{Aggregation and uncertainty.}
Entries are means with sample standard errors, $s/\sqrt{10}$, across the ten shared seeds. Resplits overlap and reuse one fitted classifier: these errors describe split, score-randomisation, and tie-handling variability conditional on the prediction cache, not variation across independently trained models. For deterministic LAC, only the split and tie-handling components apply. Differences between summaries are paired within seed. Comparing the largest observed mean is descriptive; the score parameters are held fixed across benchmark conditions. Larger Pearson correlation alone does not imply a better uncertainty ranking: the rank-based metrics and selective accuracy answer different questions.

\subsection{Selective prediction across conformal scores}
\label{app:score-sensitivity}
\label{app:selective-results}

Tables~\ref{tab:selective-lac-auarc}--\ref{tab:selective-saps-auarc} report the complete AUARC comparison for LAC, APS, RAPS, and SAPS. Appendix~\ref{app:score-structure} explains what LAC and SAPS retain from the probability vector; Appendix~\ref{app:tsallis-experiments} isolates the paired differences between penalties.

\paragraph{Normalisation and the fixed-size controls.}
The raw and normalised LAC sizes coincide at all three reported thresholds for CIFAR-100, Tiny ImageNet, and both ImageNetV2 models. Caltech-101 differs at $.05$ and $.10$: normalisation raises mean AUARC from $93.54$ to $94.57$ and from $89.13$ to $91.38$, respectively. The strongest of its three fixed-size means remains the unchanged $.01$ control ($95.96$). Hence the improvements of GH and saturation over the strongest reported size on the four main conditions are not explained by retaining empty raw cuts. Conversely, normalised LAC sizes at $.10$ are constant in the CIFAR-10 and Fashion-MNIST controls, making their correlations undefined. The tables retain both constructions throughout.

LAC GH and saturation both have larger mean AUARC than every reported fixed-size control on the four main conditions, and exceed all three reported probability baselines on Caltech-101 and ImageNetV2--ResNetv2-50. Maximum-probability uncertainty leads on Tiny ImageNet and ImageNetV2--ResNet-18. Under the tested APS, RAPS, and SAPS configurations, $1-q_{(1)}$ exceeds every credal summary in mean AUARC and error-detection AUROC on all seven conditions. In particular, SAPS fuzziness on ResNet-18 exceeds entropy and second probability, but remains below $1-q_{(1)}$ ($79.288$ versus $79.392$ AUARC). The relative ordering within the credal family still changes: saturation exceeds GH on all seven conditions under APS and RAPS, whereas GH exceeds saturation on the five LAC conditions with at least 100 labels. Appendix~\ref{app:tsallis-experiments} reports the paired differences.

\begin{table}[H]
\centering
\scriptsize
\setlength{\tabcolsep}{3pt}
\renewcommand{\arraystretch}{1.00}
\caption{LAC: AUARC (\%) for all seven dataset/model conditions. Mean$_{\pm\mathrm{SE}}$ over ten paired splits; conventions as in Appendix~\ref{app:selective-protocol}.}
\label{tab:selective-lac-auarc}
\resizebox{\textwidth}{!}{%
\begin{tabular}{@{}lccccccc@{}}
\toprule
Summary & \makecell{CIFAR-\\100} & \makecell{Caltech-\\101} & \makecell{Tiny\\ImageNet} & ImageNetV2 & \makecell{CIFAR-\\10} & \makecell{Fashion-\\MNIST} & \makecell{ImageNetV2\\(ResNet-18)} \\
\midrule
\multicolumn{8}{@{}l}{\textit{Credal-based measures}} \\
$U_0$ ($F$) & $86.42_{\pm 0.27}$ & $97.56_{\pm 0.11}$ & $84.11_{\pm 0.48}$ & $82.36_{\pm 0.46}$ & $99.02_{\pm 0.04}$ & $98.91_{\pm 0.06}$ & $76.78_{\pm 0.30}$ \\
$U_1$ (GH) & $\secondscore{87.31}_{\pm 0.29}$ & $\bestscore{97.82}_{\pm 0.11}$ & $\secondscore{85.14}_{\pm 0.46}$ & $\bestscore{88.07}_{\pm 0.32}$ & $99.11_{\pm 0.04}$ & $98.94_{\pm 0.06}$ & $\secondscore{79.12}_{\pm 0.30}$ \\
$U_2$ (sat.) & $87.13_{\pm 0.31}$ & $\secondscore{97.76}_{\pm 0.13}$ & $84.78_{\pm 0.45}$ & $87.56_{\pm 0.34}$ & $99.15_{\pm 0.03}$ & $98.94_{\pm 0.06}$ & $77.48_{\pm 0.37}$ \\
$U_\infty$ (TV-MMI) & $86.29_{\pm 0.34}$ & $97.69_{\pm 0.14}$ & $83.46_{\pm 0.49}$ & $86.17_{\pm 0.39}$ & $\bestscore{99.18}_{\pm 0.04}$ & $98.94_{\pm 0.06}$ & $74.51_{\pm 0.37}$ \\
$I_{\mathrm H}$ & $87.26_{\pm 0.29}$ & $97.72_{\pm 0.13}$ & $85.06_{\pm 0.46}$ & $\secondscore{88.02}_{\pm 0.33}$ & $99.16_{\pm 0.03}$ & $98.94_{\pm 0.06}$ & $78.76_{\pm 0.31}$ \\
\midrule
\multicolumn{8}{@{}l}{\textit{Probability-vector-based measures}} \\
\bottomrule
\end{tabular}
}
\end{table}
\begin{table}[H]
\centering
\scriptsize
\setlength{\tabcolsep}{3pt}
\renewcommand{\arraystretch}{1.00}
\caption{APS: AUARC (\%) for all seven dataset/model conditions. Mean$_{\pm\mathrm{SE}}$ over ten paired splits; conventions as in Appendix~\ref{app:selective-protocol}.}
\label{tab:selective-aps-auarc}
\resizebox{\textwidth}{!}{%
\begin{tabular}{@{}lccccccc@{}}
\toprule
Summary & \makecell{CIFAR-\\100} & \makecell{Caltech-\\101} & \makecell{Tiny\\ImageNet} & ImageNetV2 & \makecell{CIFAR-\\10} & \makecell{Fashion-\\MNIST} & \makecell{ImageNetV2\\(ResNet-18)} \\
\midrule
\multicolumn{8}{@{}l}{\textit{Credal-based measures}} \\
$U_0$ ($F$) & $86.02_{\pm 0.28}$ & $93.78_{\pm 0.21}$ & $83.60_{\pm 0.49}$ & $79.06_{\pm 0.49}$ & $98.93_{\pm 0.05}$ & $96.88_{\pm 0.29}$ & $76.12_{\pm 0.31}$ \\
$U_1$ (GH) & $86.92_{\pm 0.29}$ & $94.42_{\pm 0.20}$ & $84.61_{\pm 0.46}$ & $83.72_{\pm 0.43}$ & $99.01_{\pm 0.05}$ & $97.80_{\pm 0.22}$ & $78.50_{\pm 0.31}$ \\
$U_2$ (sat.) & $87.10_{\pm 0.30}$ & $94.80_{\pm 0.19}$ & $84.99_{\pm 0.45}$ & $85.51_{\pm 0.41}$ & $99.04_{\pm 0.05}$ & $98.21_{\pm 0.15}$ & $78.98_{\pm 0.32}$ \\
$U_\infty$ (TV-MMI) & $87.08_{\pm 0.32}$ & $95.16_{\pm 0.19}$ & $\secondscore{85.11}_{\pm 0.44}$ & $86.10_{\pm 0.39}$ & $99.08_{\pm 0.04}$ & $98.41_{\pm 0.13}$ & $79.04_{\pm 0.34}$ \\
$I_{\mathrm H}$ & $86.98_{\pm 0.28}$ & $94.54_{\pm 0.19}$ & $84.77_{\pm 0.46}$ & $84.86_{\pm 0.42}$ & $99.06_{\pm 0.05}$ & $98.34_{\pm 0.13}$ & $78.59_{\pm 0.30}$ \\
\midrule
\multicolumn{8}{@{}l}{\textit{Probability-vector-based measures}} \\
\bottomrule
\end{tabular}
}
\end{table}
\begin{table}[H]
\centering
\scriptsize
\setlength{\tabcolsep}{3pt}
\renewcommand{\arraystretch}{1.00}
\caption{RAPS: AUARC (\%) for all seven dataset/model conditions. Mean$_{\pm\mathrm{SE}}$ over ten paired splits; conventions as in Appendix~\ref{app:selective-protocol}.}
\label{tab:selective-raps-auarc}
\resizebox{\textwidth}{!}{%
\begin{tabular}{@{}lccccccc@{}}
\toprule
Summary & \makecell{CIFAR-\\100} & \makecell{Caltech-\\101} & \makecell{Tiny\\ImageNet} & ImageNetV2 & \makecell{CIFAR-\\10} & \makecell{Fashion-\\MNIST} & \makecell{ImageNetV2\\(ResNet-18)} \\
\midrule
\multicolumn{8}{@{}l}{\textit{Credal-based measures}} \\
$U_0$ ($F$) & $86.93_{\pm 0.28}$ & $94.09_{\pm 0.20}$ & $84.53_{\pm 0.47}$ & $83.00_{\pm 0.43}$ & $98.98_{\pm 0.05}$ & $97.71_{\pm 0.24}$ & $78.46_{\pm 0.31}$ \\
$U_1$ (GH) & $87.08_{\pm 0.30}$ & $94.53_{\pm 0.20}$ & $84.87_{\pm 0.46}$ & $84.88_{\pm 0.42}$ & $99.02_{\pm 0.05}$ & $98.09_{\pm 0.19}$ & $78.86_{\pm 0.32}$ \\
$U_2$ (sat.) & $87.11_{\pm 0.30}$ & $94.82_{\pm 0.19}$ & $85.02_{\pm 0.45}$ & $85.63_{\pm 0.40}$ & $99.04_{\pm 0.05}$ & $98.26_{\pm 0.15}$ & $79.02_{\pm 0.33}$ \\
$U_\infty$ (TV-MMI) & $87.08_{\pm 0.32}$ & $95.17_{\pm 0.19}$ & $\secondscore{85.10}_{\pm 0.44}$ & $86.10_{\pm 0.39}$ & $99.08_{\pm 0.04}$ & $98.42_{\pm 0.13}$ & $79.04_{\pm 0.34}$ \\
$I_{\mathrm H}$ & $87.11_{\pm 0.29}$ & $94.64_{\pm 0.19}$ & $84.98_{\pm 0.45}$ & $85.58_{\pm 0.40}$ & $99.07_{\pm 0.05}$ & $98.37_{\pm 0.13}$ & $78.93_{\pm 0.31}$ \\
\midrule
\multicolumn{8}{@{}l}{\textit{Probability-vector-based measures}} \\
\bottomrule
\end{tabular}
}
\end{table}
\begin{table}[H]
\centering
\scriptsize
\setlength{\tabcolsep}{3pt}
\renewcommand{\arraystretch}{1.00}
\caption{SAPS: AUARC (\%) for all seven dataset/model conditions. Mean$_{\pm\mathrm{SE}}$ over ten paired splits; conventions as in Appendix~\ref{app:selective-protocol}.}
\label{tab:selective-saps-auarc}
\resizebox{\textwidth}{!}{%
\begin{tabular}{@{}lccccccc@{}}
\toprule
Summary & \makecell{CIFAR-\\100} & \makecell{Caltech-\\101} & \makecell{Tiny\\ImageNet} & ImageNetV2 & \makecell{CIFAR-\\10} & \makecell{Fashion-\\MNIST} & \makecell{ImageNetV2\\(ResNet-18)} \\
\midrule
\multicolumn{8}{@{}l}{\textit{Credal-based measures}} \\
$U_0$ ($F$) & $86.62_{\pm 0.32}$ & $96.51_{\pm 0.14}$ & $84.85_{\pm 0.39}$ & $\secondscore{86.79}_{\pm 0.36}$ & $97.58_{\pm 0.18}$ & $94.76_{\pm 0.35}$ & $\secondscore{79.29}_{\pm 0.32}$ \\
$U_1$ (GH) & $86.18_{\pm 0.33}$ & $96.46_{\pm 0.14}$ & $84.45_{\pm 0.37}$ & $86.65_{\pm 0.34}$ & $97.50_{\pm 0.18}$ & $94.73_{\pm 0.34}$ & $79.04_{\pm 0.33}$ \\
$U_2$ (sat.) & $85.87_{\pm 0.34}$ & $96.44_{\pm 0.14}$ & $84.21_{\pm 0.36}$ & $86.55_{\pm 0.35}$ & $97.45_{\pm 0.18}$ & $94.71_{\pm 0.34}$ & $78.85_{\pm 0.33}$ \\
$U_\infty$ (TV-MMI) & $85.51_{\pm 0.35}$ & $96.39_{\pm 0.14}$ & $83.95_{\pm 0.36}$ & $86.42_{\pm 0.35}$ & $97.37_{\pm 0.19}$ & $94.68_{\pm 0.34}$ & $78.58_{\pm 0.34}$ \\
$I_{\mathrm H}$ & $85.79_{\pm 0.34}$ & $96.44_{\pm 0.14}$ & $84.17_{\pm 0.35}$ & $86.56_{\pm 0.35}$ & $97.43_{\pm 0.18}$ & $94.70_{\pm 0.35}$ & $78.84_{\pm 0.33}$ \\
\midrule
\multicolumn{8}{@{}l}{\textit{Probability-vector-based measures}} \\
\bottomrule
\end{tabular}
}
\end{table}

\subsection{Error detection and loss association}
\label{app:association-results}

On the same examples, we correlate uncertainty with the classifier's error indicator, keeping the predictions fixed. Table~\ref{tab:selective-lac-pearson_error} shows that $U_1$ has the highest mean Pearson error correlation on CIFAR-100, Caltech-101, and ImageNetV2--ResNetv2-50, while $1-q_{(1)}$ leads on Tiny ImageNet. However, correlation strength and ranking quality are distinct: even the monotone LAC transform of $q_{(2)}$ can change Pearson correlation without adding an ordering signal. We therefore also report error-detection AUROC and Spearman correlations, together with correlations against true-label negative log probability, in Tables~\ref{tab:selective-lac-error_auroc}, \ref{tab:selective-lac-spearman_error}, \ref{tab:selective-lac-pearson_nll}, and~\ref{tab:selective-lac-spearman_nll}. These targets need not agree: on Caltech-101, $U_1$ and $U_2$ improve error AUROC from entropy's $73.02$ to $89.46\%$ and $89.17\%$, respectively (Table~\ref{tab:selective-lac-error_auroc}), yet their Pearson log-loss correlations are lower ($.495$ and $.449$ versus $.677$; Table~\ref{tab:selective-lac-pearson_nll}). A summary useful for identifying mistakes need not best track their log-loss magnitude.

Pearson association can favour a different summary: for APS on Fashion-MNIST, $I_H$ correlates more strongly with errors than $1-q_{(1)}$ ($.497$ versus $.353$), despite its weaker error ranking. These comparisons concern different properties of uncertainty and do not establish uniform superiority of either representation.

Table~\ref{tab:selective-error-ap} retains the error-AP results for saturation. The remaining tables give error AUROC and Pearson/Spearman association with errors and log loss for all summaries and all four conformal scores.

\begin{table}[H]
\centering\small
\setlength{\tabcolsep}{5pt}
\caption{Error average precision (\%) for $U_2$ under the four conformal scores. Mean $\pm$ one standard error over ten paired splits. ImageNetV2 uses ResNetv2-50 unless labelled ResNet-18.}
\label{tab:selective-error-ap}
\begin{tabular}{lrrrr}
\toprule
Dataset / model & LAC & APS & RAPS & SAPS \\
\midrule
CIFAR-100 & $62.693\,{\scriptstyle\pm\,0.572}$ & $64.464\,{\scriptstyle\pm\,0.699}$ & $64.478\,{\scriptstyle\pm\,0.704}$ & $64.001\,{\scriptstyle\pm\,0.643}$ \\
Caltech-101 & $44.552\,{\scriptstyle\pm\,1.172}$ & $26.380\,{\scriptstyle\pm\,1.327}$ & $26.503\,{\scriptstyle\pm\,1.337}$ & $32.817\,{\scriptstyle\pm\,1.575}$ \\
Tiny ImageNet & $69.344\,{\scriptstyle\pm\,1.192}$ & $75.530\,{\scriptstyle\pm\,1.084}$ & $75.594\,{\scriptstyle\pm\,1.087}$ & $75.604\,{\scriptstyle\pm\,1.080}$ \\
ImageNetV2 (R50) & $60.842\,{\scriptstyle\pm\,0.528}$ & $65.819\,{\scriptstyle\pm\,0.877}$ & $66.018\,{\scriptstyle\pm\,0.876}$ & $66.674\,{\scriptstyle\pm\,0.782}$ \\
CIFAR-10 & $44.655\,{\scriptstyle\pm\,1.790}$ & $42.644\,{\scriptstyle\pm\,1.903}$ & $42.648\,{\scriptstyle\pm\,1.901}$ & $33.599\,{\scriptstyle\pm\,1.568}$ \\
Fashion-MNIST & $48.419\,{\scriptstyle\pm\,1.579}$ & $46.284\,{\scriptstyle\pm\,1.810}$ & $46.396\,{\scriptstyle\pm\,1.808}$ & $27.534\,{\scriptstyle\pm\,2.321}$ \\
ImageNetV2 (R18) & $66.320\,{\scriptstyle\pm\,0.458}$ & $76.839\,{\scriptstyle\pm\,0.278}$ & $76.900\,{\scriptstyle\pm\,0.285}$ & $77.064\,{\scriptstyle\pm\,0.363}$ \\
\bottomrule
\end{tabular}
\end{table}

\Needspace{26\baselineskip}
\subsubsection{Error detection}
\begin{table}[H]
\centering
\scriptsize
\setlength{\tabcolsep}{3pt}
\renewcommand{\arraystretch}{1.00}
\caption{LAC: Error AUROC (\%) for all seven dataset/model conditions. Mean$_{\pm\mathrm{SE}}$ over ten paired splits; conventions as in Appendix~\ref{app:selective-protocol}.}
\label{tab:selective-lac-error_auroc}
\resizebox{\textwidth}{!}{%
\begin{tabular}{@{}lccccccc@{}}
\toprule
Summary & \makecell{CIFAR-\\100} & \makecell{Caltech-\\101} & \makecell{Tiny\\ImageNet} & ImageNetV2 & \makecell{CIFAR-\\10} & \makecell{Fashion-\\MNIST} & \makecell{ImageNetV2\\(ResNet-18)} \\
\midrule
\multicolumn{8}{@{}l}{\textit{Credal-based measures}} \\
$U_0$ ($F$) & $81.35_{\pm 0.30}$ & $87.83_{\pm 0.54}$ & $83.85_{\pm 0.42}$ & $74.95_{\pm 0.62}$ & $90.81_{\pm 0.31}$ & $91.66_{\pm 0.33}$ & $79.46_{\pm 0.24}$ \\
$U_1$ (GH) & $\bestscore{83.34}_{\pm 0.37}$ & $\bestscore{89.46}_{\pm 0.32}$ & $\secondscore{85.92}_{\pm 0.40}$ & $\bestscore{84.84}_{\pm 0.42}$ & $91.64_{\pm 0.32}$ & $91.95_{\pm 0.33}$ & $\secondscore{83.26}_{\pm 0.16}$ \\
$U_2$ (sat.) & $82.72_{\pm 0.41}$ & $\secondscore{89.17}_{\pm 0.35}$ & $84.62_{\pm 0.42}$ & $82.99_{\pm 0.40}$ & $91.95_{\pm 0.33}$ & $\bestscore{92.01}_{\pm 0.33}$ & $78.98_{\pm 0.30}$ \\
$U_\infty$ (TV-MMI) & $80.27_{\pm 0.46}$ & $88.81_{\pm 0.39}$ & $80.99_{\pm 0.48}$ & $79.18_{\pm 0.49}$ & $\bestscore{91.99}_{\pm 0.33}$ & $91.92_{\pm 0.34}$ & $72.77_{\pm 0.37}$ \\
$I_{\mathrm H}$ & $83.18_{\pm 0.38}$ & $88.85_{\pm 0.34}$ & $85.63_{\pm 0.41}$ & $\secondscore{84.47}_{\pm 0.41}$ & $\secondscore{91.99}_{\pm 0.33}$ & $91.99_{\pm 0.33}$ & $82.39_{\pm 0.19}$ \\
\midrule
\multicolumn{8}{@{}l}{\textit{Probability-vector-based measures}} \\
\bottomrule
\end{tabular}
}
\end{table}
\begin{table}[H]
\centering
\scriptsize
\setlength{\tabcolsep}{3pt}
\renewcommand{\arraystretch}{1.00}
\caption{APS: Error AUROC (\%) for all seven dataset/model conditions. Mean$_{\pm\mathrm{SE}}$ over ten paired splits; conventions as in Appendix~\ref{app:selective-protocol}.}
\label{tab:selective-aps-error_auroc}
\resizebox{\textwidth}{!}{%
\begin{tabular}{@{}lccccccc@{}}
\toprule
Summary & \makecell{CIFAR-\\100} & \makecell{Caltech-\\101} & \makecell{Tiny\\ImageNet} & ImageNetV2 & \makecell{CIFAR-\\10} & \makecell{Fashion-\\MNIST} & \makecell{ImageNetV2\\(ResNet-18)} \\
\midrule
\multicolumn{8}{@{}l}{\textit{Credal-based measures}} \\
$U_0$ ($F$) & $80.64_{\pm 0.28}$ & $67.92_{\pm 1.05}$ & $83.09_{\pm 0.44}$ & $69.32_{\pm 0.65}$ & $90.08_{\pm 0.32}$ & $81.31_{\pm 1.18}$ & $78.61_{\pm 0.26}$ \\
$U_1$ (GH) & $82.55_{\pm 0.35}$ & $71.04_{\pm 0.97}$ & $84.93_{\pm 0.41}$ & $77.82_{\pm 0.56}$ & $90.80_{\pm 0.36}$ & $86.46_{\pm 0.82}$ & $82.21_{\pm 0.18}$ \\
$U_2$ (sat.) & $82.92_{\pm 0.40}$ & $73.10_{\pm 0.92}$ & $85.67_{\pm 0.41}$ & $81.11_{\pm 0.49}$ & $91.13_{\pm 0.40}$ & $88.58_{\pm 0.63}$ & $83.08_{\pm 0.20}$ \\
$U_\infty$ (TV-MMI) & $82.83_{\pm 0.45}$ & $75.13_{\pm 0.91}$ & $\secondscore{85.90}_{\pm 0.42}$ & $\secondscore{82.10}_{\pm 0.46}$ & $91.38_{\pm 0.41}$ & $89.61_{\pm 0.46}$ & $\secondscore{83.16}_{\pm 0.24}$ \\
$I_{\mathrm H}$ & $82.63_{\pm 0.36}$ & $71.45_{\pm 0.91}$ & $85.18_{\pm 0.41}$ & $79.68_{\pm 0.54}$ & $91.27_{\pm 0.40}$ & $89.28_{\pm 0.53}$ & $82.30_{\pm 0.18}$ \\
\midrule
\multicolumn{8}{@{}l}{\textit{Probability-vector-based measures}} \\
\bottomrule
\end{tabular}
}
\end{table}
\begin{table}[H]
\centering
\scriptsize
\setlength{\tabcolsep}{3pt}
\renewcommand{\arraystretch}{1.00}
\caption{RAPS: Error AUROC (\%) for all seven dataset/model conditions. Mean$_{\pm\mathrm{SE}}$ over ten paired splits; conventions as in Appendix~\ref{app:selective-protocol}.}
\label{tab:selective-raps-error_auroc}
\resizebox{\textwidth}{!}{%
\begin{tabular}{@{}lccccccc@{}}
\toprule
Summary & \makecell{CIFAR-\\100} & \makecell{Caltech-\\101} & \makecell{Tiny\\ImageNet} & ImageNetV2 & \makecell{CIFAR-\\10} & \makecell{Fashion-\\MNIST} & \makecell{ImageNetV2\\(ResNet-18)} \\
\midrule
\multicolumn{8}{@{}l}{\textit{Credal-based measures}} \\
$U_0$ ($F$) & $82.52_{\pm 0.35}$ & $69.38_{\pm 0.99}$ & $84.71_{\pm 0.41}$ & $76.61_{\pm 0.59}$ & $90.58_{\pm 0.35}$ & $86.04_{\pm 0.92}$ & $82.06_{\pm 0.17}$ \\
$U_1$ (GH) & $82.86_{\pm 0.39}$ & $71.63_{\pm 0.95}$ & $85.40_{\pm 0.41}$ & $79.96_{\pm 0.51}$ & $90.95_{\pm 0.38}$ & $88.04_{\pm 0.75}$ & $82.82_{\pm 0.18}$ \\
$U_2$ (sat.) & $82.94_{\pm 0.41}$ & $73.24_{\pm 0.92}$ & $85.73_{\pm 0.42}$ & $81.31_{\pm 0.48}$ & $91.17_{\pm 0.40}$ & $88.86_{\pm 0.61}$ & $83.14_{\pm 0.21}$ \\
$U_\infty$ (TV-MMI) & $82.83_{\pm 0.45}$ & $75.14_{\pm 0.91}$ & $\secondscore{85.90}_{\pm 0.42}$ & $\secondscore{82.10}_{\pm 0.46}$ & $91.39_{\pm 0.41}$ & $89.64_{\pm 0.46}$ & $\secondscore{83.16}_{\pm 0.24}$ \\
$I_{\mathrm H}$ & $82.92_{\pm 0.38}$ & $72.02_{\pm 0.91}$ & $85.59_{\pm 0.41}$ & $81.01_{\pm 0.50}$ & $91.35_{\pm 0.41}$ & $89.47_{\pm 0.51}$ & $82.90_{\pm 0.18}$ \\
\midrule
\multicolumn{8}{@{}l}{\textit{Probability-vector-based measures}} \\
\bottomrule
\end{tabular}
}
\end{table}
\begin{table}[H]
\centering
\scriptsize
\setlength{\tabcolsep}{3pt}
\renewcommand{\arraystretch}{1.00}
\caption{SAPS: Error AUROC (\%) for all seven dataset/model conditions. Mean$_{\pm\mathrm{SE}}$ over ten paired splits; conventions as in Appendix~\ref{app:selective-protocol}.}
\label{tab:selective-saps-error_auroc}
\resizebox{\textwidth}{!}{%
\begin{tabular}{@{}lccccccc@{}}
\toprule
Summary & \makecell{CIFAR-\\100} & \makecell{Caltech-\\101} & \makecell{Tiny\\ImageNet} & ImageNetV2 & \makecell{CIFAR-\\10} & \makecell{Fashion-\\MNIST} & \makecell{ImageNetV2\\(ResNet-18)} \\
\midrule
\multicolumn{8}{@{}l}{\textit{Credal-based measures}} \\
$U_0$ ($F$) & $82.53_{\pm 0.41}$ & $81.42_{\pm 0.65}$ & $\secondscore{85.85}_{\pm 0.42}$ & $\secondscore{83.27}_{\pm 0.40}$ & $78.12_{\pm 1.36}$ & $64.84_{\pm 1.54}$ & $\secondscore{83.65}_{\pm 0.21}$ \\
$U_1$ (GH) & $81.97_{\pm 0.43}$ & $81.23_{\pm 0.67}$ & $85.41_{\pm 0.43}$ & $83.01_{\pm 0.38}$ & $77.44_{\pm 1.38}$ & $64.49_{\pm 1.50}$ & $83.38_{\pm 0.24}$ \\
$U_2$ (sat.) & $81.58_{\pm 0.45}$ & $81.16_{\pm 0.68}$ & $85.14_{\pm 0.43}$ & $82.83_{\pm 0.38}$ & $76.96_{\pm 1.37}$ & $64.28_{\pm 1.49}$ & $83.16_{\pm 0.25}$ \\
$U_\infty$ (TV-MMI) & $81.05_{\pm 0.48}$ & $81.04_{\pm 0.70}$ & $84.79_{\pm 0.43}$ & $82.59_{\pm 0.38}$ & $76.17_{\pm 1.42}$ & $63.92_{\pm 1.46}$ & $82.80_{\pm 0.29}$ \\
$I_{\mathrm H}$ & $81.50_{\pm 0.44}$ & $81.19_{\pm 0.68}$ & $85.14_{\pm 0.44}$ & $82.88_{\pm 0.39}$ & $76.83_{\pm 1.40}$ & $64.24_{\pm 1.50}$ & $83.21_{\pm 0.26}$ \\
\midrule
\multicolumn{8}{@{}l}{\textit{Probability-vector-based measures}} \\
\bottomrule
\end{tabular}
}
\end{table}
\Needspace{26\baselineskip}
\subsubsection{Association with the error indicator}
\begin{table}[H]
\centering
\scriptsize
\setlength{\tabcolsep}{3pt}
\renewcommand{\arraystretch}{1.00}
\caption{LAC: Pearson correlation with errors for all seven dataset/model conditions. Mean$_{\pm\mathrm{SE}}$ over ten paired splits; conventions as in Appendix~\ref{app:selective-protocol}.}
\label{tab:selective-lac-pearson_error}
\resizebox{\textwidth}{!}{%
\begin{tabular}{@{}lccccccc@{}}
\toprule
Summary & \makecell{CIFAR-\\100} & \makecell{Caltech-\\101} & \makecell{Tiny\\ImageNet} & ImageNetV2 & \makecell{CIFAR-\\10} & \makecell{Fashion-\\MNIST} & \makecell{ImageNetV2\\(ResNet-18)} \\
\midrule
\multicolumn{8}{@{}l}{\textit{Credal-based measures}} \\
$U_0$ ($F$) & $0.470_{\pm 0.005}$ & $0.472_{\pm 0.012}$ & $0.552_{\pm 0.008}$ & $0.416_{\pm 0.010}$ & $0.490_{\pm 0.012}$ & $0.516_{\pm 0.008}$ & $0.462_{\pm 0.004}$ \\
$U_1$ (GH) & $\bestscore{0.541}_{\pm 0.006}$ & $\bestscore{0.492}_{\pm 0.007}$ & $\secondscore{0.601}_{\pm 0.008}$ & $\bestscore{0.566}_{\pm 0.007}$ & $\secondscore{0.507}_{\pm 0.010}$ & $\secondscore{0.527}_{\pm 0.008}$ & $\secondscore{0.564}_{\pm 0.003}$ \\
$U_2$ (sat.) & $0.530_{\pm 0.007}$ & $0.489_{\pm 0.007}$ & $0.574_{\pm 0.008}$ & $0.529_{\pm 0.007}$ & $\bestscore{0.510}_{\pm 0.009}$ & $\bestscore{0.529}_{\pm 0.009}$ & $0.504_{\pm 0.003}$ \\
$U_\infty$ (TV-MMI) & $0.490_{\pm 0.007}$ & $0.486_{\pm 0.008}$ & $0.525_{\pm 0.008}$ & $0.467_{\pm 0.007}$ & $0.502_{\pm 0.008}$ & $0.523_{\pm 0.011}$ & $0.416_{\pm 0.005}$ \\
$I_{\mathrm H}$ & $\secondscore{0.533}_{\pm 0.006}$ & $0.487_{\pm 0.006}$ & $0.583_{\pm 0.008}$ & $0.549_{\pm 0.007}$ & $0.501_{\pm 0.009}$ & $0.521_{\pm 0.009}$ & $0.543_{\pm 0.003}$ \\
\midrule
\multicolumn{8}{@{}l}{\textit{Probability-vector-based measures}} \\
\bottomrule
\end{tabular}
}
\end{table}
\begin{table}[H]
\centering
\scriptsize
\setlength{\tabcolsep}{3pt}
\renewcommand{\arraystretch}{1.00}
\caption{LAC: Spearman correlation with errors for all seven dataset/model conditions. Mean$_{\pm\mathrm{SE}}$ over ten paired splits; conventions as in Appendix~\ref{app:selective-protocol}.}
\label{tab:selective-lac-spearman_error}
\resizebox{\textwidth}{!}{%
\begin{tabular}{@{}lccccccc@{}}
\toprule
Summary & \makecell{CIFAR-\\100} & \makecell{Caltech-\\101} & \makecell{Tiny\\ImageNet} & ImageNetV2 & \makecell{CIFAR-\\10} & \makecell{Fashion-\\MNIST} & \makecell{ImageNetV2\\(ResNet-18)} \\
\midrule
\multicolumn{8}{@{}l}{\textit{Credal-based measures}} \\
$U_0$ ($F$) & $0.504_{\pm 0.005}$ & $0.414_{\pm 0.009}$ & $0.563_{\pm 0.008}$ & $0.399_{\pm 0.010}$ & $0.329_{\pm 0.006}$ & $0.374_{\pm 0.007}$ & $0.505_{\pm 0.004}$ \\
$U_1$ (GH) & $\bestscore{0.536}_{\pm 0.006}$ & $\secondscore{0.431}_{\pm 0.005}$ & $\secondscore{0.598}_{\pm 0.008}$ & $\bestscore{0.557}_{\pm 0.007}$ & $0.335_{\pm 0.006}$ & $0.377_{\pm 0.007}$ & $\secondscore{0.570}_{\pm 0.003}$ \\
$U_2$ (sat.) & $0.526_{\pm 0.006}$ & $0.428_{\pm 0.005}$ & $0.576_{\pm 0.008}$ & $0.527_{\pm 0.006}$ & $0.338_{\pm 0.006}$ & $0.377_{\pm 0.007}$ & $0.496_{\pm 0.005}$ \\
$U_\infty$ (TV-MMI) & $0.487_{\pm 0.007}$ & $0.425_{\pm 0.005}$ & $0.516_{\pm 0.009}$ & $0.466_{\pm 0.007}$ & $\secondscore{0.339}_{\pm 0.007}$ & $0.377_{\pm 0.007}$ & $0.390_{\pm 0.006}$ \\
$I_{\mathrm H}$ & $0.534_{\pm 0.006}$ & $0.425_{\pm 0.005}$ & $0.593_{\pm 0.008}$ & $\secondscore{0.551}_{\pm 0.007}$ & $0.338_{\pm 0.006}$ & $0.377_{\pm 0.007}$ & $0.555_{\pm 0.003}$ \\
\midrule
\multicolumn{8}{@{}l}{\textit{Probability-vector-based measures}} \\
\bottomrule
\end{tabular}
}
\end{table}
\begin{table}[H]
\centering
\scriptsize
\setlength{\tabcolsep}{3pt}
\renewcommand{\arraystretch}{1.00}
\caption{APS: Pearson correlation with errors for all seven dataset/model conditions. Mean$_{\pm\mathrm{SE}}$ over ten paired splits; conventions as in Appendix~\ref{app:selective-protocol}.}
\label{tab:selective-aps-pearson_error}
\resizebox{\textwidth}{!}{%
\begin{tabular}{@{}lccccccc@{}}
\toprule
Summary & \makecell{CIFAR-\\100} & \makecell{Caltech-\\101} & \makecell{Tiny\\ImageNet} & ImageNetV2 & \makecell{CIFAR-\\10} & \makecell{Fashion-\\MNIST} & \makecell{ImageNetV2\\(ResNet-18)} \\
\midrule
\multicolumn{8}{@{}l}{\textit{Credal-based measures}} \\
$U_0$ ($F$) & $0.405_{\pm 0.006}$ & $0.167_{\pm 0.015}$ & $0.504_{\pm 0.008}$ & $0.325_{\pm 0.009}$ & $0.452_{\pm 0.015}$ & $0.450_{\pm 0.011}$ & $0.401_{\pm 0.004}$ \\
$U_1$ (GH) & $0.483_{\pm 0.007}$ & $0.209_{\pm 0.016}$ & $0.560_{\pm 0.008}$ & $0.457_{\pm 0.009}$ & $0.450_{\pm 0.015}$ & $0.450_{\pm 0.014}$ & $0.509_{\pm 0.003}$ \\
$U_2$ (sat.) & $0.510_{\pm 0.007}$ & $0.244_{\pm 0.017}$ & $0.589_{\pm 0.008}$ & $0.514_{\pm 0.009}$ & $0.445_{\pm 0.015}$ & $0.437_{\pm 0.016}$ & $0.554_{\pm 0.003}$ \\
$U_\infty$ (TV-MMI) & $\bestscore{0.518}_{\pm 0.008}$ & $0.280_{\pm 0.017}$ & $\secondscore{0.600}_{\pm 0.008}$ & $\secondscore{0.534}_{\pm 0.008}$ & $0.438_{\pm 0.015}$ & $0.421_{\pm 0.018}$ & $\secondscore{0.565}_{\pm 0.004}$ \\
$I_{\mathrm H}$ & $0.512_{\pm 0.006}$ & $0.218_{\pm 0.016}$ & $0.586_{\pm 0.008}$ & $0.490_{\pm 0.009}$ & $\bestscore{0.488}_{\pm 0.013}$ & $\bestscore{0.497}_{\pm 0.013}$ & $0.537_{\pm 0.003}$ \\
\midrule
\multicolumn{8}{@{}l}{\textit{Probability-vector-based measures}} \\
\bottomrule
\end{tabular}
}
\end{table}
\begin{table}[H]
\centering
\scriptsize
\setlength{\tabcolsep}{3pt}
\renewcommand{\arraystretch}{1.00}
\caption{APS: Spearman correlation with errors for all seven dataset/model conditions. Mean$_{\pm\mathrm{SE}}$ over ten paired splits; conventions as in Appendix~\ref{app:selective-protocol}.}
\label{tab:selective-aps-spearman_error}
\resizebox{\textwidth}{!}{%
\begin{tabular}{@{}lccccccc@{}}
\toprule
Summary & \makecell{CIFAR-\\100} & \makecell{Caltech-\\101} & \makecell{Tiny\\ImageNet} & ImageNetV2 & \makecell{CIFAR-\\10} & \makecell{Fashion-\\MNIST} & \makecell{ImageNetV2\\(ResNet-18)} \\
\midrule
\multicolumn{8}{@{}l}{\textit{Credal-based measures}} \\
$U_0$ ($F$) & $0.493_{\pm 0.005}$ & $0.196_{\pm 0.012}$ & $0.550_{\pm 0.008}$ & $0.309_{\pm 0.010}$ & $0.323_{\pm 0.006}$ & $0.283_{\pm 0.011}$ & $0.490_{\pm 0.005}$ \\
$U_1$ (GH) & $0.523_{\pm 0.006}$ & $0.230_{\pm 0.012}$ & $0.581_{\pm 0.008}$ & $0.445_{\pm 0.009}$ & $0.329_{\pm 0.007}$ & $0.330_{\pm 0.008}$ & $0.552_{\pm 0.003}$ \\
$U_2$ (sat.) & $0.529_{\pm 0.006}$ & $0.253_{\pm 0.011}$ & $0.593_{\pm 0.008}$ & $0.497_{\pm 0.008}$ & $0.331_{\pm 0.007}$ & $0.349_{\pm 0.007}$ & $0.567_{\pm 0.003}$ \\
$U_\infty$ (TV-MMI) & $0.528_{\pm 0.007}$ & $0.275_{\pm 0.011}$ & $\secondscore{0.597}_{\pm 0.008}$ & $\secondscore{0.513}_{\pm 0.007}$ & $0.334_{\pm 0.007}$ & $0.362_{\pm 0.006}$ & $\secondscore{0.568}_{\pm 0.004}$ \\
$I_{\mathrm H}$ & $0.525_{\pm 0.006}$ & $0.235_{\pm 0.011}$ & $0.585_{\pm 0.008}$ & $0.474_{\pm 0.008}$ & $0.333_{\pm 0.007}$ & $0.358_{\pm 0.007}$ & $0.553_{\pm 0.003}$ \\
\midrule
\multicolumn{8}{@{}l}{\textit{Probability-vector-based measures}} \\
\bottomrule
\end{tabular}
}
\end{table}
\begin{table}[H]
\centering
\scriptsize
\setlength{\tabcolsep}{3pt}
\renewcommand{\arraystretch}{1.00}
\caption{RAPS: Pearson correlation with errors for all seven dataset/model conditions. Mean$_{\pm\mathrm{SE}}$ over ten paired splits; conventions as in Appendix~\ref{app:selective-protocol}.}
\label{tab:selective-raps-pearson_error}
\resizebox{\textwidth}{!}{%
\begin{tabular}{@{}lccccccc@{}}
\toprule
Summary & \makecell{CIFAR-\\100} & \makecell{Caltech-\\101} & \makecell{Tiny\\ImageNet} & ImageNetV2 & \makecell{CIFAR-\\10} & \makecell{Fashion-\\MNIST} & \makecell{ImageNetV2\\(ResNet-18)} \\
\midrule
\multicolumn{8}{@{}l}{\textit{Credal-based measures}} \\
$U_0$ ($F$) & $0.413_{\pm 0.007}$ & $0.182_{\pm 0.015}$ & $0.487_{\pm 0.008}$ & $0.413_{\pm 0.009}$ & $0.443_{\pm 0.015}$ & $\secondscore{0.455}_{\pm 0.013}$ & $0.429_{\pm 0.003}$ \\
$U_1$ (GH) & $0.475_{\pm 0.007}$ & $0.216_{\pm 0.017}$ & $0.552_{\pm 0.008}$ & $0.481_{\pm 0.009}$ & $0.445_{\pm 0.015}$ & $0.441_{\pm 0.015}$ & $0.511_{\pm 0.003}$ \\
$U_2$ (sat.) & $0.499_{\pm 0.008}$ & $0.246_{\pm 0.017}$ & $0.580_{\pm 0.008}$ & $0.515_{\pm 0.009}$ & $0.442_{\pm 0.015}$ & $0.430_{\pm 0.016}$ & $0.545_{\pm 0.004}$ \\
$U_\infty$ (TV-MMI) & $0.509_{\pm 0.008}$ & $0.280_{\pm 0.017}$ & $\secondscore{0.593}_{\pm 0.008}$ & $\secondscore{0.532}_{\pm 0.008}$ & $0.436_{\pm 0.015}$ & $0.416_{\pm 0.018}$ & $\secondscore{0.558}_{\pm 0.005}$ \\
$I_{\mathrm H}$ & $0.515_{\pm 0.007}$ & $0.227_{\pm 0.016}$ & $0.587_{\pm 0.008}$ & $0.512_{\pm 0.009}$ & $\bestscore{0.486}_{\pm 0.013}$ & $\bestscore{0.492}_{\pm 0.013}$ & $0.548_{\pm 0.003}$ \\
\midrule
\multicolumn{8}{@{}l}{\textit{Probability-vector-based measures}} \\
\bottomrule
\end{tabular}
}
\end{table}
\begin{table}[H]
\centering
\scriptsize
\setlength{\tabcolsep}{3pt}
\renewcommand{\arraystretch}{1.00}
\caption{RAPS: Spearman correlation with errors for all seven dataset/model conditions. Mean$_{\pm\mathrm{SE}}$ over ten paired splits; conventions as in Appendix~\ref{app:selective-protocol}.}
\label{tab:selective-raps-spearman_error}
\resizebox{\textwidth}{!}{%
\begin{tabular}{@{}lccccccc@{}}
\toprule
Summary & \makecell{CIFAR-\\100} & \makecell{Caltech-\\101} & \makecell{Tiny\\ImageNet} & ImageNetV2 & \makecell{CIFAR-\\10} & \makecell{Fashion-\\MNIST} & \makecell{ImageNetV2\\(ResNet-18)} \\
\midrule
\multicolumn{8}{@{}l}{\textit{Credal-based measures}} \\
$U_0$ ($F$) & $0.523_{\pm 0.005}$ & $0.212_{\pm 0.012}$ & $0.577_{\pm 0.008}$ & $0.425_{\pm 0.009}$ & $0.327_{\pm 0.006}$ & $0.326_{\pm 0.008}$ & $0.549_{\pm 0.003}$ \\
$U_1$ (GH) & $0.529_{\pm 0.006}$ & $0.237_{\pm 0.011}$ & $0.589_{\pm 0.008}$ & $0.479_{\pm 0.008}$ & $0.330_{\pm 0.007}$ & $0.344_{\pm 0.008}$ & $0.562_{\pm 0.003}$ \\
$U_2$ (sat.) & $0.530_{\pm 0.006}$ & $0.254_{\pm 0.011}$ & $0.594_{\pm 0.008}$ & $0.500_{\pm 0.008}$ & $0.332_{\pm 0.007}$ & $0.352_{\pm 0.007}$ & $0.568_{\pm 0.003}$ \\
$U_\infty$ (TV-MMI) & $0.528_{\pm 0.007}$ & $0.275_{\pm 0.011}$ & $\secondscore{0.597}_{\pm 0.008}$ & $\secondscore{0.513}_{\pm 0.007}$ & $0.334_{\pm 0.007}$ & $0.363_{\pm 0.006}$ & $\secondscore{0.568}_{\pm 0.004}$ \\
$I_{\mathrm H}$ & $0.529_{\pm 0.006}$ & $0.241_{\pm 0.011}$ & $0.592_{\pm 0.008}$ & $0.496_{\pm 0.008}$ & $0.333_{\pm 0.007}$ & $0.360_{\pm 0.007}$ & $0.564_{\pm 0.003}$ \\
\midrule
\multicolumn{8}{@{}l}{\textit{Probability-vector-based measures}} \\
\bottomrule
\end{tabular}
}
\end{table}
\begin{table}[H]
\centering
\scriptsize
\setlength{\tabcolsep}{3pt}
\renewcommand{\arraystretch}{1.00}
\caption{SAPS: Pearson correlation with errors for all seven dataset/model conditions. Mean$_{\pm\mathrm{SE}}$ over ten paired splits; conventions as in Appendix~\ref{app:selective-protocol}.}
\label{tab:selective-saps-pearson_error}
\resizebox{\textwidth}{!}{%
\begin{tabular}{@{}lccccccc@{}}
\toprule
Summary & \makecell{CIFAR-\\100} & \makecell{Caltech-\\101} & \makecell{Tiny\\ImageNet} & ImageNetV2 & \makecell{CIFAR-\\10} & \makecell{Fashion-\\MNIST} & \makecell{ImageNetV2\\(ResNet-18)} \\
\midrule
\multicolumn{8}{@{}l}{\textit{Credal-based measures}} \\
$U_0$ ($F$) & $0.480_{\pm 0.007}$ & $0.323_{\pm 0.016}$ & $0.584_{\pm 0.008}$ & $0.536_{\pm 0.007}$ & $0.377_{\pm 0.012}$ & $0.293_{\pm 0.023}$ & $0.559_{\pm 0.004}$ \\
$U_1$ (GH) & $0.481_{\pm 0.008}$ & $0.330_{\pm 0.016}$ & $0.581_{\pm 0.008}$ & $0.533_{\pm 0.007}$ & $0.381_{\pm 0.012}$ & $0.295_{\pm 0.023}$ & $0.552_{\pm 0.005}$ \\
$U_2$ (sat.) & $0.485_{\pm 0.008}$ & $0.334_{\pm 0.016}$ & $0.583_{\pm 0.008}$ & $0.534_{\pm 0.007}$ & $0.383_{\pm 0.013}$ & $0.295_{\pm 0.024}$ & $0.552_{\pm 0.005}$ \\
$U_\infty$ (TV-MMI) & $0.489_{\pm 0.008}$ & $0.338_{\pm 0.015}$ & $0.585_{\pm 0.008}$ & $0.535_{\pm 0.007}$ & $0.385_{\pm 0.014}$ & $0.295_{\pm 0.024}$ & $0.551_{\pm 0.005}$ \\
$I_{\mathrm H}$ & $0.506_{\pm 0.008}$ & $0.352_{\pm 0.014}$ & $\secondscore{0.601}_{\pm 0.008}$ & $\secondscore{0.552}_{\pm 0.007}$ & $0.401_{\pm 0.012}$ & $0.298_{\pm 0.023}$ & $\secondscore{0.568}_{\pm 0.005}$ \\
\midrule
\multicolumn{8}{@{}l}{\textit{Probability-vector-based measures}} \\
\bottomrule
\end{tabular}
}
\end{table}
\begin{table}[H]
\centering
\scriptsize
\setlength{\tabcolsep}{3pt}
\renewcommand{\arraystretch}{1.00}
\caption{SAPS: Spearman correlation with errors for all seven dataset/model conditions. Mean$_{\pm\mathrm{SE}}$ over ten paired splits; conventions as in Appendix~\ref{app:selective-protocol}.}
\label{tab:selective-saps-spearman_error}
\resizebox{\textwidth}{!}{%
\begin{tabular}{@{}lccccccc@{}}
\toprule
Summary & \makecell{CIFAR-\\100} & \makecell{Caltech-\\101} & \makecell{Tiny\\ImageNet} & ImageNetV2 & \makecell{CIFAR-\\10} & \makecell{Fashion-\\MNIST} & \makecell{ImageNetV2\\(ResNet-18)} \\
\midrule
\multicolumn{8}{@{}l}{\textit{Credal-based measures}} \\
$U_0$ ($F$) & $0.523_{\pm 0.006}$ & $0.344_{\pm 0.009}$ & $\secondscore{0.597}_{\pm 0.009}$ & $\secondscore{0.532}_{\pm 0.006}$ & $0.227_{\pm 0.012}$ & $0.133_{\pm 0.013}$ & $\secondscore{0.577}_{\pm 0.003}$ \\
$U_1$ (GH) & $0.514_{\pm 0.007}$ & $0.342_{\pm 0.009}$ & $0.589_{\pm 0.009}$ & $0.528_{\pm 0.006}$ & $0.221_{\pm 0.012}$ & $0.129_{\pm 0.013}$ & $0.572_{\pm 0.004}$ \\
$U_2$ (sat.) & $0.508_{\pm 0.007}$ & $0.341_{\pm 0.009}$ & $0.585_{\pm 0.009}$ & $0.525_{\pm 0.006}$ & $0.217_{\pm 0.012}$ & $0.128_{\pm 0.013}$ & $0.568_{\pm 0.004}$ \\
$U_\infty$ (TV-MMI) & $0.499_{\pm 0.007}$ & $0.340_{\pm 0.010}$ & $0.579_{\pm 0.009}$ & $0.521_{\pm 0.006}$ & $0.211_{\pm 0.013}$ & $0.124_{\pm 0.013}$ & $0.562_{\pm 0.005}$ \\
$I_{\mathrm H}$ & $0.507_{\pm 0.007}$ & $0.341_{\pm 0.009}$ & $0.585_{\pm 0.009}$ & $0.525_{\pm 0.006}$ & $0.216_{\pm 0.012}$ & $0.127_{\pm 0.013}$ & $0.569_{\pm 0.004}$ \\
\midrule
\multicolumn{8}{@{}l}{\textit{Probability-vector-based measures}} \\
\bottomrule
\end{tabular}
}
\end{table}
\Needspace{26\baselineskip}
\subsubsection{Association with log loss}
\begin{table}[H]
\centering
\scriptsize
\setlength{\tabcolsep}{3pt}
\renewcommand{\arraystretch}{1.00}
\caption{LAC: Pearson correlation with NLL for all seven dataset/model conditions. Mean$_{\pm\mathrm{SE}}$ over ten paired splits; conventions as in Appendix~\ref{app:selective-protocol}.}
\label{tab:selective-lac-pearson_nll}
\resizebox{\textwidth}{!}{%
\begin{tabular}{@{}lccccccc@{}}
\toprule
Summary & \makecell{CIFAR-\\100} & \makecell{Caltech-\\101} & \makecell{Tiny\\ImageNet} & ImageNetV2 & \makecell{CIFAR-\\10} & \makecell{Fashion-\\MNIST} & \makecell{ImageNetV2\\(ResNet-18)} \\
\midrule
\multicolumn{8}{@{}l}{\textit{Credal-based measures}} \\
$U_0$ ($F$) & $0.408_{\pm 0.008}$ & $0.657_{\pm 0.024}$ & $0.467_{\pm 0.012}$ & $0.451_{\pm 0.010}$ & $0.299_{\pm 0.010}$ & $0.269_{\pm 0.008}$ & $0.475_{\pm 0.010}$ \\
$U_1$ (GH) & $\bestscore{0.422}_{\pm 0.007}$ & $0.495_{\pm 0.016}$ & $\bestscore{0.484}_{\pm 0.013}$ & $0.540_{\pm 0.010}$ & $\secondscore{0.307}_{\pm 0.011}$ & $\secondscore{0.271}_{\pm 0.008}$ & $\secondscore{0.490}_{\pm 0.011}$ \\
$U_2$ (sat.) & $0.387_{\pm 0.007}$ & $0.449_{\pm 0.015}$ & $0.444_{\pm 0.013}$ & $0.447_{\pm 0.009}$ & $0.305_{\pm 0.011}$ & $0.268_{\pm 0.008}$ & $0.378_{\pm 0.010}$ \\
$U_\infty$ (TV-MMI) & $0.339_{\pm 0.007}$ & $0.418_{\pm 0.014}$ & $0.394_{\pm 0.012}$ & $0.355_{\pm 0.008}$ & $0.296_{\pm 0.011}$ & $0.259_{\pm 0.008}$ & $0.272_{\pm 0.009}$ \\
$I_{\mathrm H}$ & $\secondscore{0.415}_{\pm 0.007}$ & $0.513_{\pm 0.015}$ & $0.468_{\pm 0.013}$ & $0.509_{\pm 0.010}$ & $\bestscore{0.308}_{\pm 0.011}$ & $\bestscore{0.275}_{\pm 0.008}$ & $0.462_{\pm 0.011}$ \\
\midrule
\multicolumn{8}{@{}l}{\textit{Probability-vector-based measures}} \\
\bottomrule
\end{tabular}
}
\end{table}
\begin{table}[H]
\centering
\scriptsize
\setlength{\tabcolsep}{3pt}
\renewcommand{\arraystretch}{1.00}
\caption{LAC: Spearman correlation with NLL for all seven dataset/model conditions. Mean$_{\pm\mathrm{SE}}$ over ten paired splits; conventions as in Appendix~\ref{app:selective-protocol}.}
\label{tab:selective-lac-spearman_nll}
\resizebox{\textwidth}{!}{%
\begin{tabular}{@{}lccccccc@{}}
\toprule
Summary & \makecell{CIFAR-\\100} & \makecell{Caltech-\\101} & \makecell{Tiny\\ImageNet} & ImageNetV2 & \makecell{CIFAR-\\10} & \makecell{Fashion-\\MNIST} & \makecell{ImageNetV2\\(ResNet-18)} \\
\midrule
\multicolumn{8}{@{}l}{\textit{Credal-based measures}} \\
$U_0$ ($F$) & $0.806_{\pm 0.004}$ & $\secondscore{0.945}_{\pm 0.007}$ & $0.802_{\pm 0.007}$ & $0.711_{\pm 0.007}$ & $0.978_{\pm 0.002}$ & $0.822_{\pm 0.002}$ & $0.728_{\pm 0.004}$ \\
$U_1$ (GH) & $\secondscore{0.837}_{\pm 0.005}$ & $0.902_{\pm 0.005}$ & $\bestscore{0.824}_{\pm 0.007}$ & $\secondscore{0.829}_{\pm 0.006}$ & $0.986_{\pm 0.001}$ & $0.830_{\pm 0.002}$ & $\secondscore{0.788}_{\pm 0.004}$ \\
$U_2$ (sat.) & $0.822_{\pm 0.006}$ & $0.873_{\pm 0.004}$ & $0.798_{\pm 0.008}$ & $0.764_{\pm 0.006}$ & $0.983_{\pm 0.001}$ & $0.833_{\pm 0.002}$ & $0.679_{\pm 0.006}$ \\
$U_\infty$ (TV-MMI) & $0.786_{\pm 0.007}$ & $0.843_{\pm 0.003}$ & $0.749_{\pm 0.009}$ & $0.697_{\pm 0.007}$ & $0.971_{\pm 0.001}$ & $0.835_{\pm 0.002}$ & $0.559_{\pm 0.006}$ \\
$I_{\mathrm H}$ & $0.834_{\pm 0.005}$ & $0.874_{\pm 0.004}$ & $0.819_{\pm 0.007}$ & $0.809_{\pm 0.006}$ & $0.979_{\pm 0.001}$ & $0.834_{\pm 0.002}$ & $0.771_{\pm 0.005}$ \\
\midrule
\multicolumn{8}{@{}l}{\textit{Probability-vector-based measures}} \\
\bottomrule
\end{tabular}
}
\end{table}
\begin{table}[H]
\centering
\scriptsize
\setlength{\tabcolsep}{3pt}
\renewcommand{\arraystretch}{1.00}
\caption{APS: Pearson correlation with NLL for all seven dataset/model conditions. Mean$_{\pm\mathrm{SE}}$ over ten paired splits; conventions as in Appendix~\ref{app:selective-protocol}.}
\label{tab:selective-aps-pearson_nll}
\resizebox{\textwidth}{!}{%
\begin{tabular}{@{}lccccccc@{}}
\toprule
Summary & \makecell{CIFAR-\\100} & \makecell{Caltech-\\101} & \makecell{Tiny\\ImageNet} & ImageNetV2 & \makecell{CIFAR-\\10} & \makecell{Fashion-\\MNIST} & \makecell{ImageNetV2\\(ResNet-18)} \\
\midrule
\multicolumn{8}{@{}l}{\textit{Credal-based measures}} \\
$U_0$ ($F$) & $0.365_{\pm 0.008}$ & $0.615_{\pm 0.013}$ & $0.432_{\pm 0.012}$ & $0.358_{\pm 0.010}$ & $0.240_{\pm 0.010}$ & $0.167_{\pm 0.009}$ & $0.425_{\pm 0.009}$ \\
$U_1$ (GH) & $0.399_{\pm 0.008}$ & $0.668_{\pm 0.013}$ & $0.466_{\pm 0.012}$ & $0.497_{\pm 0.010}$ & $0.232_{\pm 0.009}$ & $0.162_{\pm 0.007}$ & $0.500_{\pm 0.010}$ \\
$U_2$ (sat.) & $0.401_{\pm 0.008}$ & $0.692_{\pm 0.012}$ & $0.478_{\pm 0.013}$ & $\bestscore{0.534}_{\pm 0.009}$ & $0.226_{\pm 0.009}$ & $0.153_{\pm 0.007}$ & $0.510_{\pm 0.011}$ \\
$U_\infty$ (TV-MMI) & $0.391_{\pm 0.008}$ & $\secondscore{0.703}_{\pm 0.012}$ & $0.478_{\pm 0.013}$ & $\secondscore{0.531}_{\pm 0.009}$ & $0.218_{\pm 0.009}$ & $0.141_{\pm 0.007}$ & $0.498_{\pm 0.011}$ \\
$I_{\mathrm H}$ & $\bestscore{0.421}_{\pm 0.008}$ & $0.673_{\pm 0.012}$ & $\bestscore{0.483}_{\pm 0.013}$ & $0.518_{\pm 0.010}$ & $0.267_{\pm 0.011}$ & $\secondscore{0.202}_{\pm 0.007}$ & $\bestscore{0.513}_{\pm 0.011}$ \\
\midrule
\multicolumn{8}{@{}l}{\textit{Probability-vector-based measures}} \\
\bottomrule
\end{tabular}
}
\end{table}
\begin{table}[H]
\centering
\scriptsize
\setlength{\tabcolsep}{3pt}
\renewcommand{\arraystretch}{1.00}
\caption{APS: Spearman correlation with NLL for all seven dataset/model conditions. Mean$_{\pm\mathrm{SE}}$ over ten paired splits; conventions as in Appendix~\ref{app:selective-protocol}.}
\label{tab:selective-aps-spearman_nll}
\resizebox{\textwidth}{!}{%
\begin{tabular}{@{}lccccccc@{}}
\toprule
Summary & \makecell{CIFAR-\\100} & \makecell{Caltech-\\101} & \makecell{Tiny\\ImageNet} & ImageNetV2 & \makecell{CIFAR-\\10} & \makecell{Fashion-\\MNIST} & \makecell{ImageNetV2\\(ResNet-18)} \\
\midrule
\multicolumn{8}{@{}l}{\textit{Credal-based measures}} \\
$U_0$ ($F$) & $0.787_{\pm 0.004}$ & $0.833_{\pm 0.005}$ & $0.786_{\pm 0.007}$ & $0.594_{\pm 0.008}$ & $0.952_{\pm 0.002}$ & $0.484_{\pm 0.015}$ & $0.706_{\pm 0.004}$ \\
$U_1$ (GH) & $0.824_{\pm 0.005}$ & $0.874_{\pm 0.004}$ & $0.813_{\pm 0.007}$ & $0.754_{\pm 0.007}$ & $0.966_{\pm 0.001}$ & $0.659_{\pm 0.008}$ & $0.776_{\pm 0.004}$ \\
$U_2$ (sat.) & $0.830_{\pm 0.005}$ & $0.892_{\pm 0.003}$ & $0.821_{\pm 0.007}$ & $0.809_{\pm 0.007}$ & $0.973_{\pm 0.001}$ & $0.763_{\pm 0.008}$ & $\secondscore{0.785}_{\pm 0.005}$ \\
$U_\infty$ (TV-MMI) & $0.827_{\pm 0.006}$ & $0.905_{\pm 0.003}$ & $0.821_{\pm 0.007}$ & $\secondscore{0.823}_{\pm 0.007}$ & $0.977_{\pm 0.001}$ & $0.822_{\pm 0.007}$ & $0.780_{\pm 0.005}$ \\
$I_{\mathrm H}$ & $0.828_{\pm 0.005}$ & $0.884_{\pm 0.003}$ & $0.818_{\pm 0.007}$ & $0.792_{\pm 0.007}$ & $0.975_{\pm 0.001}$ & $0.804_{\pm 0.006}$ & $0.779_{\pm 0.004}$ \\
\midrule
\multicolumn{8}{@{}l}{\textit{Probability-vector-based measures}} \\
\bottomrule
\end{tabular}
}
\end{table}
\begin{table}[H]
\centering
\scriptsize
\setlength{\tabcolsep}{3pt}
\renewcommand{\arraystretch}{1.00}
\caption{RAPS: Pearson correlation with NLL for all seven dataset/model conditions. Mean$_{\pm\mathrm{SE}}$ over ten paired splits; conventions as in Appendix~\ref{app:selective-protocol}.}
\label{tab:selective-raps-pearson_nll}
\resizebox{\textwidth}{!}{%
\begin{tabular}{@{}lccccccc@{}}
\toprule
Summary & \makecell{CIFAR-\\100} & \makecell{Caltech-\\101} & \makecell{Tiny\\ImageNet} & ImageNetV2 & \makecell{CIFAR-\\10} & \makecell{Fashion-\\MNIST} & \makecell{ImageNetV2\\(ResNet-18)} \\
\midrule
\multicolumn{8}{@{}l}{\textit{Credal-based measures}} \\
$U_0$ ($F$) & $0.357_{\pm 0.008}$ & $0.635_{\pm 0.013}$ & $0.415_{\pm 0.011}$ & $0.472_{\pm 0.009}$ & $0.230_{\pm 0.010}$ & $0.169_{\pm 0.008}$ & $0.447_{\pm 0.010}$ \\
$U_1$ (GH) & $0.385_{\pm 0.008}$ & $0.675_{\pm 0.013}$ & $0.457_{\pm 0.012}$ & $0.522_{\pm 0.009}$ & $0.226_{\pm 0.010}$ & $0.156_{\pm 0.007}$ & $0.497_{\pm 0.010}$ \\
$U_2$ (sat.) & $0.389_{\pm 0.008}$ & $0.694_{\pm 0.012}$ & $0.471_{\pm 0.012}$ & $\bestscore{0.535}_{\pm 0.009}$ & $0.223_{\pm 0.010}$ & $0.147_{\pm 0.007}$ & $0.506_{\pm 0.011}$ \\
$U_\infty$ (TV-MMI) & $0.383_{\pm 0.009}$ & $\secondscore{0.703}_{\pm 0.012}$ & $0.472_{\pm 0.013}$ & $0.531_{\pm 0.009}$ & $0.216_{\pm 0.009}$ & $0.136_{\pm 0.007}$ & $0.497_{\pm 0.011}$ \\
$I_{\mathrm H}$ & $\bestscore{0.412}_{\pm 0.008}$ & $0.679_{\pm 0.011}$ & $\bestscore{0.480}_{\pm 0.013}$ & $\secondscore{0.534}_{\pm 0.010}$ & $0.263_{\pm 0.011}$ & $\secondscore{0.195}_{\pm 0.007}$ & $\secondscore{0.512}_{\pm 0.011}$ \\
\midrule
\multicolumn{8}{@{}l}{\textit{Probability-vector-based measures}} \\
\bottomrule
\end{tabular}
}
\end{table}
\begin{table}[H]
\centering
\scriptsize
\setlength{\tabcolsep}{3pt}
\renewcommand{\arraystretch}{1.00}
\caption{RAPS: Spearman correlation with NLL for all seven dataset/model conditions. Mean$_{\pm\mathrm{SE}}$ over ten paired splits; conventions as in Appendix~\ref{app:selective-protocol}.}
\label{tab:selective-raps-spearman_nll}
\resizebox{\textwidth}{!}{%
\begin{tabular}{@{}lccccccc@{}}
\toprule
Summary & \makecell{CIFAR-\\100} & \makecell{Caltech-\\101} & \makecell{Tiny\\ImageNet} & ImageNetV2 & \makecell{CIFAR-\\10} & \makecell{Fashion-\\MNIST} & \makecell{ImageNetV2\\(ResNet-18)} \\
\midrule
\multicolumn{8}{@{}l}{\textit{Credal-based measures}} \\
$U_0$ ($F$) & $0.825_{\pm 0.005}$ & $0.859_{\pm 0.004}$ & $0.812_{\pm 0.007}$ & $0.732_{\pm 0.007}$ & $0.961_{\pm 0.001}$ & $0.642_{\pm 0.009}$ & $0.775_{\pm 0.004}$ \\
$U_1$ (GH) & $0.830_{\pm 0.005}$ & $0.881_{\pm 0.004}$ & $0.819_{\pm 0.007}$ & $0.792_{\pm 0.007}$ & $0.969_{\pm 0.001}$ & $0.725_{\pm 0.008}$ & $0.784_{\pm 0.004}$ \\
$U_2$ (sat.) & $0.830_{\pm 0.005}$ & $0.894_{\pm 0.003}$ & $0.821_{\pm 0.007}$ & $0.813_{\pm 0.007}$ & $0.973_{\pm 0.001}$ & $0.779_{\pm 0.008}$ & $0.785_{\pm 0.005}$ \\
$U_\infty$ (TV-MMI) & $0.827_{\pm 0.006}$ & $0.905_{\pm 0.003}$ & $0.821_{\pm 0.007}$ & $\secondscore{0.823}_{\pm 0.007}$ & $0.977_{\pm 0.001}$ & $0.825_{\pm 0.009}$ & $0.780_{\pm 0.005}$ \\
$I_{\mathrm H}$ & $0.832_{\pm 0.005}$ & $0.890_{\pm 0.003}$ & $0.822_{\pm 0.007}$ & $0.813_{\pm 0.007}$ & $0.975_{\pm 0.001}$ & $0.817_{\pm 0.007}$ & $\secondscore{0.786}_{\pm 0.004}$ \\
\midrule
\multicolumn{8}{@{}l}{\textit{Probability-vector-based measures}} \\
\bottomrule
\end{tabular}
}
\end{table}
\begin{table}[H]
\centering
\scriptsize
\setlength{\tabcolsep}{3pt}
\renewcommand{\arraystretch}{1.00}
\caption{SAPS: Pearson correlation with NLL for all seven dataset/model conditions. Mean$_{\pm\mathrm{SE}}$ over ten paired splits; conventions as in Appendix~\ref{app:selective-protocol}.}
\label{tab:selective-saps-pearson_nll}
\resizebox{\textwidth}{!}{%
\begin{tabular}{@{}lccccccc@{}}
\toprule
Summary & \makecell{CIFAR-\\100} & \makecell{Caltech-\\101} & \makecell{Tiny\\ImageNet} & ImageNetV2 & \makecell{CIFAR-\\10} & \makecell{Fashion-\\MNIST} & \makecell{ImageNetV2\\(ResNet-18)} \\
\midrule
\multicolumn{8}{@{}l}{\textit{Credal-based measures}} \\
$U_0$ ($F$) & $0.344_{\pm 0.007}$ & $0.716_{\pm 0.013}$ & $0.459_{\pm 0.012}$ & $\secondscore{0.522}_{\pm 0.009}$ & $0.165_{\pm 0.011}$ & $0.052_{\pm 0.012}$ & $0.489_{\pm 0.010}$ \\
$U_1$ (GH) & $0.344_{\pm 0.008}$ & $\secondscore{0.718}_{\pm 0.012}$ & $0.456_{\pm 0.012}$ & $0.517_{\pm 0.009}$ & $0.168_{\pm 0.011}$ & $0.054_{\pm 0.012}$ & $0.484_{\pm 0.010}$ \\
$U_2$ (sat.) & $0.346_{\pm 0.008}$ & $0.718_{\pm 0.012}$ & $0.456_{\pm 0.012}$ & $0.516_{\pm 0.009}$ & $0.169_{\pm 0.011}$ & $0.055_{\pm 0.012}$ & $0.482_{\pm 0.010}$ \\
$U_\infty$ (TV-MMI) & $0.348_{\pm 0.008}$ & $0.713_{\pm 0.012}$ & $0.457_{\pm 0.012}$ & $0.512_{\pm 0.009}$ & $0.171_{\pm 0.011}$ & $0.056_{\pm 0.012}$ & $0.478_{\pm 0.010}$ \\
$I_{\mathrm H}$ & $0.359_{\pm 0.008}$ & $0.716_{\pm 0.011}$ & $0.468_{\pm 0.012}$ & $0.520_{\pm 0.009}$ & $0.179_{\pm 0.012}$ & $0.058_{\pm 0.013}$ & $0.481_{\pm 0.011}$ \\
\midrule
\multicolumn{8}{@{}l}{\textit{Probability-vector-based measures}} \\
\bottomrule
\end{tabular}
}
\end{table}
\begin{table}[H]
\centering
\scriptsize
\setlength{\tabcolsep}{3pt}
\renewcommand{\arraystretch}{1.00}
\caption{SAPS: Spearman correlation with NLL for all seven dataset/model conditions. Mean$_{\pm\mathrm{SE}}$ over ten paired splits; conventions as in Appendix~\ref{app:selective-protocol}.}
\label{tab:selective-saps-spearman_nll}
\resizebox{\textwidth}{!}{%
\begin{tabular}{@{}lccccccc@{}}
\toprule
Summary & \makecell{CIFAR-\\100} & \makecell{Caltech-\\101} & \makecell{Tiny\\ImageNet} & ImageNetV2 & \makecell{CIFAR-\\10} & \makecell{Fashion-\\MNIST} & \makecell{ImageNetV2\\(ResNet-18)} \\
\midrule
\multicolumn{8}{@{}l}{\textit{Credal-based measures}} \\
$U_0$ ($F$) & $0.782_{\pm 0.006}$ & $0.907_{\pm 0.003}$ & $0.788_{\pm 0.005}$ & $\secondscore{0.832}_{\pm 0.006}$ & $0.221_{\pm 0.009}$ & $0.125_{\pm 0.013}$ & $0.777_{\pm 0.005}$ \\
$U_1$ (GH) & $0.756_{\pm 0.006}$ & $0.891_{\pm 0.002}$ & $0.764_{\pm 0.005}$ & $0.815_{\pm 0.006}$ & $0.212_{\pm 0.009}$ & $0.120_{\pm 0.013}$ & $0.766_{\pm 0.005}$ \\
$U_2$ (sat.) & $0.740_{\pm 0.006}$ & $0.886_{\pm 0.002}$ & $0.751_{\pm 0.005}$ & $0.804_{\pm 0.006}$ & $0.206_{\pm 0.009}$ & $0.116_{\pm 0.013}$ & $0.759_{\pm 0.005}$ \\
$U_\infty$ (TV-MMI) & $0.721_{\pm 0.006}$ & $0.875_{\pm 0.003}$ & $0.738_{\pm 0.005}$ & $0.791_{\pm 0.006}$ & $0.200_{\pm 0.010}$ & $0.112_{\pm 0.013}$ & $0.749_{\pm 0.005}$ \\
$I_{\mathrm H}$ & $0.736_{\pm 0.006}$ & $0.884_{\pm 0.002}$ & $0.749_{\pm 0.005}$ & $0.802_{\pm 0.006}$ & $0.206_{\pm 0.009}$ & $0.117_{\pm 0.013}$ & $0.758_{\pm 0.005}$ \\
\midrule
\multicolumn{8}{@{}l}{\textit{Probability-vector-based measures}} \\
\bottomrule
\end{tabular}
}
\end{table}

\section{OOD AND CORRUPTION DETECTION}
\label{app:ood-details}

\subsection{Protocol and evaluation}

\paragraph{Benchmarks and fixed predictor.}
The in-distribution (ID) source is CIFAR-10. We use the same fixed pretrained CIFAR-10 ResNet-56 for all comparisons and reuse its cached class probabilities. The five dataset shifts are CIFAR-100, CIFAR-10.1 v6, Tiny ImageNet, SVHN, and Fashion-MNIST. CIFAR-10.1 preserves the ten-class task and represents a subtle collection shift; the other four datasets test semantic or more distant shifts. The corruption benchmark contains every combination of the 15 CIFAR-10-C corruption families and severities 1--5, giving 75 conditions. Images are converted to RGB, resized to $32\times32$, and normalised with CIFAR-10 channel means $(.4914,.4822,.4465)$ and standard deviations $(.2023,.1994,.2010)$. No shifted examples or labels are used to fit the classifier, calibrate the contour, or choose a score parameter.

\paragraph{Paired sampling and evaluation.}
For each seed $0,\ldots,9$, we permute the 10,000 clean CIFAR-10 examples, assign the first 5,000 to calibration, and reserve the remaining 5,000 for ID evaluation. The same random generator then permutes the OOD cache. We evaluate a balanced sample of 5,000 ID and 5,000 OOD observations, except for CIFAR-10.1, which contains 2,000 OOD observations and is compared with the first 2,000 observations from the ID evaluation permutation. Clean calibration and ID evaluation indices are disjoint. All scores and uncertainty summaries use the same evaluation rows within a seed. AUROC treats OOD as the positive class, so a larger uncertainty score should indicate an OOD observation.

\paragraph{Corruption-sampling limitation.}
We retain the original paired protocol: corrupted observations are sampled independently from the full cache of 10,000 base images. Consequently, the corrupted view of an image used for clean calibration can appear in the OOD evaluation sample. The clean calibration and clean evaluation sets are disjoint, but the clean-calibration and corrupted-evaluation sets need not be disjoint at the base-image level. These results should not be interpreted as an evaluation on entirely unseen base images. The same sampling rule is used for all four nonconformity scores.

\paragraph{Evaluated summaries and cardinalities.}
We use the summaries and score settings in Appendix~\ref{appendix: exp}. LAC and APS include all five core summaries; RAPS and SAPS include the four original summaries. Every score includes the three probability baselines and raw/normalised cardinalities at $\alpha\in\{.01,.05,.10\}$. Normalisation replaces an empty cut by a singleton. In particular, raw LAC sets at $\alpha=.10$ are empty or singleton on these evaluated rows, so normalised sizes are constant and have AUROC $50\%$. The strongest tested LAC size remains $\alpha=.01$, with no empty sets at that level.

\paragraph{Aggregation and uncertainty.}
Entries are AUROC percentages, with means and one standard error over ten paired seeds. For CIFAR-10-C, we first average all 75 condition-specific AUROCs within each seed, then compute the across-seed mean and standard error. Paired differences are formed within seeds before averaging. These standard errors describe variation across overlapping resamplings of a fixed model and fixed data; they do not include model-training variation, and the 75 corruption conditions are not independent experimental replicates. \ranklegend{} Rankings use unrounded means, compare methods within each dataset column in aggregate tables and within each condition row in the detailed tables, and are descriptive rather than significance claims. Tied means share a rank.

\subsection{LAC results and the second-probability control}
LAC GH exceeds predictive entropy on each of the six aggregate benchmarks, in every paired seed, and on the mean of each of the 75 corruption conditions. On the corruption average, the paired GH gains are $0.305\pm0.020$ percentage points over entropy and $0.335\pm0.025$ over $q_{(2)}$. This is not uniform dominance over probability baselines: on CIFAR-10.1, GH trails $q_{(2)}$ by $0.047\pm0.038$ points, while the highest-mean core summary, $I_{\mathrm H}$, exceeds it by only $0.036\pm0.017$ points (seven of ten seeds). Fuzziness has the highest mean on SVHN and Fashion-MNIST; GH leads CIFAR-100, Tiny ImageNet, and the corruption average. GH, saturation, and $I_H$ each exceed $1-q_{(1)}$ on all six aggregates. Entropy also exceeds $1-q_{(1)}$ throughout, so including this baseline changes no aggregate winner.

The saturating summary exceeds TV-MMI in aggregate mean on all six benchmarks and every reported probability baseline on CIFAR-100, Tiny ImageNet, and the corruption average. Its paired gains over entropy on these three benchmarks are $0.272\pm0.013$, $0.285\pm0.009$, and $0.244\pm0.008$ percentage points, respectively. On CIFAR-10.1 it also exceeds all three probability baselines in mean, although the gain over $q_{(2)}$ is only $0.018\pm0.019$ points. GH remains stronger than saturation on five of the six aggregates; saturation trails entropy on SVHN and Fashion-MNIST. These differences are consistent with the stronger emphasis of $U_2$ on the highest contour ranks, whereas GH retains more influence from lower-ranked labels. They illustrate the practical role of the penalty choice within the same calibrated contour.

For deterministic LAC, TV-MMI equals $T_n(q_{(2)}(x))$, with $T_n$ defined in Appendix~\ref{app:score-structure}. It therefore supplies no additional strict ranking beyond $q_{(2)}$; finite-calibration ties can change AUROC. We verified this identity and the corresponding tie-only explanation of the AUROC differences in all 800 condition--seed pairs. The remaining core summaries can use more of the calibrated contour, but their empirical advantage remains score- and dataset-dependent.

\subsection{Complete comparison of LAC, APS, RAPS, and SAPS}
\label{app:ood-score-ablation}

All four constructions use the fixed settings and implementations in Appendix~\ref{app:nonconformity-scores}. For OOD evaluation, the common NumPy sequence generates the clean permutation, OOD permutation, calibration uniforms, ID uniforms, and OOD uniforms, in that order. Uniform matrices are indexed by sorted class rank. Different sorting conventions can assign different draws to tied classes, while evaluation samples remain identical. The task-specific verification checks are recorded in Appendix~\ref{app:reproducibility}.

Tables~\ref{tab:ood-aggregate-lac}--\ref{tab:ood-aggregate-saps} retain every aggregate, the summaries evaluated for each score, all three probability baselines, and all raw and normalised cardinalities. RAPS is broadly similar to APS: fuzziness leads on SVHN and Fashion-MNIST, while predictive entropy leads CIFAR-100, Tiny ImageNet, and the corruption average, and $q_{(2)}$ leads CIFAR-10.1. SAPS is substantially weaker on these data: all four core summaries fall below the three probability baselines on every aggregate. Nevertheless, fuzziness exceeds every tested raw and normalised fixed-level cardinality in aggregate mean for both RAPS and SAPS. APS and RAPS each have a credal summary exceeding $1-q_{(1)}$ on CIFAR-100, Tiny ImageNet, SVHN, and Fashion-MNIST, but not CIFAR-10.1 or the corruption average. This contrasts with selective prediction, where $1-q_{(1)}$ exceeds all their credal summaries in mean AUARC. The results support using the full contour beyond a few fixed cuts, with task- and score-dependent comparisons against probability summaries. The SAPS result is specific to the fixed weight, randomisation, calibration size, and classifier tested here; neither a parameter search nor a deterministic-SAPS or larger-calibration ablation is implied.

\begin{table}[H]
\centering\footnotesize
\setlength{\tabcolsep}{4pt}
\caption{LAC OOD AUROC (\%), mean $\pm$ one standard error. C100, C10.1, Tiny, and Fashion abbreviate CIFAR-100, CIFAR-10.1, Tiny ImageNet, and Fashion-MNIST; C10-C averages all 75 corruption conditions. Five credal summaries, including $U_2$, and both raw and normalised cardinalities are shown. \ranklegend{}}
\label{tab:ood-aggregate-lac}
\begin{tabular}{@{}lcccccc@{}}
\toprule
Ranking & C100 & C10.1 & Tiny & SVHN & Fashion & C10-C \\
\midrule
\multicolumn{7}{@{}l}{\textit{Credal-based measures}}\\
$U_0$ ($F$) & $\secondscore{86.08}_{\pm 0.08}$ & $59.84_{\pm 0.17}$ & $\secondscore{85.80}_{\pm 0.09}$ & $\bestscore{93.80}_{\pm 0.06}$ & $\bestscore{96.71}_{\pm 0.03}$ & $68.76_{\pm 0.08}$ \\
$U_1$ (GH) & $\bestscore{86.24}_{\pm 0.08}$ & $60.08_{\pm 0.17}$ & $\bestscore{85.88}_{\pm 0.08}$ & $\secondscore{93.34}_{\pm 0.07}$ & $\secondscore{95.72}_{\pm 0.03}$ & $\bestscore{68.93}_{\pm 0.08}$ \\
$U_2$ (sat.) & $86.01_{\pm 0.08}$ & $\secondscore{60.15}_{\pm 0.17}$ & $85.59_{\pm 0.08}$ & $92.61_{\pm 0.07}$ & $94.59_{\pm 0.04}$ & $\secondscore{68.87}_{\pm 0.09}$ \\
$U_\infty$ (TV-MMI) & $85.35_{\pm 0.09}$ & $60.09_{\pm 0.17}$ & $84.87_{\pm 0.09}$ & $91.20_{\pm 0.08}$ & $92.81_{\pm 0.06}$ & $68.58_{\pm 0.09}$ \\
$I_{\mathrm H}$ & $85.94_{\pm 0.09}$ & $\bestscore{60.17}_{\pm 0.17}$ & $85.51_{\pm 0.09}$ & $92.39_{\pm 0.08}$ & $94.24_{\pm 0.05}$ & $68.85_{\pm 0.09}$ \\
\midrule
\multicolumn{7}{@{}l}{\textit{Probability-vector-based measures}}\\
\bottomrule
\end{tabular}
\end{table}

\begin{table}[H]
\centering\footnotesize
\setlength{\tabcolsep}{4pt}
\caption{APS OOD AUROC (\%), with the same benchmarks, baselines, aggregation, and paired seeds as Table~\ref{tab:ood-aggregate-lac}. \ranklegend{}}
\label{tab:ood-aggregate-aps}
\begin{tabular}{@{}lcccccc@{}}
\toprule
Ranking & C100 & C10.1 & Tiny & SVHN & Fashion & C10-C \\
\midrule
\multicolumn{7}{@{}l}{\textit{Credal-based measures}}\\
$U_0$ ($F$) & $84.99_{\pm 0.08}$ & $59.28_{\pm 0.19}$ & $84.76_{\pm 0.09}$ & $\bestscore{93.08}_{\pm 0.07}$ & $96.26_{\pm 0.03}$ & $68.04_{\pm 0.09}$ \\
$U_1$ (GH) & $85.29_{\pm 0.07}$ & $59.50_{\pm 0.19}$ & $84.97_{\pm 0.09}$ & $92.82_{\pm 0.08}$ & $95.58_{\pm 0.04}$ & $68.27_{\pm 0.09}$ \\
$U_2$ (sat.) & $85.42_{\pm 0.07}$ & $59.63_{\pm 0.19}$ & $85.05_{\pm 0.09}$ & $92.62_{\pm 0.08}$ & $95.14_{\pm 0.04}$ & $68.38_{\pm 0.09}$ \\
$U_\infty$ (TV-MMI) & $85.49_{\pm 0.07}$ & $59.77_{\pm 0.19}$ & $85.06_{\pm 0.08}$ & $92.36_{\pm 0.08}$ & $94.64_{\pm 0.04}$ & $68.46_{\pm 0.09}$ \\
$I_{\mathrm H}$ & $\secondscore{85.61}_{\pm 0.07}$ & $59.72_{\pm 0.19}$ & $\secondscore{85.23}_{\pm 0.08}$ & $\secondscore{92.84}_{\pm 0.08}$ & $95.33_{\pm 0.04}$ & $68.48_{\pm 0.09}$ \\
\midrule
\multicolumn{7}{@{}l}{\textit{Probability-vector-based measures}}\\
\bottomrule
\end{tabular}
\end{table}

\begin{table}[H]
\centering\footnotesize
\setlength{\tabcolsep}{4pt}
\caption{RAPS OOD AUROC (\%), using fixed penalty $.01$ and $k_{\rm reg}=5$, with the same benchmarks, probability and cardinality baselines, aggregation, and paired seeds as Table~\ref{tab:ood-aggregate-lac}. The four original credal summaries are evaluated. \ranklegend{}}
\label{tab:ood-aggregate-raps}
\input{tables/generated_ood_aggregate_raps}
\end{table}

\begin{table}[H]
\centering\footnotesize
\setlength{\tabcolsep}{4pt}
\caption{SAPS OOD AUROC (\%), using fixed weight $.2$, with the same benchmarks, probability and cardinality baselines, aggregation, and paired seeds as Table~\ref{tab:ood-aggregate-lac}. The four original credal summaries are evaluated. \ranklegend{}}
\label{tab:ood-aggregate-saps}
\input{tables/generated_ood_aggregate_saps}
\end{table}

\subsection{Calibration-size sensitivity}
\label{app:ood-calibration-sensitivity}

This supplementary study uses randomised APS and the 15 severity-3 corruptions. Table~\ref{tab:ood_calibration_ablation} varies the number of calibration observations while holding the evaluation observations fixed within each repetition. Calibration sets are nested prefixes of a randomly permuted 5,000-observation pool. We aggregate the five dataset shifts and the 15 CIFAR-10-C severity-3 conditions separately. We show $|C^{\mathrm{raw}}_{0.01}|$, the strongest raw set-size method in the original APS OOD averages; all other fixed levels are lower. With only 100 observations, $|C^{\mathrm{raw}}_{0.01}|$ is modestly stronger, and the methods are nearly tied for the five dataset shifts at 250 observations. From 500 calibration observations onward, TV-MMI is consistently stronger on both benchmark families. This comparison suggests that the advantage emerges once the calibration sample resolves the contour sufficiently well; in the extreme small-sample regime, fixed-level cardinality can remain competitive.

\begin{table}[H]
    \centering
    \caption{Calibration-size sensitivity of distribution-shift AUROC (on a 0--1 scale). Entries are means with one standard error over 10 randomised-APS score draws. Dataset-shift results average over five datasets; CIFAR-10-C results average over all 15 severity-3 corruptions. These are supplementary APS results, separate from the main LAC comparison. \ranklegend{} Methods are compared within each row and benchmark family.}
    \label{tab:ood_calibration_ablation}
    \resizebox{0.92\textwidth}{!}{%
    \begin{tabular}{rcccc}
        \toprule
        & \multicolumn{2}{c}{Five dataset shifts} & \multicolumn{2}{c}{CIFAR-10-C, severity 3} \\
        \cmidrule(lr){2-3}\cmidrule(lr){4-5}
        & \textit{Credal-based} & \textit{Set-size-based} & \textit{Credal-based} & \textit{Set-size-based} \\
        $n_{\mathrm{cal}}$ & $\text{TV-MMI}$ & $|C^{\mathrm{raw}}_{0.01}|$ & $\text{TV-MMI}$ & $|C^{\mathrm{raw}}_{0.01}|$ \\
        \midrule
        100 & ${\secondscore{0.8140}_{\pm .0051}}$ & ${\bestscore{0.8224}_{\pm .0074}}$ & ${\secondscore{0.6565}_{\pm .0059}}$ & ${\bestscore{0.6600}_{\pm .0064}}$ \\
        250 & ${\secondscore{0.8281}_{\pm .0015}}$ & ${\bestscore{0.8300}_{\pm .0026}}$ & ${\bestscore{0.6720}_{\pm .0023}}$ & ${\secondscore{0.6615}_{\pm .0039}}$ \\
        500 & ${\bestscore{0.8300}_{\pm .0015}}$ & ${\secondscore{0.8217}_{\pm .0050}}$ & ${\bestscore{0.6747}_{\pm .0024}}$ & ${\secondscore{0.6503}_{\pm .0059}}$ \\
        1000 & ${\bestscore{0.8336}_{\pm .0009}}$ & ${\secondscore{0.8254}_{\pm .0040}}$ & ${\bestscore{0.6817}_{\pm .0010}}$ & ${\secondscore{0.6524}_{\pm .0047}}$ \\
        2500 & ${\bestscore{0.8344}_{\pm .0008}}$ & ${\secondscore{0.8270}_{\pm .0006}}$ & ${\bestscore{0.6831}_{\pm .0008}}$ & ${\secondscore{0.6516}_{\pm .0007}}$ \\
        5000 & ${\bestscore{0.8346}_{\pm .0008}}$ & ${\secondscore{0.8272}_{\pm .0010}}$ & ${\bestscore{0.6836}_{\pm .0008}}$ & ${\secondscore{0.6518}_{\pm .0009}}$ \\
        \bottomrule
    \end{tabular}}
\end{table}

\subsection{Every corruption and severity}
Tables~\ref{tab:ood-conditions-lac}--\ref{tab:ood-conditions-saps} report all 75 corruption--severity conditions for every score construction. Each table retains the four original core summaries and the three probability baselines; the LAC and APS tables additionally include $U_2$. No conditions are selected according to their outcome. The aggregate tables above additionally report all six raw and normalised cardinality baselines.

\input{tables/generated_ood_conditions_lac}
\input{tables/generated_ood_conditions_aps}
\input{tables/generated_ood_conditions_raps}
\input{tables/generated_ood_conditions_saps}

\Needspace{12\baselineskip}
\section{TSALLIS PENALTY COMPARISONS AND TAIL DIAGNOSTICS}
\label{app:tsallis-experiments}

This appendix examines how the fixed Tsallis penalties in Corollary~\ref{cor: tsallis_path} weight the same conformal contour, and whether the small contributions of lower-ranked labels affect selective prediction. The four settings are $U_0=F/(K-1)$ (fuzziness), $U_1=\operatorname{GH}/\log_2K$, $U_2$ (saturation), and $U_\infty=\mathrm{TV\text{-}MMI}$. These settings are fixed across datasets; $q$ is not tuned on evaluation labels. Their positive normalisation constants preserve ranking and correlation metrics within each fixed label space. The saturation derivation is in Appendix~\ref{app:saturation}; numerical verification and reproduction details are in Appendix~\ref{app:reproducibility}.

\subsection{Comparison of four fixed penalties}

All comparisons reuse the seven dataset/model conditions, fixed predictions, ten paired calibration/evaluation splits, and random tie keys of Appendix~\ref{app:selective-details}. Table~\ref{tab:tsallis-overview} places the LAC and APS results side by side. Appendix~\ref{app:selective-details} reports the complete selective-prediction and error/loss comparisons, including entropy imprecision and the probability-vector and fixed-size baselines; Appendix~\ref{app:ood-details} reports distribution-shift detection. Saturation is evaluated under all four scores for selective prediction and association, and under LAC and APS for OOD detection; it is not evaluated in active learning.

\begin{table}[H]
\centering
\caption{Effect of the Tsallis set-size penalty on selective prediction: mean AUARC (\%) over ten paired splits. The summaries $U_0,U_1,U_2,U_\infty$ correspond to normalised fuzziness, GH, saturation, and TV-MMI. Bold identifies the largest mean among these four summaries within each score construction; no setting is selected separately by dataset. All seven dataset/model conditions are shown. Complete results with standard errors and baselines appear in Appendix~\ref{app:selective-details}; paired differences appear below.}
\label{tab:tsallis-overview}
\small
\setlength{\tabcolsep}{5pt}
\input{tables/generated_tsallis_overview}
\end{table}

Under LAC, saturation substantially improves on TV-MMI on CIFAR-100, Tiny ImageNet, and both ImageNetV2 models, while GH retains an advantage on these conditions. On Caltech-101 and ImageNetV2--ResNetv2-50, saturation exceeds every reported probability-vector summary in mean AUARC. Relative to the strongest reported predictive baseline, the respective paired gains are $0.077\pm0.017$ points over second probability and $0.684\pm0.101$ points over $1-q_{(1)}$. Conversely, it trails entropy on CIFAR-100 and both entropy and maximum-probability uncertainty on Tiny ImageNet and ImageNetV2--ResNet-18. Under APS and RAPS, saturation improves on GH in every condition mean, while TV-MMI generally remains stronger. SAPS gives the opposite comparison with GH: saturation is lower in all seven condition means. Under APS, RAPS, and SAPS, $1-q_{(1)}$ exceeds all five credal summaries in mean AUARC on every condition (Appendix~\ref{app:score-sensitivity}). Thus the useful weighting of the contour depends on the score construction; no penalty is uniformly preferable.

These comparisons also depend on the evaluation target. With LAC on ImageNetV2--ResNet-18, the Spearman association with NLL is $0.679$ for saturation versus $0.788$ for GH, despite saturation's selective-prediction improvement over TV-MMI (Appendix~\ref{app:selective-details}). For LAC OOD detection, saturation exceeds TV-MMI on all six aggregate means but trails GH on five; CIFAR-10.1 gives a small gain over GH. It exceeds every reported probability-vector baseline on CIFAR-100, Tiny ImageNet, the CIFAR-10-C average, and narrowly on CIFAR-10.1, but trails entropy on SVHN and Fashion-MNIST (Appendix~\ref{app:ood-details}).

\subsection{Paired comparisons across score constructions}

Tables~\ref{tab:tsallis-lac-paired}--\ref{tab:tsallis-saps-paired} report $U_2$ minus each comparator in AUARC percentage points. Means and standard errors are computed from within-seed differences over ten overlapping split replicates. These paired SEs describe resampling variation conditional on the fitted model, not uncertainty across independently trained models. Comparisons are exploratory; selecting the strongest observed baseline mean is descriptive and is not a separately tuned policy.

\begin{table}[H]
\centering
\caption{LAC: paired AUARC differences (percentage points), $U_2$ minus each comparator. Positive values favour saturation. Mean $\pm$ paired SE over ten splits.}
\label{tab:tsallis-lac-paired}
\resizebox{\textwidth}{!}{\input{tables/generated_tsallis_lac_paired}}
\end{table}

\begin{table}[H]
\centering
\caption{APS: paired AUARC differences (percentage points), $U_2$ minus each comparator. Positive values favour saturation. Mean $\pm$ paired SE over ten splits.}
\label{tab:tsallis-aps-paired}
\resizebox{\textwidth}{!}{\input{tables/generated_tsallis_aps_paired}}
\end{table}

\begin{table}[H]
\centering
\caption{RAPS: paired AUARC differences (percentage points), $U_2$ minus each comparator. Positive values favour saturation. Mean $\pm$ paired SE over ten splits.}
\label{tab:tsallis-raps-paired}
\resizebox{\textwidth}{!}{\input{tables/generated_tsallis_raps_paired}}
\end{table}

\begin{table}[H]
\centering
\caption{SAPS: paired AUARC differences (percentage points), $U_2$ minus each comparator. Positive values favour saturation. Mean $\pm$ paired SE over ten splits.}
\label{tab:tsallis-saps-paired}
\resizebox{\textwidth}{!}{\input{tables/generated_tsallis_saps_paired}}
\end{table}

\subsection{Contribution of lower-ranked labels}

As a fixed diagnostic, we truncate the saturating sum at rank $\min\{5,K\}$, keeping only $(1-1/K)^{-1}\sum_{j=2}^{\min\{5,K\}}\alpha_j/[j(j-1)]$. The saturating penalty and this diagnostic were fixed before examining the added experiment's outcomes; the earlier benchmark results were already available. Table~\ref{tab:tsallis-tail} reports rank agreement and contribution shares under LAC. Ranks 2--5 account for roughly 93--96\% of the full score's mean relative contribution, yet keeping the remaining terms improves AUARC on Tiny ImageNet by $0.366\pm0.025$ points and on ImageNetV2--ResNet-18 by $0.987\pm0.031$. Both gains occur in all ten splits. These shares concern score magnitude, including the finite conformal p-value floor; they are not fractions of predictive information. Small contributions can still change useful rankings.

\begin{table}[H]
\centering
\caption{LAC rank agreement and tail contributions. Correlations and contribution shares are averaged over ten splits. The last column gives the full score's AUARC minus the top-5 truncation in percentage points, with paired SE.}
\label{tab:tsallis-tail}
\resizebox{\textwidth}{!}{\input{tables/generated_tsallis_tail}}
\end{table}

\clearpage
\section{SAMPLE-SIZE AND POPULATION-LIMIT BEHAVIOUR}
\label{sec:asymptotic_ambiguity_ablation}

Does the conformal core become precise once both stages receive abundant data? We examine this in a controlled,
correctly specified three-class softmax model with non-degenerate $P(Y\mid X)$. The fitted
classifier is a multinomial logistic-regression model with negligible regularisation, and the nonconformity score is that of least ambiguous classifiers~(LAC),
$s(x,y)=1-\hat p_y(x)$. We use a fixed $3{,}000$-point evaluation sample and consider three
paired regimes over 40 repetitions: (i) an oracle classifier with $n_{\mathrm{cal}}$ varying,
(ii) $n_{\mathrm{train}}$ varying with a $200{,}000$-point Monte Carlo approximation to
population calibration, and (iii) $n_{\mathrm{train}}=n_{\mathrm{cal}}$ varying jointly.
Nested samples are used within a repetition. The same oracle classifier and population-calibration
sample define the horizontal reference limits.

Figure~\ref{fig:asymptotic_ambiguity_ablation} shows convergence to positive limits as the
calibration and training samples grow. At
$n_{\mathrm{train}}=n_{\mathrm{cal}}=30{,}000$, the joint estimates are $0.1625$,
$0.6740$, and $0.1881$ for $\text{TV-MMI}$,
$I_{\mathrm H}$, and $\operatorname{GH}$, respectively;
their oracle/population references are $0.1629$, $0.6753$, and $0.1886$.
The largest absolute discrepancy is $0.0014$. To interpret these positive limits, we sharpen the oracle conditional
distribution while keeping the feature model fixed. Ambiguity decreases with the mean conditional
entropy and is exactly zero for the Dirac-label population. Thus, in this controlled model, large
training and calibration samples remove sampling variation but do not force conformal ambiguity
to vanish when the population contour remains non-degenerate.

\begin{figure}[!htbp]
    \centering
    \includegraphics[width=0.95\textwidth]{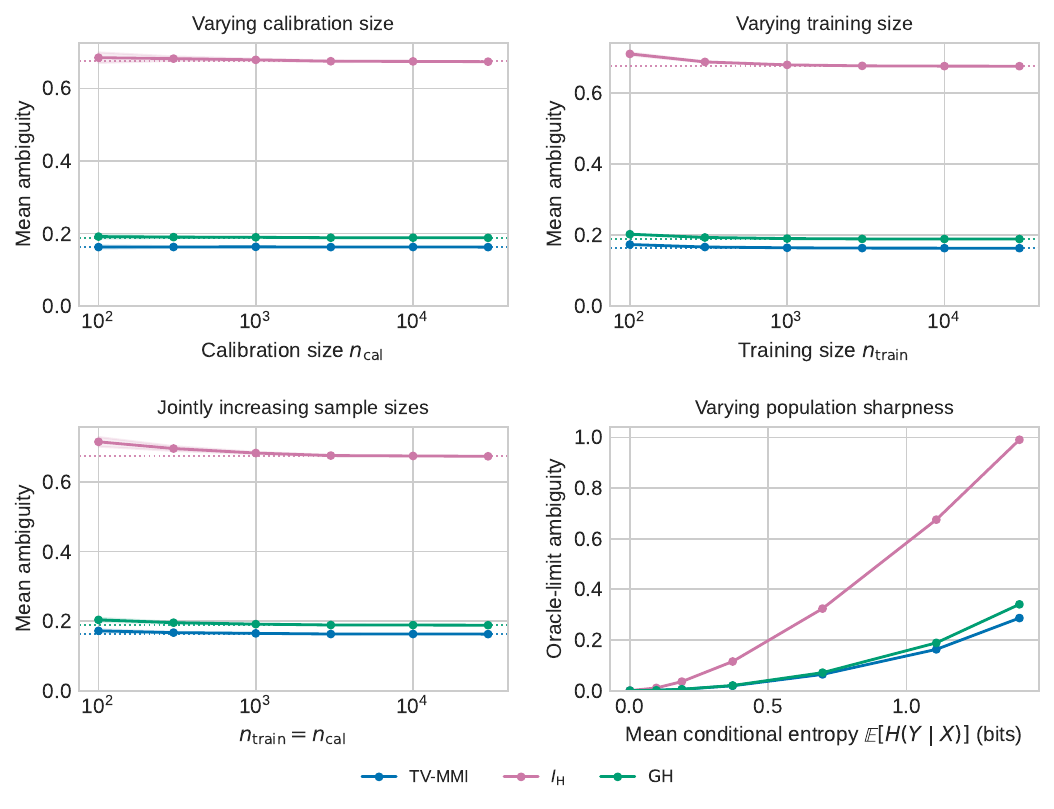}
\caption{Large-sample behaviour in a correctly specified three-class problem. Solid curves are
    means over 40 repetitions and
    ribbons show one standard error; dotted lines are oracle-model/population-calibration Monte
    Carlo references. The lower-right panel varies the oracle softmax temperature and includes
    the exact Dirac control at the origin. Increasing both sample sizes removes sampling variation
    but does not collapse the conformal credal set when $P(Y\mid X)$ is non-degenerate.}
    \label{fig:asymptotic_ambiguity_ablation}
\end{figure}

\clearpage
\section{EXPLORATORY ACTIVE LEARNING}
\label{app:active-learning}

Active learning changes the fitted predictor after each acquisition, introducing a calibration dependence absent from the fixed-predictor experiments in the main text. We therefore report it separately as an exploratory study of sample efficiency. Following \citet{shaker} and \citet{thomas2024improving}, we use uncertainty rankings to acquire labels for a random-forest classifier; the measured outcome is subsequent predictive accuracy, without assuming that any ranking isolates epistemic uncertainty.

\subsection{Protocol and evaluation}

\paragraph{Datasets and fixed partitions.}
We use four OpenML datasets~\citep{vanschoren2014openml}: the multiple-features factors and Karhunen--Lo\`eve views (Digits v1 and v3), optical handwritten digits (Digits v2), and letter recognition (Letters). Before acquisition, each of ten seeds creates a label-stratified partition into five disjoint samples: an initial labelled set, an unlabelled pool, an acquisition-calibration set, an audit-calibration set, and a test set containing $20\%$ of the observations. Table~\ref{tab:active-learning-partitions} gives all partition sizes. Queried observations move from the pool into the fitted sample; both calibration samples and the test sample remain fixed.

\begin{table}[htbp]
\centering
\caption{Active-learning datasets and fixed partition sizes. The last five columns partition every dataset.}
\label{tab:active-learning-partitions}
\small
\setlength{\tabcolsep}{5pt}
\begin{tabular}{lrrrrrrrr}
\toprule
Dataset & Features & Classes & Total & Initial & Pool & Acq. cal. & Audit cal. & Test \\
\midrule
Digits v1 & 216 & 10 & 2,000 & 100 & 750 & 500 & 250 & 400 \\
Digits v2 & 64 & 10 & 5,620 & 150 & 2,846 & 1,000 & 500 & 1,124 \\
Digits v3 & 64 & 10 & 2,000 & 100 & 750 & 500 & 250 & 400 \\
Letters & 16 & 26 & 20,000 & 500 & 12,500 & 2,000 & 1,000 & 4,000 \\
\bottomrule
\end{tabular}
\end{table}

\paragraph{Acquisition and evaluation.}
At each state, we fit a fresh 200-tree random forest to the currently labelled sample, rank the remaining pool, and acquire the ten highest-uncertainty observations. We continue for 300 acquired labels, giving 31 evaluated model states, including the initial state. For each construction---LAC, randomised APS, RAPS, and SAPS---the acquisition-calibration sample supplies ordinary conformal rank p-values. We evaluate TV-MMI, $\operatorname{GH}$, $I_{\mathrm H}$, and fuzziness $F$ on the contour obtained by setting one maximising p-value to one. The three cardinality rules $|C_\alpha^{\mathrm{raw}}|$ count the \emph{unstretched} p-values above $\alpha\in\{.01,.05,.10\}$. Predictive baselines acquire the largest entropy, the largest maximum-probability uncertainty $1-q_{(1)}$ (least confidence), or the largest second probability $q_{(2)}$, where probabilities are ordered decreasingly.

RAPS uses penalty $.01$ and rank cutoff $5$; SAPS uses weight $.2$. These parameters are fixed across datasets. Splits, forest seeds, score-randomisation seeds where applicable, and random priorities for exact ranking ties are paired across methods. At every fitted state, a separate predictor uses the audit-calibration sample to evaluate coverage and set size at $\alpha=.1$ on the test set. Audit outputs never enter acquisition or model selection. We summarise accuracy by the trapezoidal area under the accuracy--acquisition curve divided by the 300-label horizon (AULC); tables express it as a percentage. Standard errors use ten paired seeds and include partition and learner variation.

\subsection{Calibration dependence and validity}
For a predictor fitted independently of its calibration sample, that sample may be enlarged with fresh, independently sampled labelled observations kept out of model fitting, while preserving exchangeability. Adaptively acquired labels do not generally satisfy those conditions, so simply appending queried observations does not justify the ordinary split-conformal guarantee. Keeping the acquisition-calibration set fixed is also insufficient after the initial state: earlier rankings depend on this set, the rankings determine the acquired labels, and those labels influence subsequent fitted models. Feedback-aware procedures can address such dependence under additional assumptions~\citep{fannjiang2022feedback}; this experiment does not establish such a guarantee.

The separate audit holdout removes this feedback from the audit predictor. With an independently sampled, exchangeable audit-calibration/test draw, the usual marginal argument applies at each fixed state conditional on the fitted model; it does not give coverage conditional on the realised calibration sample or simultaneous coverage over acquisition rounds. Our label-stratified finite-benchmark audit is used as an empirical diagnostic. The acquisition results below assess ranking utility and should not be read as a claim of sequential conformal validity.

\subsection{Results and paired comparisons}
Tables~\ref{tab:active-learning-ablation}--\ref{tab:active-learning-saps} report the complete four-score comparison. Higher AULC is better; colour ranks compare unrounded means within a dataset and are descriptive, following Appendix~\ref{appendix: exp}. Every entry is recomputed from matched acquisition trajectories with recorded data, source, and seed provenance. The three predictive baselines, including $1-q_{(1)}$, use their own acquisition trajectories from the matched experiment suite. Their accuracy trajectories are reused across score constructions because acquisition does not use conformal scores; their APS audit values are not reused as audits of other constructions. RAPS/SAPS use the TorchCP all-label implementation~\citep{huang2024torchcp}, including its tie convention, whereas APS retains the existing stable NumPy sorting convention. Consequently, differences between APS and RAPS can reflect probability-tie handling as well as regularisation.

\paragraph{Results and limits.}
Among the reported rules, $1-q_{(1)}$ has the highest mean AULC on the three digit datasets under LAC, while $|C_{.05}^{\mathrm{raw}}|$ leads on Letters. All four LAC contour summaries fall below both predictive entropy and $1-q_{(1)}$ on the digit datasets. Fuzziness is the strongest LAC contour summary on every dataset and improves on $1-q_{(1)}$ on Letters by $0.516\pm .173$ percentage points. On each dataset, at least one of the three prespecified size levels also has a higher mean than every LAC contour summary. This observation compares the displayed fixed rules; selecting the best level separately on each dataset would be a post-hoc oracle. LAC TV-MMI closely tracks the second-probability baseline, consistent with their monotone relationship before finite-calibration ties and the subsequent divergence of adaptive acquisition paths.

With APS and RAPS, TV-MMI is the strongest contour summary on all four datasets and improves on predictive entropy in every case. Relative to $1-q_{(1)}$, TV-MMI gains $0.250\pm .062$ points on Digits v3 and $0.137\pm .112$ on Letters, leading the reported rules on these two datasets, but has lower means on Digits v1 and v2. RAPS changes these TV-MMI results negligibly. Under SAPS, $I_{\mathrm H}$ has the highest reported mean on Digits v1 and v3, and TV-MMI leads on Digits v2. On Letters, GH exceeds $1-q_{(1)}$ by $0.312\pm .113$ points, while $|C_{.10}^{\mathrm{raw}}|$ has the highest reported mean and exceeds all four SAPS contour summaries. These gains over $1-q_{(1)}$ contrast with the selective-prediction results for APS, RAPS, and SAPS, illustrating task dependence. Table~\ref{tab:active-learning-paired} gives selected within-seed contrasts; the complete mean/SE tables give the corresponding absolute scores. These outcomes support a dependence on the score construction and acquisition problem, without a universal sample-efficiency advantage for a particular credal summary or over predictive baselines and fixed set sizes.

\begin{table}[!htbp]
\centering
\caption{LAC active learning: AULC (\%; mean $\pm$ one standard error over ten paired seeds). Higher is better. Conventions follow Appendix~\ref{app:active-learning}.}
\label{tab:active-learning-ablation}
\label{tab:active_learning}
\small
\setlength{\tabcolsep}{8pt}
\begin{tabular}{lcccc}
\toprule
Acquisition rule & Digits v1 & Digits v2 & Digits v3 & Letters \\
\midrule
\multicolumn{5}{l}{\textit{Probability-vector-based measures}} \\
Predictive entropy & $\secondscore{95.38}\pm 0.33$ & $95.23\pm 0.23$ & $\secondscore{93.35}\pm 0.36$ & $78.89\pm 0.28$ \\
Max-probability uncertainty $1-q_{(1)}$ & $\bestscore{95.68}\pm 0.31$ & $\bestscore{96.07}\pm 0.21$ & $\bestscore{93.57}\pm 0.35$ & $79.56\pm 0.25$ \\
Second probability $q_{(2)}$ & $94.51\pm 0.37$ & $94.37\pm 0.15$ & $91.95\pm 0.35$ & $76.01\pm 0.33$ \\
\midrule
\multicolumn{5}{l}{\textit{Credal-based measures}} \\
$\mathrm{TV\!\text{-}\!MMI}$ & $94.47\pm 0.34$ & $94.36\pm 0.17$ & $91.86\pm 0.33$ & $76.02\pm 0.33$ \\
$\operatorname{GH}$ & $94.87\pm 0.34$ & $94.64\pm 0.18$ & $92.55\pm 0.31$ & $77.22\pm 0.32$ \\
$I_{\mathrm H}$ & $94.87\pm 0.33$ & $94.62\pm 0.15$ & $92.51\pm 0.27$ & $77.91\pm 0.35$ \\
$F$ & $95.16\pm 0.34$ & $94.91\pm 0.21$ & $93.08\pm 0.34$ & $\secondscore{80.07}\pm 0.27$ \\
\midrule
\multicolumn{5}{l}{\textit{Set-size-based measures}} \\
$|C_{.01}^{\mathrm{raw}}|$ & $95.34\pm 0.33$ & $\secondscore{95.86}\pm 0.24$ & $93.24\pm 0.41$ & $78.84\pm 0.27$ \\
$|C_{.05}^{\mathrm{raw}}|$ & $94.37\pm 0.45$ & $95.35\pm 0.13$ & $93.04\pm 0.35$ & $\bestscore{80.13}\pm 0.31$ \\
$|C_{.10}^{\mathrm{raw}}|$ & $93.17\pm 0.45$ & $94.29\pm 0.19$ & $92.00\pm 0.34$ & $79.79\pm 0.32$ \\
\bottomrule
\end{tabular}
\end{table}

\begin{table}[!htbp]
\centering
\caption{APS active learning: AULC (\%; mean $\pm$ one standard error over ten paired seeds). Higher is better. Conventions follow Appendix~\ref{app:active-learning}.}
\label{tab:active-learning-aps}
\small
\setlength{\tabcolsep}{8pt}
\begin{tabular}{lcccc}
\toprule
Acquisition rule & Digits v1 & Digits v2 & Digits v3 & Letters \\
\midrule
\multicolumn{5}{l}{\textit{Probability-vector-based measures}} \\
Predictive entropy & $95.38\pm 0.33$ & $95.23\pm 0.23$ & $93.35\pm 0.36$ & $78.89\pm 0.28$ \\
Max-probability uncertainty $1-q_{(1)}$ & $\bestscore{95.68}\pm 0.31$ & $\bestscore{96.07}\pm 0.21$ & $93.57\pm 0.35$ & $\secondscore{79.56}\pm 0.25$ \\
Second probability $q_{(2)}$ & $94.51\pm 0.37$ & $94.37\pm 0.15$ & $91.95\pm 0.35$ & $76.01\pm 0.33$ \\
\midrule
\multicolumn{5}{l}{\textit{Credal-based measures}} \\
$\mathrm{TV\!\text{-}\!MMI}$ & $\secondscore{95.59}\pm 0.31$ & $\secondscore{95.96}\pm 0.22$ & $\bestscore{93.82}\pm 0.33$ & $\bestscore{79.69}\pm 0.26$ \\
$\operatorname{GH}$ & $95.49\pm 0.30$ & $95.85\pm 0.22$ & $93.59\pm 0.35$ & $79.23\pm 0.26$ \\
$I_{\mathrm H}$ & $95.37\pm 0.33$ & $95.82\pm 0.19$ & $\secondscore{93.60}\pm 0.40$ & $79.07\pm 0.27$ \\
$F$ & $95.52\pm 0.30$ & $95.74\pm 0.19$ & $93.54\pm 0.35$ & $78.99\pm 0.32$ \\
\midrule
\multicolumn{5}{l}{\textit{Set-size-based measures}} \\
$|C_{.01}^{\mathrm{raw}}|$ & $94.72\pm 0.36$ & $94.77\pm 0.25$ & $92.72\pm 0.44$ & $78.40\pm 0.34$ \\
$|C_{.05}^{\mathrm{raw}}|$ & $95.09\pm 0.33$ & $95.03\pm 0.23$ & $93.17\pm 0.42$ & $78.75\pm 0.35$ \\
$|C_{.10}^{\mathrm{raw}}|$ & $95.36\pm 0.31$ & $95.40\pm 0.24$ & $93.20\pm 0.40$ & $78.90\pm 0.30$ \\
\bottomrule
\end{tabular}
\end{table}

\begin{table}[!htbp]
\centering
\caption{RAPS active learning: AULC (\%; mean $\pm$ one standard error over ten paired seeds). Higher is better. Conventions follow Appendix~\ref{app:active-learning}.}
\label{tab:active-learning-raps}
\small
\setlength{\tabcolsep}{8pt}
\begin{tabular}{lcccc}
\toprule
Acquisition rule & Digits v1 & Digits v2 & Digits v3 & Letters \\
\midrule
\multicolumn{5}{l}{\textit{Probability-vector-based measures}} \\
Predictive entropy & $95.38\pm 0.33$ & $95.23\pm 0.23$ & $93.35\pm 0.36$ & $78.89\pm 0.28$ \\
Max-probability uncertainty $1-q_{(1)}$ & $\bestscore{95.68}\pm 0.31$ & $\bestscore{96.07}\pm 0.21$ & $93.57\pm 0.35$ & $\secondscore{79.56}\pm 0.25$ \\
Second probability $q_{(2)}$ & $94.51\pm 0.37$ & $94.37\pm 0.15$ & $91.95\pm 0.35$ & $76.01\pm 0.33$ \\
\midrule
\multicolumn{5}{l}{\textit{Credal-based measures}} \\
$\mathrm{TV\!\text{-}\!MMI}$ & $\secondscore{95.60}\pm 0.31$ & $\secondscore{95.96}\pm 0.22$ & $\bestscore{93.82}\pm 0.33$ & $\bestscore{79.69}\pm 0.26$ \\
$\operatorname{GH}$ & $95.54\pm 0.32$ & $95.88\pm 0.20$ & $\secondscore{93.63}\pm 0.33$ & $79.33\pm 0.25$ \\
$I_{\mathrm H}$ & $95.47\pm 0.35$ & $95.68\pm 0.19$ & $93.62\pm 0.37$ & $79.09\pm 0.27$ \\
$F$ & $95.52\pm 0.32$ & $95.79\pm 0.24$ & $93.50\pm 0.35$ & $79.02\pm 0.28$ \\
\midrule
\multicolumn{5}{l}{\textit{Set-size-based measures}} \\
$|C_{.01}^{\mathrm{raw}}|$ & $94.75\pm 0.36$ & $94.59\pm 0.26$ & $92.51\pm 0.51$ & $78.65\pm 0.27$ \\
$|C_{.05}^{\mathrm{raw}}|$ & $95.11\pm 0.35$ & $94.96\pm 0.22$ & $93.23\pm 0.42$ & $78.76\pm 0.30$ \\
$|C_{.10}^{\mathrm{raw}}|$ & $95.36\pm 0.31$ & $95.40\pm 0.24$ & $93.20\pm 0.40$ & $78.92\pm 0.33$ \\
\bottomrule
\end{tabular}
\end{table}

\begin{table}[!htbp]
\centering
\caption{SAPS active learning: AULC (\%; mean $\pm$ one standard error over ten paired seeds). Higher is better. Conventions follow Appendix~\ref{app:active-learning}.}
\label{tab:active-learning-saps}
\small
\setlength{\tabcolsep}{8pt}
\begin{tabular}{lcccc}
\toprule
Acquisition rule & Digits v1 & Digits v2 & Digits v3 & Letters \\
\midrule
\multicolumn{5}{l}{\textit{Probability-vector-based measures}} \\
Predictive entropy & $95.38\pm 0.33$ & $95.23\pm 0.23$ & $93.35\pm 0.36$ & $78.89\pm 0.28$ \\
Max-probability uncertainty $1-q_{(1)}$ & $95.68\pm 0.31$ & $96.07\pm 0.21$ & $93.57\pm 0.35$ & $79.56\pm 0.25$ \\
Second probability $q_{(2)}$ & $94.51\pm 0.37$ & $94.37\pm 0.15$ & $91.95\pm 0.35$ & $76.01\pm 0.33$ \\
\midrule
\multicolumn{5}{l}{\textit{Credal-based measures}} \\
$\mathrm{TV\!\text{-}\!MMI}$ & $\secondscore{95.69}\pm 0.32$ & $\bestscore{96.17}\pm 0.20$ & $\secondscore{93.73}\pm 0.36$ & $79.63\pm 0.27$ \\
$\operatorname{GH}$ & $95.67\pm 0.27$ & $96.06\pm 0.19$ & $93.69\pm 0.37$ & $79.87\pm 0.29$ \\
$I_{\mathrm H}$ & $\bestscore{95.70}\pm 0.32$ & $\secondscore{96.11}\pm 0.21$ & $\bestscore{93.80}\pm 0.35$ & $79.81\pm 0.27$ \\
$F$ & $95.62\pm 0.31$ & $96.09\pm 0.15$ & $93.66\pm 0.33$ & $79.73\pm 0.17$ \\
\midrule
\multicolumn{5}{l}{\textit{Set-size-based measures}} \\
$|C_{.01}^{\mathrm{raw}}|$ & $95.57\pm 0.33$ & $96.04\pm 0.17$ & $93.47\pm 0.37$ & $80.02\pm 0.25$ \\
$|C_{.05}^{\mathrm{raw}}|$ & $95.58\pm 0.31$ & $96.06\pm 0.18$ & $93.54\pm 0.38$ & $\secondscore{80.02}\pm 0.27$ \\
$|C_{.10}^{\mathrm{raw}}|$ & $95.45\pm 0.32$ & $96.01\pm 0.18$ & $93.54\pm 0.42$ & $\bestscore{80.21}\pm 0.32$ \\
\bottomrule
\end{tabular}
\end{table}

\begin{table}[!htbp]
\centering
\caption{Selected paired differences in normalised active-learning AULC (percentage points; mean $\pm$ one standard error of the ten within-seed differences). Positive values favour the first acquisition rule. These descriptive contrasts explain the patterns in the complete tables; they are not multiplicity-adjusted tests.}
\label{tab:active-learning-paired}
\small
\setlength{\tabcolsep}{7pt}
\begin{tabular}{lcccc}
\toprule
Paired contrast & Digits v1 & Digits v2 & Digits v3 & Letters \\
\midrule
LAC: $F-H(q)$ & $-0.219\pm 0.215$ & $-0.323\pm 0.223$ & $-0.266\pm 0.144$ & $+1.187\pm 0.168$ \\
LAC: $F-(1-q_{(1)})$ & $-0.519\pm 0.146$ & $-1.156\pm 0.233$ & $-0.487\pm 0.121$ & $+0.516\pm 0.173$ \\
LAC: TV-MMI $-q_{(2)}$ & $-0.042\pm 0.095$ & $-0.005\pm 0.038$ & $-0.093\pm 0.066$ & $+0.002\pm 0.002$ \\
APS: TV-MMI $-H(q)$ & $+0.218\pm 0.106$ & $+0.727\pm 0.114$ & $+0.470\pm 0.083$ & $+0.808\pm 0.064$ \\
APS: TV-MMI $-(1-q_{(1)})$ & $-0.082\pm 0.082$ & $-0.106\pm 0.107$ & $+0.250\pm 0.062$ & $+0.137\pm 0.112$ \\
RAPS: TV-MMI $-H(q)$ & $+0.218\pm 0.106$ & $+0.727\pm 0.114$ & $+0.470\pm 0.083$ & $+0.808\pm 0.064$ \\
RAPS: TV-MMI $-(1-q_{(1)})$ & $-0.082\pm 0.082$ & $-0.106\pm 0.107$ & $+0.250\pm 0.062$ & $+0.137\pm 0.112$ \\
SAPS: $I_{\mathrm H}-H(q)$ & $+0.323\pm 0.089$ & $+0.881\pm 0.084$ & $+0.457\pm 0.086$ & $+0.924\pm 0.112$ \\
SAPS: $\operatorname{GH}-H(q)$ & $+0.294\pm 0.088$ & $+0.825\pm 0.126$ & $+0.347\pm 0.102$ & $+0.983\pm 0.118$ \\
SAPS: $\operatorname{GH}-(1-q_{(1)})$ & $-0.006\pm 0.074$ & $-0.008\pm 0.082$ & $+0.127\pm 0.075$ & $+0.312\pm 0.113$ \\
\bottomrule
\end{tabular}
\end{table}


\end{document}